%% file: main.tex
\documentclass{article}

\usepackage{microtype}
\usepackage{graphicx}
\usepackage{subcaption}
\usepackage{booktabs} 

\usepackage{hyperref}

\usepackage[accepted]{icml2023}

\input{math_commands.tex}

\usepackage{hyperref}
\usepackage{url}
\usepackage{amsmath}
\usepackage{amsthm}
\usepackage{amssymb}
\usepackage{algorithm}
\usepackage{enumitem}
\usepackage{algorithmic}
\usepackage{mathtools}
\usepackage{tabularx}
\usepackage{wrapfig}
\usepackage{capt-of}
\usepackage{booktabs}
\usepackage{float}
\usepackage{comment}
\usepackage{varwidth}
\usepackage{pifont}
\usepackage{setspace} 
\newcommand{\cmark}{\textcolor{teal}{\ding{51}}}%
\newcommand{\xmark}{\textcolor{red}{\ding{55}}}%
\usepackage{capt-of}
\usepackage{graphicx}
\usepackage[geometry]{ifsym}

\usepackage[capitalize,noabbrev]{cleveref}

\theoremstyle{plain}
\newtheorem{theorem}{Theorem}[section]

\newtheorem{lemma}[theorem]{Lemma}
\newtheorem{corollary}[theorem]{Corollary}
\theoremstyle{definition}
\newtheorem{definition}[theorem]{Definition}
\newtheorem{assumption}[theorem]{Assumption}
\theoremstyle{remark}
\newtheorem{remark}[theorem]{Remark}

\usepackage[textsize=tiny]{todonotes}

\newcommand{\norm}[1]{\left\lVert#1\right\rVert}
\newcommand*{\Scale}[2][4]{\scalebox{#1}{$#2$}}%
\newcommand{\sg}{\texttt{StopGrad}}

\usepackage{Definitions}

\newcommand{\bX}{\bar X}
\newcommand{\by}{{\bar y}}
\newcommand{\bz}{{\bar \rvz}}
\newcommand{\bzeta}{{\bar \zeta}}
\newcommand{\bS}{{\bar S}}
\newcommand{\bell}{{\bar \ell}}
\newcommand{\hf}{{\hat f}}
\newcommand{\bbf}{{\bar f}}

\definecolor{Red}{rgb}{0.768, 0.054, 0.054}
\definecolor{Blue}{rgb}{0.152, 0.294, 0.925}
\definecolor{Green}{rgb}{0,0.4,0.7}
\definecolor{hotpink}{rgb}{1.0, 0.41, 0.71}
\definecolor{brown}{rgb}{0.59, 0.29, 0.0}
\definecolor{purple}{rgb}{0.59, 0.44, 0.84}
\definecolor{darkpastelgreen}{rgb}{0.01, 0.75, 0.24}
\definecolor{celestialblue}{rgb}{0.29, 0.59, 0.82}
\definecolor{ceruleanblue}{rgb}{0.16, 0.32, 0.75}
\definecolor{goldenrod}{rgb}{0.85, 0.65, 0.13}
\definecolor{navyblue}{rgb}{0.0, 0.0, 0.5}
\definecolor{coolgrey}{rgb}{0.55, 0.57, 0.67}
\definecolor{darkseagreen}{rgb}{0.56, 0.74, 0.56}
\definecolor{darkturquoise}{rgb}{0.0, 0.81, 0.82}
\hypersetup{
    colorlinks=true,
    citecolor=teal,
    linkcolor=Red,
    urlcolor=Green,
    }

\usepackage{xspace}
\makeatletter
\DeclareRobustCommand\onedot{\futurelet\@let@token\@onedot}
\def\@onedot{\ifx\@let@token.\else.\null\fi\xspace}

\def\ie{\emph{i.e}\onedot}

\def\wrt{w.r.t\onedot} 

\makeatother

\icmltitlerunning{Scalable Set Encoding with Universal Mini-Batch Consistency and Unbiased Full Set Gradient Approximation}

\begin{document}
\let\ra\rightarrow
\let\mb\mathbf
\let\mbb\mathbb
\let\mc\mathcal
\let\mf\mathfrak
\let\tld\tilde
\let\bs\boldsymbol

\twocolumn[
\icmltitle{\texorpdfstring{Scalable Set Encoding with Universal Mini-Batch Consistency and\\ Unbiased Full Set Gradient Approximation}{Scalable Set Encoding with Universal Mini-Batch Consistency and Unbiased Full Set Gradient Approximation}}



\icmlsetsymbol{equal}{*}

\begin{icmlauthorlist}
\icmlauthor{Jeffrey Willette}{equal,kaist}
\icmlauthor{Seanie Lee}{equal,kaist}
\icmlauthor{Bruno Andreis}{kaist}
\icmlauthor{Kenji Kawaguchi}{sing}
\icmlauthor{Juho Lee}{kaist}
\icmlauthor{Sung Ju Hwang}{kaist}
\end{icmlauthorlist}

\icmlaffiliation{kaist}{KAIST}
\icmlaffiliation{sing}{National University of Singapore}

\icmlcorrespondingauthor{Jeffrey Willete}{jwillette@kaist.ac.kr}
\icmlcorrespondingauthor{Seanie Lee}{lsnfamily02@kaist.ac.kr}
\icmlcorrespondingauthor{Sung Ju Hwang}{sjhwang82@kaist.ac.kr}

\icmlkeywords{Machine Learning, ICML}

\vskip 0.3in
]



\printAffiliationsAndNotice{\icmlEqualContribution} 

\input{sections/0_abstract.tex}
\input{sections/1_introduction.tex}
\input{sections/2_related_works.tex}
\input{sections/3_method.tex}
\input{sections/5_experiment.tex}
\input{sections/6_conclusion.tex}


\bibliography{reference}
\bibliographystyle{icml2023}

\newpage
\clearpage
\onecolumn
\input{sections/7_appendix}

\end{document}

%% file: math_commands.tex
\usepackage{amsmath,amsfonts,bm}

\def\eqref#1{equation~\ref{#1}}

\def\1{\bm{1}}

\def\ra{{\textnormal{a}}}
\def\rb{{\textnormal{b}}}

\def\rvs{{\mathbf{s}}}

\def\rvx{{\mathbf{x}}}

\def\rvz{{\mathbf{z}}}

\def\vb{{\bm{b}}}

\DeclareMathAlphabet{\mathsfit}{\encodingdefault}{\sfdefault}{m}{sl}
\SetMathAlphabet{\mathsfit}{bold}{\encodingdefault}{\sfdefault}{bx}{n}

\let\ab\allowbreak

%% file: sections/0_abstract.tex
\begin{abstract}
Recent work on mini-batch consistency (MBC) for set functions has brought attention to the need for sequentially processing and aggregating chunks of a partitioned set while guaranteeing the same output for all partitions. However, existing constraints on MBC architectures lead to models with limited expressive power. 
Additionally, prior work has not addressed how to deal with large sets during training when the full set gradient is required. 
To address these issues, we propose a Universally MBC (UMBC) class of set functions which can be used in conjunction with arbitrary non-MBC components while still satisfying MBC, enabling a wider range of function classes to be used in MBC settings. 
Furthermore, we propose an efficient MBC training algorithm which gives an unbiased approximation of the full set gradient and has a \emph{constant memory overhead for any set size for both train- and test-time}.
We conduct extensive experiments including image completion, text classification, unsupervised clustering, and cancer detection on high-resolution images to verify the efficiency and efficacy of our scalable set encoding framework. Our code is available at \href{https://github.com/jeffwillette/umbc}{\color{teal}{\texttt{github.com/jeffwillette/umbc}}} 
\end{abstract}

%% file: sections/1_introduction.tex
\vspace{-0.1in}
\section{Introduction}
\vspace{-0.05in}
For a variety of problems for which deep models can be applied, unordered sets naturally arise as an input. For example, a set of words in a document~\citep{jurafsky2008speech} and sets of patches within an image for multiple instance learning~\citep{mil-medical-img}. Functions which encode sets are commonly known as set encoders, and most previously proposed set encoding functions~\cite{deepsets,set-trans} have implicitly assumed that the whole set can fit into memory and be accessed in a single chunk. However, this is not a realistic assumption if it is necessary to process large sets or streaming data. As shown in~\cref{st-chunk-stream},  Set Transformer~\citep{set-trans} cannot properly handle streaming data and suffers performance degradation. Please see~\cref{motivation-2} for more qualitative examples. \citet{mbc} identified this problem, and introduced the mini-batch consistency (MBC) property which dictates that an MBC set encoding model must be able to sequentially process subsets from a partition of a set while guaranteeing the same output over any partitioning scheme, as illustrated in~\cref{fig:concept}. In order to satisfy the MBC property, they devised an attention-based MBC model, the Slot Set Encoder (SSE). 

\input{figures/concept}
\input{figures/motivation}
Although SSE satisfies MBC, there are several limitations. First, it has limited expressive power due to the constraints imposed on its architecture. Instead of the conventional softmax attention~\citep{attention}, the attention of SSE is restricted to using a sigmoid for attention without normalization over the rows of the attention matrix, which may be undesirable for applications requiring convex combinations of inputs. Moreover, the Hierarchical SSE is a composition of pure MBC functions and thus cannot utilize more expressive non-MBC models, such as those utilizing self-attention. Another crucial limitation of the SSE is its limited scalability during training. Training models with large sets requires computing gradients over the full set which can be computationally prohibitive. SSE proposes to randomly sample a small subset for gradient computation, which is a \emph{biased estimator} of the full set gradient as we show in~\cref{app:bias}.

To tackle these limitations of SSE, we propose \emph{Universal MBC (UMBC) set functions} which enable utilizing a broader range of functions while still satisfying the MBC property. Firstly, we relax the restriction to the sigmoid on the activation functions for attention and show that cross-attention with a wider class of activation functions, including the softmax, is MBC. Moreover, we re-interpret UMBC's output as a set, which as we show in~\cref{fig:concept,sec:method}, \emph{universally} allows for the application of non-MBC set encoders when processing UMBC's output sets, resulting in more expressive functions while maintaining the MBC property. For a concrete example, UMBC used in conjunction with the (non-MBC) Set Transformer (ST) produces consistent output for any partition of a set as shown in~\cref{embedding-var}, and outperforms all other MBC models for clustering streaming data as illustrated in~\cref{motivation}. 

Lastly, for training MBC models, we propose a novel and scalable algorithm to approximate full set gradient. Specifically, we obtain the full set representation by partitioning the set into subsets and aggregating the subset representations while only considering a portion of the subsets for gradient computation. We find this leads to a constant memory overhead for computing the gradient with a fixed size subset, and is an \emph{unbiased estimator} of the full set gradient.

To verify the efficacy and efficiency of our proposed UMBC framework and full set gradient approximation algorithm, we perform extensive experiments on a variety of tasks including image completion, text classification, unsupervised clustering, and cancer detection on high-resolution images. Furthermore, we theoretically show that UMBC is a universal approximator of continuous permutation invariant functions under some mild assumptions and the proposed training algorithm minimizes the total loss of the full set version by making progress toward its stationary points. We summarize our contributions as follows:
\begin{itemize}
[itemsep=1.0mm, parsep=0pt, leftmargin=*]
    \item We propose a UMBC framework which allows for a broad class of activation functions, including softmax, for attention and also enables utilizing non-MBC functions in conjuction with UMBC while satisfying MBC, resulting in more expressive and less restrictive architectures.
    
    \item We propose an efficient training algorithm with a constant memory overhead for any set size by deriving an unbiased estimator of the full set gradient which empirically performs comparably to using the full set gradient.

    \item We theoretically show that UMBC is a universal approximator to continuous permutation invariant functions under mild assumptions and our algorithm minimizes the full set total loss by making progress toward its stationary points.
\end{itemize}

\input{figures/embedding_var}

%% file: figures/concept.tex
\begin{figure*}
    \centering
    \includegraphics[width=0.9\textwidth]{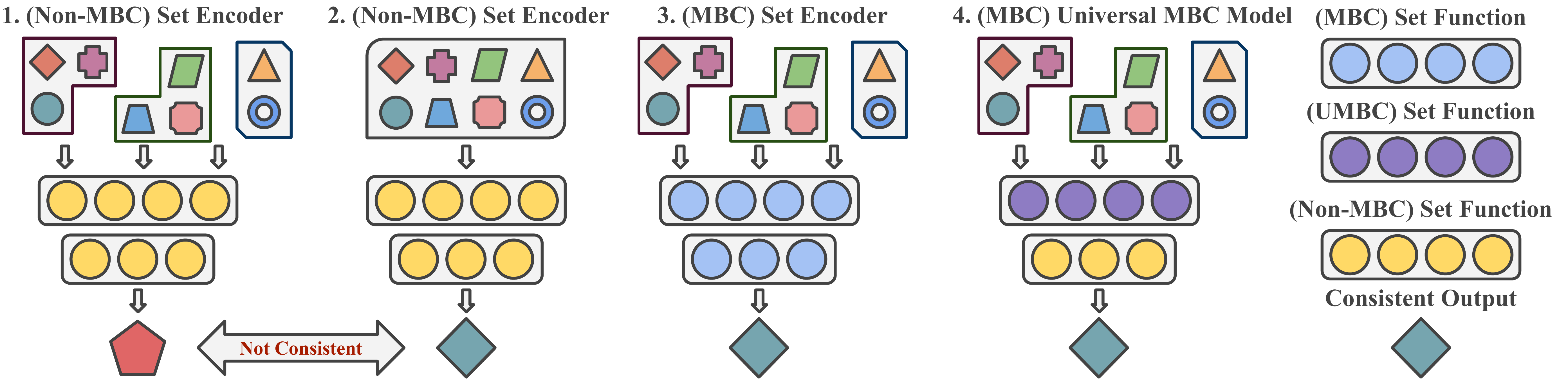}
    \vspace{-0.15in}
    \caption{\small Non-MBC models (1, 2) produce inconsistent outputs when given different set partitions. MBC models~(3) produce consistent outputs for any random partition with a specific architecture. UMBC composes both MBC/non-MBC components, expanding the possible set of MBC functions, and allowing for more expressive models.}
    \label{fig:concept}
    \vspace{-0.16in}
\end{figure*}

%% file: figures/motivation.tex
\begin{figure*}
\centering
    \begin{subfigure}{0.24\textwidth}
        \centering
        \includegraphics[width=\linewidth]{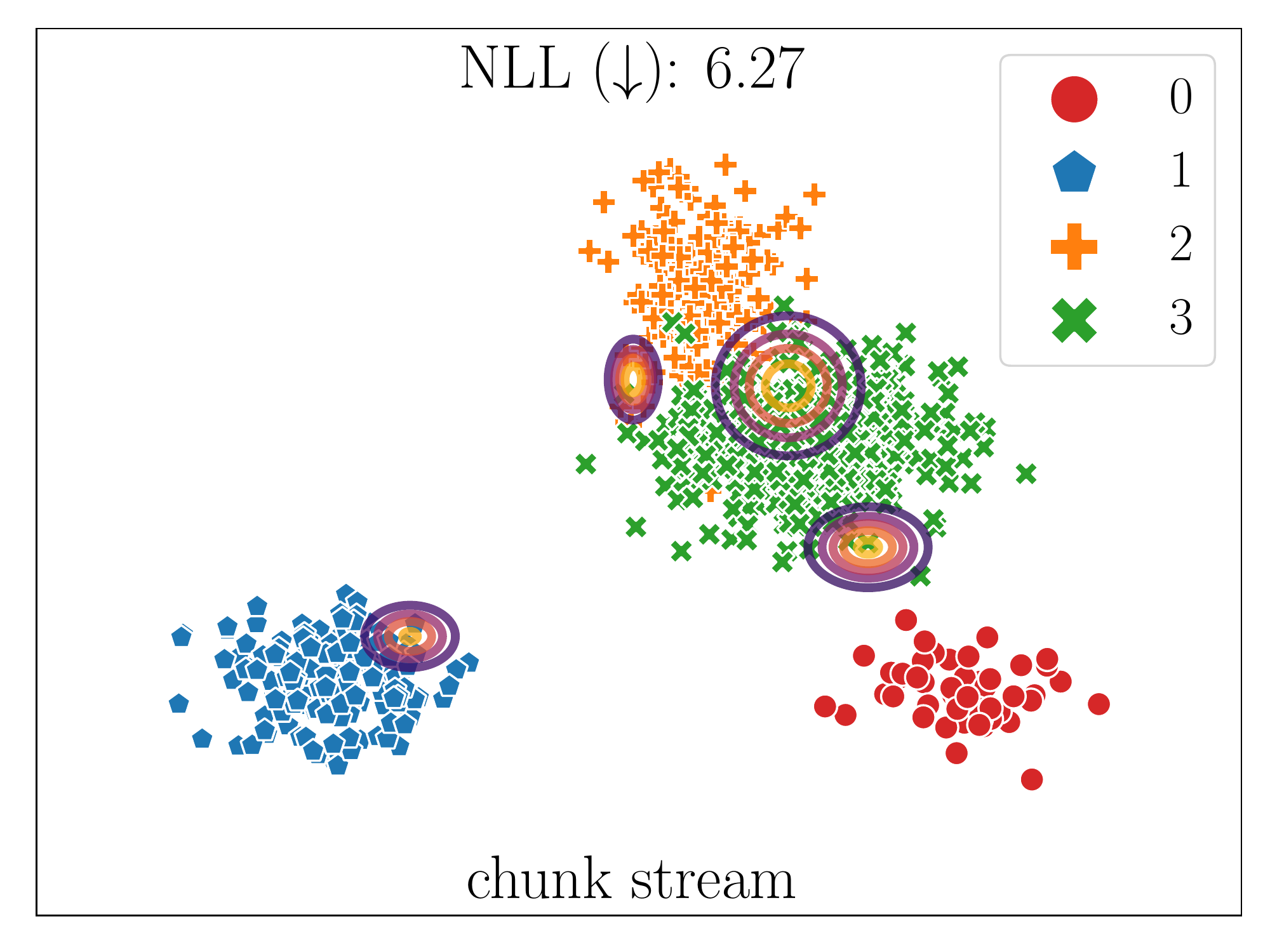}
        \captionsetup{justification=centering,margin=0.5cm}		
        \vspace{-0.25in}
        \caption{\small Set Transformer \xmark}
        \label{st-chunk-stream}	
    \end{subfigure}%
    \begin{subfigure}{0.24\textwidth}
	\centering
	\includegraphics[width=\linewidth]{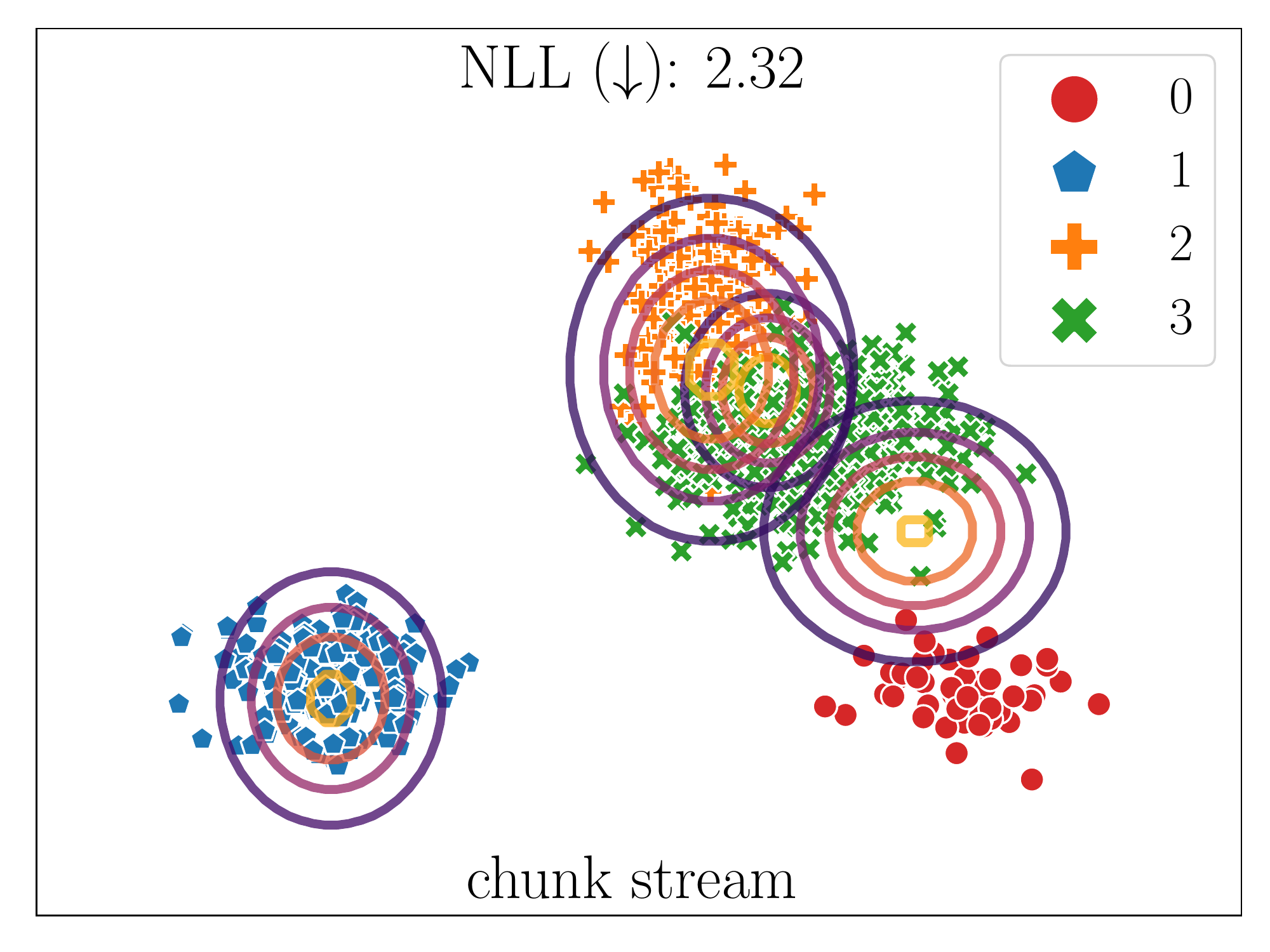}
	\captionsetup{justification=centering,margin=0.5cm}
	\vspace{-0.25in}
	\caption{\small Deepsets \cmark} 
	\label{ds-chunk-stream}
    \end{subfigure}%
    \begin{subfigure}{0.24\textwidth}
	\centering
	\includegraphics[width=\linewidth]{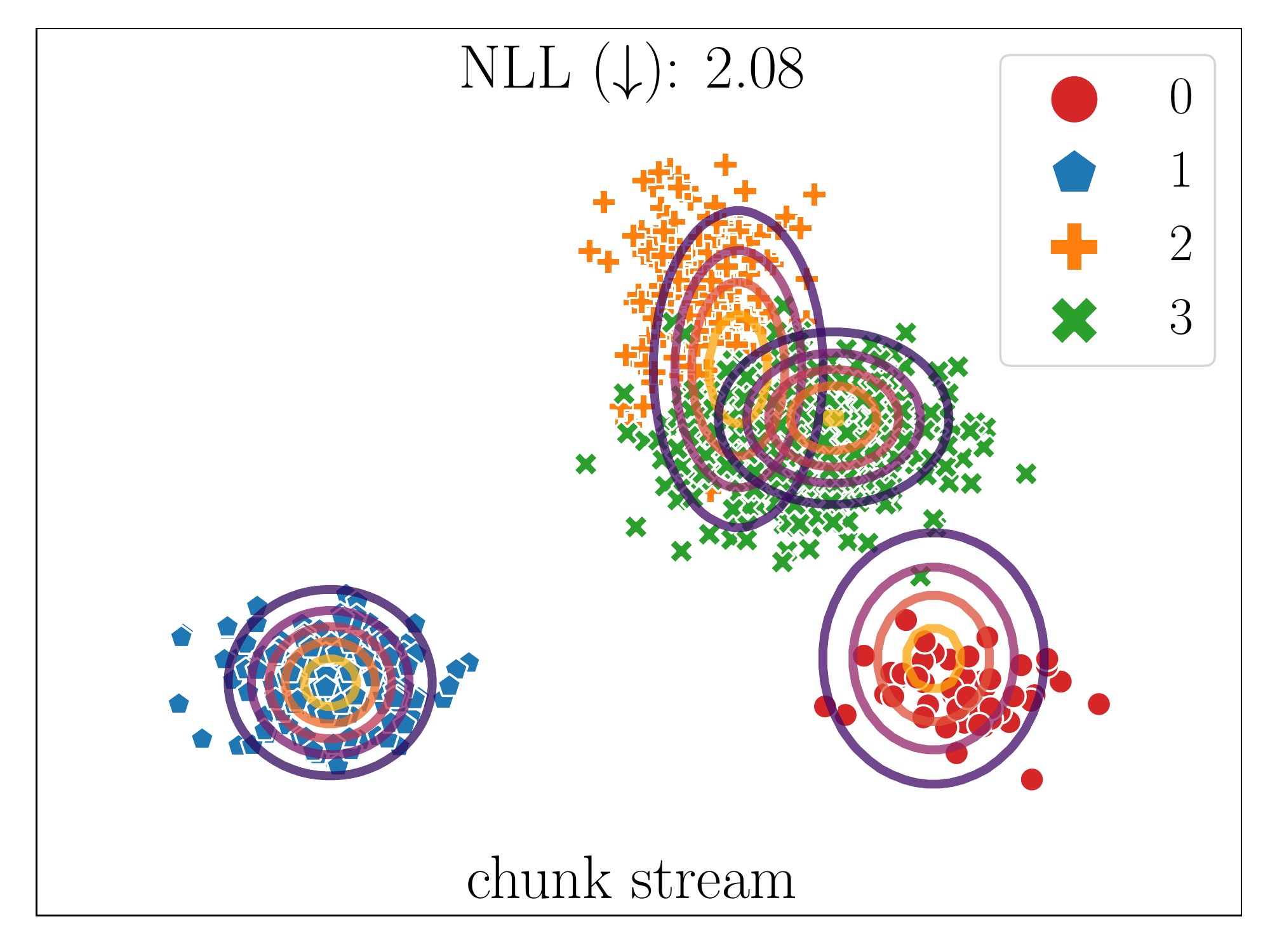}
	\captionsetup{justification=centering,margin=0.5cm}
	\vspace{-0.25in}
	\caption{\small Slot Set Encoder \cmark}
	\label{sse-chunk-stream}
    \end{subfigure}%
    \begin{subfigure}{0.24\textwidth}
	\centering
	\includegraphics[width=\linewidth]{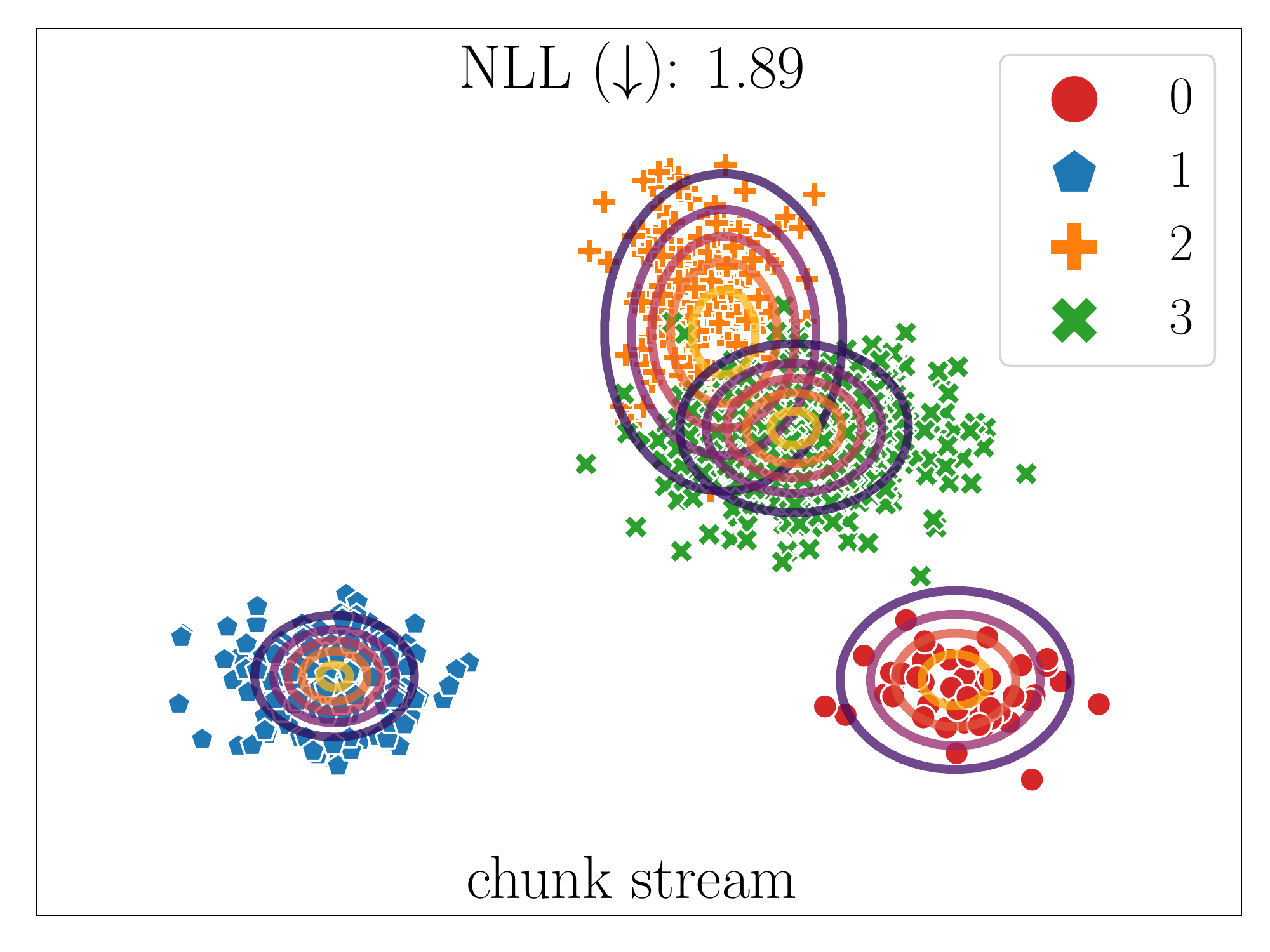}
	\captionsetup{justification=centering,margin=0.5cm}
	\vspace{-0.25in}
	\caption{\small UMBC+ST (Ours) \cmark}
	\label{umbc-chunk-stream}
    \end{subfigure}%
    \vspace{-0.15in}
    \caption{\small \textbf{Streaming inputs}: A non-MBC model~(\subref{st-chunk-stream}) suffers performance degradation in streaming settings. MBC models~(\subref{ds-chunk-stream}, \subref{sse-chunk-stream}, \subref{umbc-chunk-stream}) can handle streaming inputs consistently. Creating an MBC composition of both MBC/non-MBC components~(\subref{umbc-chunk-stream}) creates the strongest MBC model. A (\cmark) indicates MBC models while (\xmark) indicates non-MBC models. }
    \label{motivation}
    \vspace{-0.2in}
\end{figure*}

%% file: figures/embedding_var.tex
\begin{figure}
\vspace{-0.23in}
    \begin{center}
    \includegraphics[width=0.65\linewidth]{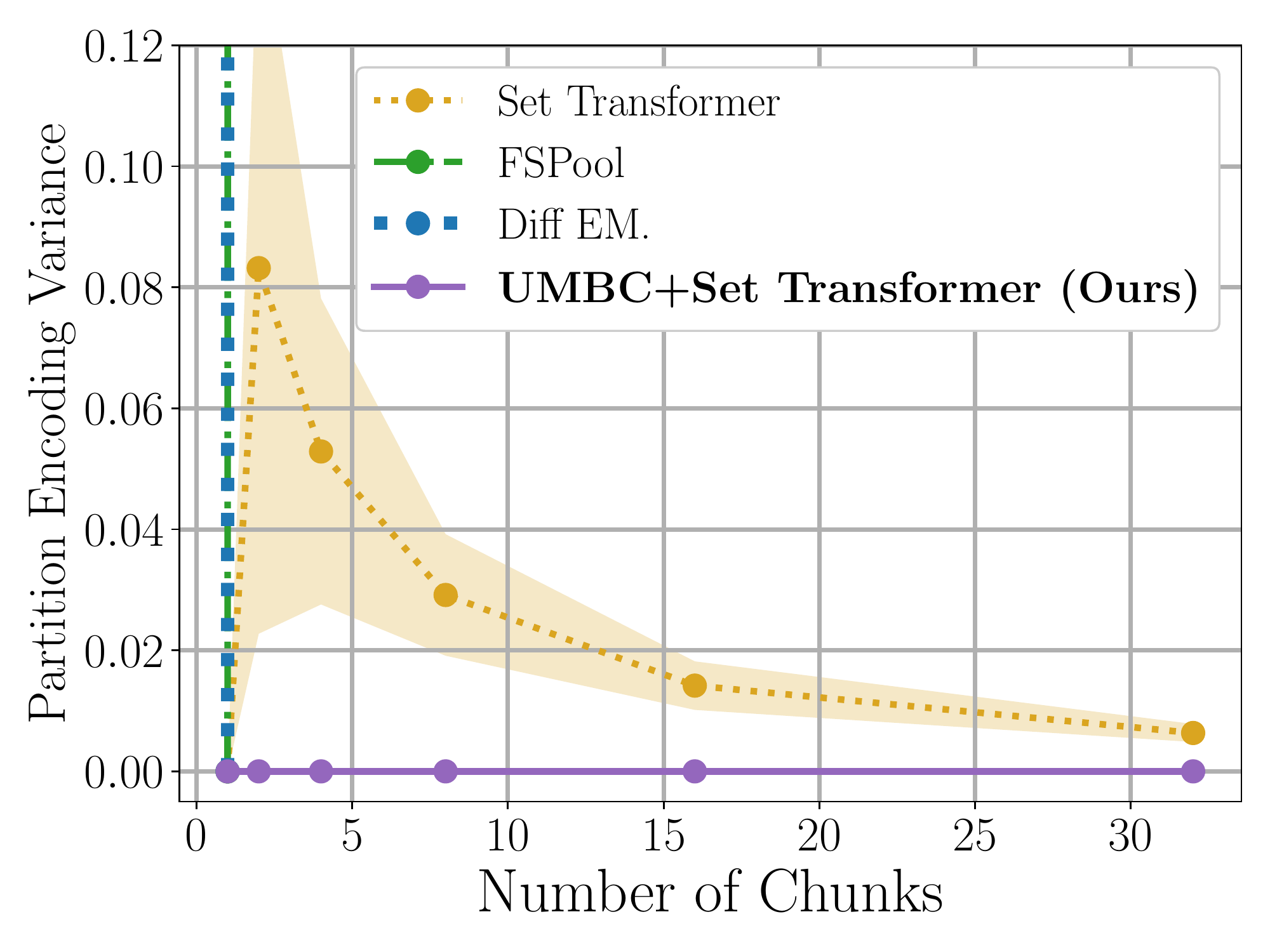}
    \end{center}
    \vspace{-0.25in}
    \caption{\small Variance between encodings of 100 different partitions of the same set. 
    UMBC+ST satisfies MBC and thus has no variance.}
    \label{embedding-var}
\vspace{-0.23in}
\end{figure}

%% file: sections/2_related_works.tex
\section{Related Work}
\textbf{Set Encoding.}
Deep learning for set structured data has been an active research topic since the introduction of DeepSets \citep{deepsets}, which solidified the requirements of deep set functions, namely permutation equivariant feature extraction and permutation invariant set pooling. \citet{deepsets} have shown that under certain conditions, functions which satisfy the aforementioned requirements act as universal approximators for functions of sets. Subsequently, the Set Transformer~\citep{set-trans} applied attention~\citep{attention} to sets, which has proven to be a powerful tool for set functions. Self-attentive set functions excel on tasks where independently processing set elements may fail to capture pairwise interactions between elements. Subsequent works which utilize pairwise set element interactions include optimal transport \citep{optimal-transport-kernel-embedding} and expectation maximization \citep{diff-em}. Other notable approaches to permutation invariant set pooling include featurewise sorting \citep{fspool}, and canonical orderings of set elements~\citep{janossy}.

\textbf{Mini-Batch Consistency (MBC).} Every method mentioned in the preceding paragraph suffers from an architectural bias which limits them to seeing and processing the whole set in a single chunk. \citet{mbc} identified this problem, and highlighted the necessity for MBC which guarantees that processing and aggregating each subset from a set partition results in the same representation as encoding the entire set at once~(\cref{def:mbc}). This is important in settings where the data may not fit into memory due to either large data or limited on-devices resources. In addition to identifying the MBC property, \citet{mbc} also proposed the Slot Set Encoder (SSE) which utilizes cross attention between learnable `slots' and set elements in conjunction with simple activation functions in order to achieve an MBC model. As shown in~\cref{table:mbc-status}, however, SSE cannot utilize self-attention to model pairwise interactions of set elements due to the constraints imposed on its architecture, which makes it less expressive than the Set Transformer.

\input{figures/mbc_table}

%% file: figures/mbc_table.tex
\begin{table}
\vspace{-0.05in}
\caption{\small Properties of set functions. UMBC models can use arbitrary component set functions and are therefore unconstrained.}
\label{table:mbc-status}
\centering
\resizebox{0.45\textwidth}{!}{
      \begin{tabular}{lccc}
        \toprule
        Model & MBC & Cross-Attn. & Self-Attn. \\
        \midrule
        DeepSets \citep{deepsets} & \cmark & \xmark & \xmark \\
        SSE \citep{mbc} & \cmark & \cmark & \xmark \\
        FSPool \citep{fspool} & \xmark & \xmark & \xmark \\
        Diff EM \citep{diff-em} & \xmark & \xmark & \xmark \\
        Set Transformer \citep{set-trans} & \xmark & \cmark & \cmark \\
        \midrule
        \textbf{(Ours) UMBC+Any Set Function} & \cmark & \cmark & \cmark \\
        \bottomrule
      \end{tabular}
}
\vspace{-0.25in}
\end{table}

%% file: sections/3_method.tex
\section{Method}\label{sec:method}
In this section, we describe the problem we target and provide a formulation for UMBC models along with a derivation of our unbiased full set gradient approximation algorithm. All proofs of theorems are deferred to Appendix~\ref{app:proofs}.
\subsection{Preliminaries}\label{sec:method-prelim}
Let $\mathfrak{X}$ be a $d_x$-dimensional vector space over $\RR$ and let $2^\mathfrak{X}$ be the power set of $\mathfrak{X}$. We focus on a collection of finite sets $\Xcal$, which is a subset of $2^\mathfrak{X}$  such that $\sup_{X\in\Xcal}\lvert X\rvert\in\NN$. We want to construct a parametric function $f_\theta: \Xcal \to \mathcal{Z}$ satisfying permutation invariance. Specifically, given a set $X_i=\{\rvx_{i,j}\}_{j=1}^{N_i} \in \Xcal$, the output of the function $Z_i=f_\theta(X_i)$ is a fixed sized representation which is invariant to all permutations of the indices $\{1,\ldots,N_i\}$.
For supervised learning, we define a task specific decoder $g_\lambda:\Zcal \to \RR^{d_y}$ and optimize parameters $\theta$ and $\lambda$ to minimize the loss
\begin{equation}
    L(\theta, \lambda) = \frac{1}{n}\sum_{i=1}^n \ell((g_\lambda \circ f_\theta)(X_i), y_i)
\label{eq:loss-func}
\end{equation}
on training data $((X_i, y_i))_{i=1}^n$, where $y_i$ is a label for the input set $X_i$ and $\ell$ denotes a loss function.
\begin{definition}[Permutation Invariance]
\vspace{-0.05in}
    Let $\mathfrak{S}_N$ be the set of all permutations of $\{1,\ldots, N\}$, \ie $\mathfrak{S}_N = \{\pi:[N] \to [N]  \mid  \pi \text{ is bijective} \}$ where $[N]\coloneqq \{1,\ldots, N\}$. A function $f_\theta: \Xcal\to\mathcal{Z}$ is permutation invariant \textit{iff} $f_\theta(\{\rvx_{\pi(1)}, \ldots \rvx_{\pi(N)} \}) = f_\theta(\{\rvx_1, \ldots, \rvx_N\})$ for all $X\in\Xcal$ and  for all permutation $\pi \in \mathfrak{S}_N$.
\end{definition}
We further assume that the cardinality of a set $X$ is sufficiently large, such that loading and processing the whole set at once is computationally prohibitive. For non-MBC models, a naïve approach to solve this problem would be to encode a small subset of the full set as an approximation, leading to a possibly suboptimal representation of the full set. Instead, \citet{mbc} propose a \textit{mini-batch consistent} (MBC) set encoder, the Slot Set Encoder (SSE), to piecewise process disjoint subsets of the full set and aggregate them to obtain a consistent full set representation. 
\begin{definition}[Mini-Batch Consistency]
\label{def:mbc}
    We say a function $f_\theta$ is mini-batch consistent \textit{iff} for any $X\in \Xcal$, there is a function $h$ such that for any partition $\zeta(X)$ of the set $X$,
    \begin{align}
        f_\theta(X) = h\left(\{f_\theta(S) \in \Zcal \mid S\in \zeta(X)\}  \right).
    \end{align}
\end{definition}
Models which satisfy the MBC property can partition a set into  subsets, encode, and then aggregate the subset representations to achieve the exact same output as encoding the full set. Due to constraints on the architecture of the SSE, however, on certain tasks the SSE shows weaker performance than non-MBC set encoders such as Set Transformer~\citep{set-trans} which utilizes self attention. To tackle this limitation, we propose Universal MBC (UMBC) set encoders which are both MBC and also allow for the use of arbitrary non-MBC set functions while still satisfying MBC property. 

\subsection{Universal Mini-Batch Consistent Set Encoder}\label{sec:method-umbc}
In this section, we provide a formulation of our UMBC set encoder $f_\theta$. Given an input set $X\in\Xcal$, we represent it as a matrix $X=[\rvx_1 \cdots \rvx_N]^\top \in \RR^{N\times d_x}$ whose rows are elements in the set, and independently process each element with $\phi:\RR^{d_x}\to\RR^{d_h}$ as $\Phi(X)= [\phi(\rvx_1) \cdots \phi(\rvx_N)]^\top$, where $\phi$ is a deep feature extractor. We then compute the un-normalized attention score between a set of learnable slots $\Sigma=[\rvs_1\cdots\rvs_k]^\top\in\RR^{k\times d_s}$ and $\Phi(X)$ as:
\begin{gather}
    Q =  \Scale[0.96]{\texttt{LN}}(\Sigma W^Q),  K(X) = \Phi(X)W^K,   V(X) =  \Phi(X) W^V \nonumber\\
    \hat{A} =  \sigma\left(\sqrt{d^{-1}}\cdot QK(X)^\top\right) \in \mathbb{R}^{k\times N},
\label{eq:sse-module}
\end{gather}
where $\sigma$ is an element-wise activation function with $\sigma(x) \gneq 0$ for all $x\in\RR$, \texttt{LN} denotes layer normalization~\cite{layer-norm}, and  $W^Q\in\mathbb{R}^{d_s\times d},W^K\in\mathbb{R}^{d_h\times d}, W^V\in\mathbb{R}^{d_h\times d}$ are parameters which are part of $\theta$. For simplicity, we omit biases for $Q,K$, and $V$. With the un-normalized attention score $\hat{A}$, we can define 
a map
\begin{equation}
    \hat{f}_\theta: X\in\mathbb{R}^{N\times d_x}\mapsto \nu_p(\hat{A})V(X) \in \mathbb{R}^{k \times d}
\label{eq:sse}
\end{equation}
for $p=1, 2$, where $\nu_p: \RR^{k\times d} \to \RR^{k\times d}$ is defined by either $\nu_1(\hat{A})_{i,j}= \hat{A}_{i,j}/\sum_{i=1}^k\hat{A}_{i,j}$ which normalizes the columns or the identity mapping $\nu_2(\hat{A})_{i,j}=\hat{A}_{i,j}$. The choice of $\nu_p$ depends on the desired activation function $\sigma$. Alternatively, similar to slot attention~\citep{slot-attn}, we can make the function stochastic by sampling $\rvs_i \sim \mathcal{N}(\mathbf{\mu}_i, \text{diag}(\texttt{softplus}(\mathbf{v}_i)))$ with reparameterization~\citep{vae} for $i=1,\ldots, k$, where $\mu_i \in \RR^{d_s}, \mathbf{v}_i\in \RR^{d_s}$ are part of the parameters $\theta$. 
If we sample $\rvs_1, \ldots, \rvs_k \stackrel{\text{iid}}{\sim}\mathcal{N}(\mu_1, \text{diag}(\texttt{softplus}(\mathbf{v}_1)))$ with a sigmoid for $\sigma$ and $\nu_1$ for normalization, and then apply a pooling function (sum, mean, min, or max) to the columns of $[\hat{f}_\theta(X)^\top_1 \cdots \hat{f}_\theta(X)^\top_k]^\top\in\RR^{k\times d}$, we achieve a function equivalent to the SSE, where $\hat{f}_\theta(X)_i$ is $i$-th row of $\hf_\theta(X)$.

However, SSE has some drawbacks. First, since the attention score of $\nu_p(\hat{A})_{i,j}$ is independent to the other $N-1$ attention scores $\nu_p(\hat{A})_{i,l}$ for $l\neq j$, it is impossible for the rows of $\nu_p(\hat{A})$ to be convex coefficients as the softmax outputs in conventional attention~\citep{attention}. Notably, in some of  our experiments, the constrained attention activation originally used in the SSE, which we call slot-sigmoid, significantly degrades generalization performance.
Furthermore, stacking hierarchical SSE layers has been shown to harm performance~\citep{mbc}, which limits the power of the overall model.

To overcome these limitations of the SSE, we propose a Universal Mini-Batch Consistent (UMBC) set encoder $f_\theta$ by allowing the set function $f_\theta$ to also use arbitrary non-MBC functions. Firstly, we propose normalizing the attention matrix $\nu_p(\hat{A})$ over rows to consider dependency among different elements of the set in the attention operation:
\begin{gather}
\label{eq:sse-constant}
    \bbf_\theta: X \in \RR^{N \times d_x}\mapsto \nu_p(\hat{A}) \1_N \in\RR^{k} \\
f_\theta: X\in \RR^{N \times d_x} \mapsto \text{diag}\left(\bbf_\theta(X) \right)^{-1} \hf_\theta(X) \in \RR^{k\times d}
\end{gather}
where $\1_N=(1,\ldots, 1)\in\RR^{N}$. We prove that a UMBC set encoder $f_\theta$ is \emph{permutation invariant, equivariant, and MBC}.
\begin{theorem}
A UMBC function is permutation invariant.      
\label{thm:perm-inv}
\end{theorem}
Any strictly positive elementwise function is a valid $\sigma$. For an instance, if we use the identity mapping $\nu_2$ with $\sigma(\cdot)\coloneqq\exp(\cdot)$, the attention matrix $\text{diag}(\bbf_\theta(X))^{-1}\nu_p(\hat{A})$ is equivalent to applying the softmax to each row of $\hat{A}$, which is hypothesized to break the MBC property by~\citet{mbc}. 
However, we show that this does not break the MBC property in Appendix~\ref{app:proof:mbc}.
Intuitively, since 
\begin{align}\label{eq:f-and-normalization}
    \hf_\theta(X) = \sum_{S\in\zeta(X)} \hf_\theta(S), \text{ and } \bbf_\theta(X) = \sum_{S\in\zeta(X)} \bbf_\theta(S)
\end{align}
holds for any partition $\zeta(X)$ of the set $X$, we can iteratively process each subset $S\in\zeta(X)$ and aggregate them without losing any information of $f_\theta(X)$, \ie, $f_\theta$ is MBC even when normalizing over the $N$ elements of the set. Note that the operation outlined above is mathematically equivalent to the softmax, but uses a non-standard implementation. We discuss the implementation and list 5 such valid attention activation functions which satisfy the MBC property in~\cref{sec:method-attn-acts}.

\begin{theorem}
\label{thm:mbc}
Given the slots $\Sigma=[\rvs_1 \cdots \rvs_k]^\top \in \RR^{k\times d_s}$, a UMBC set encoder is mini-batch consistent. 
\end{theorem}
Lastly, we may consider the output of a UMBC set encoder $f_\theta(X)$ as either a fixed vector or a set of $k$ elements. Under the set interpretation, we may therefore apply subsequent functions on the set of cardinality $k$. To provide a valid input to subsequent set encoders, it is sufficient to view UMBC as a set to set function $\varphi(\Sigma; X, \theta):\Sigma \mapsto f_\theta(X)$ for each set $X\in\Xcal$, which is permutation equivariant \wrt the slots $\Sigma$.
\begin{definition}
    A function $\varphi: \RR^{k\times d_s}\to \RR^{k\times d}$ is said to be permutation equivariant \textit{iff} $\varphi([\Sigma_{\pi(1)}^\top \cdots \Sigma_{\pi(k)}^\top]^\top) = [\varphi(\Sigma)_{\pi(1)}^\top \cdots \varphi(\Sigma)_{\pi(k)}^\top ]^\top$ for all $\Sigma \in\RR^{k\times d_s}$ and for all $\pi \in \mathfrak{S}_k$, where $\mathfrak{S}_k=\{ \pi:[k]\to[k]\mid \pi \text{ is bijective}\}$ contains all permutations of $\{1,\ldots,k\}$, and $\varphi(\Sigma)_i, \Sigma_i$ denote $i$-th row of $\varphi(\Sigma)$ and $\Sigma$, respectively.
\end{definition}
\begin{theorem}
For each input $X\in\Xcal$,  $\varphi(\Sigma; X, \theta): \Sigma\mapsto f_\theta(X)$ is equivariant \wrt permutations of the slots $\Sigma$.
\label{thm:perm-eq}
\end{theorem}
A key insight is that we can leverage non-MBC set encoders such as Set Transformer after a UMBC layer to improve expressive power of an MBC model while still satisfying MBC~(\cref{def:mbc}). As a result, with some assumptions, a UMBC set encoder used in combination with any continuously sum decomposable~\citep{deepsets} permutation invariant deep neural network is a universal approximator of the class of continuously sum decomposable functions. 
\begin{theorem}
\label{thm:univ-approx}
Let $d_x=1$ and restrict the domain $\Xcal$ to $[0,1]^M$. Suppose that the nonlinear activation function of $\phi$ has nonzero Taylor coefficients up to degree $M$. Then, UMBC used in conjunction with any continuously sum-decomposable permutation-invariant deep neural network with nonlinear activation functions that are not polynomials of finite degrees is a universal approximator of the class of functions $\mathcal{F}
=\{f: [0,1]^M\to\RR \ \mid f\text{ is continuous and  permutation invariant}\}$.
\end{theorem}
Although we use a non-MBC set encoder on top of UMBC, this \textit{does not} violate the MBC property. Since we may obtain $f_\theta(X)$ by sequentially processing each subset of $X$ and the resulting set with cardinality $k$ is assumed small enough to load $f_\theta(X)$ in memory, we can directly provide the MBC output of UMBC to the non-MBC set encoder.  
\begin{corollary}
    Let $\hat{g}_\omega:\RR^{k\times d} \to \Zcal$ be a (non-MBC) set encoder and let $f_\theta:\Xcal\to\RR^{k\times d}$ be a UMBC set encoder. Then $\hat{g}_\omega \circ f_\theta$ is mini-batch consistent.
\label{cor:mbc-no-mbc}
\end{corollary}
For notational convenience, we write $g_\lambda$ to indicate the composition $g_\lambda \circ \hat{g}_\omega$ of a set encoder $\hat{g}_\omega:\RR^{k\times d}\to\Zcal$ and a decoder $g_\lambda:\Zcal\to\RR^{d_y}$, throughout the paper. Similarly, the parameter $\lambda$ denotes  $(\omega, \lambda)$. 
\subsection{Stochastic Approximation of the Full Set Gradient}\label{sec:method-gradapprox}
Although we can leverage SSE or UMBC at test-time by sequentially processing subsets to obtain the full set representation $f_\theta(X)$, at train-time it is infeasible to utilize the gradient of the loss~(\eqref{eq:loss-func}) \wrt the full set. Computation of the full set gradient with automatic differentiation requires storing the entire computation graph for all forward passes of each subset $S$ from $\zeta(X)$ denoted as a partition of a set $X$, which incurs a prohibitive computational cost for large sets. As a simple approximation, \citet{mbc} propose randomly sampling a single subset $S_{i,j} \in \zeta(X_i)$ and computing the gradient of the loss $\ell((g_\lambda \circ f_\theta)(S_{i,j}), y_i)$ based on a single subset at each iteration. 
\begin{remark}
    Let $\zeta(X_i)$ be a partition of set $X_i\in\Xcal$ and $S_{i,j}\in\zeta(X_i)$ be a subset of $X_i$. Then the gradient of $1/n\sum_{i=1}^n\ell((g_\lambda \circ f_\theta)(S_{i,j}), y_i)$ is a \emph{biased estimation} of the full set gradient and leads to a suboptimal solution in our experiments. Please see Appendix~\ref{app:bias} for further details.
\end{remark}
In order to tackle this issue, we propose an \textit{unbiased estimation} of the full set gradient which incurs a constant memory overhead. Firstly, we uniformly and independently sample a mini-batch $((\bX_i, \by_i))_{i=1}^m$ from the training dataset $((X_i,y_i))_{i=1}^n$ for every iteration $t\in\mathbb{N}_+$. We denote this process by $((\bX_{i}, \by_{i}))^m_{i=1} \sim D[((X_{i}, y_{i}))^n_{i=1}].$ 
Then, for each $\bX_i$, we sample a mini-batch $\bzeta_t(\bX_i)=\{\bS_1, \ldots, \bS_{\lvert \bzeta_t(\bX_i)\rvert}\}$ from the partition $\zeta_t(\bX_i)=\{S_1, \ldots, S_{\lvert \zeta_t(\bX_i)\rvert}\}$ of $\bX_i$, \ie, all $\bS_j$ are drawn independently and uniformly from $\zeta_t(\bX_i)$. Denote this process by $\bzeta_t(\bX_i)\sim D[\zeta_t(\bX_i)]$. Instead of storing the computational graph of all forward passes of subsets in the partition $\zeta_t(\bX_i)$ of a set $\bX_i$, we apply $\sg$ to all subsets $S \notin \bzeta_t(\bX_i)$ as follows:
\begin{align}
    \hf_{\theta}^{\bzeta_t, \zeta_t}(\bX_i) &= \sum_{S \in\bzeta_t(\bX_i)}  \hf_{\theta}(S)+ \Scale[0.8]{\texttt{StopGrad}}(\sum_{\mathclap{S \in\zeta_t(\bX_i)\setminus \bzeta_t(\bX_i)}}  \hf_{\theta}(S)) \\
    \bbf_{\theta}^{\bzeta_t, \zeta_t}(\bX_{i})_{}&= \sum_{S \in\bzeta_t(\bX_i)}  \bbf_{\theta}(S)+ \Scale[0.8]{\texttt{StopGrad}}(\sum_{\mathclap{S \in\zeta_t(\bX_i)\setminus \bzeta_t(\bX_i)}}  \bbf_{\theta}(S))
\label{eq:stop-grad}
\end{align}
where, for any function $q:\theta \mapsto q(\theta)$, the symbol $\sg(q(\theta))$ denotes a constant with its value being $q(\theta)$, \ie, $\partial \sg(q(\theta))/\partial\theta = 0$. 
For simplicity, we omit the superscript $\zeta_t$ if there is no ambiguity. Finally, we update both the parameter $\theta$ and $\lambda$ of the respective encoder and decoder using the  gradient of the following functions, respectively at $t\in\mathbb{N}_+$:
\begin{align}
L_{t,1}(\theta, \lambda) &= \frac{1}{m}\sum_{i=1}^m \frac{|\zeta_t(\bX_i)|}{|\bzeta_t(\bX_i)|}\ell(g_\lambda(f_{\theta}^{\bzeta_t}(\bX_{i}) ), y_i) \label{eq:loss-theta}\\
L_{t,2}(\theta, \lambda) &= \frac{1}{m}\sum_{i=1}^m \frac{1}{\lvert \bzeta_t(\bX_i)\rvert}\ell(g_\lambda(f_{\theta}^{\bzeta_t}(\bX_{i}) ), y_i). 
\label{eq:loss-lambda}
\end{align}
We outline our proposed training method in~\cref{algo}. Note that we can apply our algorithm to any set encoder for which a full set representation can be decomposed into a summation of subset representations as in~\eqref{eq:f-and-normalization} such as Deepsets with sum or mean pooling, or SSE which are in fact special cases of UMBC. Furthermore, we can apply the algorithm to any differentiable non-MBC set encoder if we simply place a UMBC layer before the non-MBC function. As a consequence of the $\sg()$ operation, if we set $\lvert\bzeta_t(\bX_i)\rvert=1$, our method incurs the same computation graph storage cost as randomly sampling a single subset. Moreover, $\partial L_{t,1}(\theta, \lambda)/\partial \theta $ and $\partial L_{t,2}(\theta, \lambda)/\partial \lambda$ are unbiased estimators of $\partial L(\theta,\lambda)/\partial \theta$ and $\partial L(\theta,\lambda)/\partial\lambda$, respectively.
\begin{theorem}
\label{thm:umbaised}
For any $t\in\mathbb{N}_+$, $\frac{\partial L_{t,1}(\theta, \lambda)}{\partial \theta}$ and $\frac{\partial L_{t,2}(\theta,\lambda)}{\partial \lambda}$ are unbiased estimators of $\frac{\partial L(\theta, \lambda)}{\partial \theta}$ and $\frac{\partial L(\theta, \lambda)}{\partial \lambda}$ as follows:
{\small
\begin{align}
\EE_{((\bX_{i}, \by_{i}))^m_{i=1} } \EE_{(\bzeta_t(\bX_i))^m_{i=1}}\left[\frac{\partial L_{t,1}(\theta, \lambda)}{\partial \theta}\right]&=\frac{\partial L(\theta,\lambda)}{\partial \theta} \\
\EE_{((\bX_{i}, \by_{i}))^m_{i=1} } \EE_{(\bzeta_t(\bX_i))^m_{i=1}}\left[\frac{\partial L_{t,2}(\theta, \lambda)}{\partial \lambda}\right]&=\frac{\partial L(\theta,\lambda)}{\partial \lambda} \label{eq:lambda}
,
\end{align}}%
where the first expectation is taken for $((\bX_{i}, \by_{i}))^m_{i=1} \sim D[((X_{i}, y_{i}))^n_{i=1}]$, and the second expectation is taken for  $\bzeta_t(\bX_i) \sim D[\zeta_t(\bX_i)]$ for all $i \in \{1,\ldots,m\}$.
\end{theorem}

Under mild conditions, Theorem \ref{thm:umbaised} implies that the updating $\theta$ with our method makes progress towards minimizing $L(\theta,\lambda)$. We formalize this statement in Appendix \ref{app:opt}.

%% file: sections/5_experiment.tex
\vspace{-0.1in}
\section{Experiments}
\input{figures/mvn.tex}
\vspace{-0.05in}
\subsection{Amortized Clustering}
\vspace{-0.05in}
We consider amortized clustering on a dataset generated from Mixture of $K$ Gaussians (MoGs) (See \cref{app:mog-details} for dataset construction details). Given a set $X_i=\{\rvx_{i,j} \}_{j=1}^{N_i}$ sampled from a MoGs, the goal is to output the mixing coefficients, and Gaussian mean and variance, which minimizes the negative log-likelihood of the set as follows:
\begin{gather}\label{eq:mog}
    \{\pi_j(X_i), \mu_j(X_i), \Sigma_j(X_i)\}_{j=1}^K= f_\theta(X_i) \\
    \ell(f_\theta(X_i)) = -\sum_{l=1}^{N_i}\log \sum_{j=1}^K \pi_j(X_i) \mathcal{N}(\rvx_{i,l}; \Theta_j(X_i))
\end{gather}
where $\Theta_j(X_i) = (\mu_j(X_i), \Sigma_j(X_i))$ denotes a mean vector and a diagonal covariance matrix for $j$-th Gaussian, and $\pi_j(X_i)$ is $j$-th mixing coefficient. Note that there is no label $y_i$ since it is an unsupervised clustering problem. We optimize the parameters of the set encoder $\theta$ to minimize the loss over a batch, $L(\theta)=1/n\sum_{i=1}^n \ell(f_\theta(X_i))$.

\textbf{Setup.}
We evaluate training with the full set gradient vs. the unbiased estimation of the full set gradient. In this setting, for gradient computation, MBC models use a subset of 8 elements from a full set of 1024 elements. Non MBC models such as Set Transformer (ST), FSpool~\citep{fspool}, and Diff EM~\citep{diff-em} are also trained with the set of 8 elements. We compare our UMBC model against Deepsets, SSE, Set Transformer, FSPool, and Diff EM. Note that at test-time all non-MBC models process every 8 element subset from the full set independently and aggregate the representations with mean pooling.

\input{figures/mvn_table.tex}
\textbf{Results.} 
In~\cref{mvn}, interestingly, the unbiased estimation of the full set gradient (\textcolor{red}{\textbf{red}}) is almost indistinguishable from the full set gradient (\textcolor{celestialblue}{\textbf{blue}}) for DeepSets and UMBC, while there is a significant gap for SSE. In all cases, the unbiased estimation of the full set gradient outperforms training with the biased gradient approximation with only the set of 8 elements per random sample (\textcolor{darkpastelgreen}{\textbf{green}}), which is proposed by~\citet{mbc}. 
Lastly, as shown in~\cref{tab:mvn-memory}, we compare all models in terms of generalization performance (NLL), memory usage, and wall-clock time for processing a single subset. All non-MBC models show underperformance due to their violation of the MBC property. However, if we utilize `UMBC+' compositions, the composition becomes MBC with significantly improved performance and little added overhead for memory and time complexity. In contrast, a composition of pure MBC functions, the Hierarchical SSE degrades the performance of SSE. Notably, UMBC with Set Transformer outperforms all other models whereas Set Transformer alone achieves the worst NLL. These results verify expressive power of UMBC in conjuction with non-MBC models.

\input{figures/celeba.tex}
\vspace{-0.1in}
\subsection{Image Completion}
\vspace{-0.05in}
In this task, we are given a set of $M$ RGB pixel values $y_i \in\RR^{3}\times  \cdots \times \RR^{3}$ of an image as well as the corresponding 2-dimensional coordinates $(\rvx_{i,1},\ldots, \rvx_{i,M}) $ normalized to be in $[0,1]^2\times  \cdots \times [0,1]^2$. The goal of the task is to predict RGB pixel values of all $M$ coordinates of the image. Specifically, given a context set $X_i=\{(\rvx_{i,c_j}, y_{i, c_j})\}_{j=1}^{N_i}$ processed by the set encoder, we obtain the set representation $f_\theta(X_i)\in\RR^{k\times d}$. Then, a decoder $g_\lambda: \RR^{k\times d} \times [0,1]^2 \times \cdots \times [0,1]^2 \to \RR^{6} \times \cdots \times \RR^{6}$ which utilizes both the set representation and the target coordinates, learns to predict a mean and variance for each coordinate of the image as $((\hat{\mu}_{i,j,l})_{l=1}^3, (\hat{\sigma}_{i,j,l} )_{l=1}^3)_{j=1}^M = g_\lambda(\rvz_i)$, where $\rvz_i= (f_\theta(X_i), \rvx_{i,1}, \ldots, \rvx_{i,M})$.
Then we compute the negative log-likelihood of the label set $y_i=((y_{i,j,l})_{l=1}^3 )_{j=1}^{M}$:
\begin{align}
    \ell(g_\lambda(\rvz_i), y_i) = -\sum_{j=1}^M\sum_{l=1}^3 \log \Ncal(y_{i,j,l}; \hat{\mu}_{i,j,l}, \hat{\sigma}_{i,j,l}),
\end{align}
where $\Ncal(\cdot;\mu, \sigma)$ is a univariate Gaussian probability density function.
Finally, we optimize $\theta$ and $\lambda$ to minimize the loss $L(\theta, \lambda)=1/n\sum_{i=1}^n\ell(g_\lambda(\rvz_i),y_i)$.

\textbf{Setup.} In our experiments, we impose the restriction that a set encoder is only allowed to compute the gradient with 100 elements of a context set $X_i$ during training and the model can only process 100 elements of the context set at once at test time. We train the set encoders in a Conditional Neural Process~\citep{cnp} framework, using $32\times 32$ images from the CelebA dataset~\citep{celeba}. We vary the cardinality of the context set size $N_i\in \{100, \ldots, 500\}$ and compare the negative log-likelihood (NLL) of each model. For baselines, we compare our UMBC set encoder against: Deepsets, SSE, Hierarchical SSE, Set Transformer (ST), FSPool, and Diff EM. For our UMBC, we use the softmax for $\sigma$ in~\eqref{eq:sse-module} and place the ST after the UMBC layer.

\textbf{Results.} First, as shown in \cref{fig:celeba-all-model}, our \textcolor{purple}{\textbf{UMBC + ST}} model outperforms all baselines, empirically verifying the expressive power of UMBC. SSE underperforms in terms of NLL due to its constrained architecture. Moreover, stacking hierarchical SSE layers degrades the performance of SSE for larger sets. Note that all MBC set encoders (\textbf{Deepsets}, \textcolor{brown}{\textbf{SSE}}, \textcolor{darkturquoise}{\textbf{Hierarchical SSE}} and \textcolor{purple}{\textbf{UMBC}}) in \cref{fig:celeba-all-model} are trained with our proposed unbiased gradient approximation in~\cref{algo}. On the other hand, we train non-MBC models such as \textcolor{goldenrod}{\textbf{Set Transformer (ST)}}, \textcolor{darkseagreen}{\textbf{FSPool}}, and \textcolor{coolgrey}{\textbf{Diff EM}} with a randomly sampled subset of 100 elements, and perform mean pooling over all subset representations at test-time to approximate an MBC model. 

Additionally, \cref{fig:celeba-mem} shows GPU memory usage for each model while processing sets of varying cardinalities without a memory constraint. The marker size is proportional to the set cardinality. Notably, all four MBC models incur a \emph{constant memory} overhead to process any set size, as we can apply $\sg$ to most of the subset, and compute an unbiased gradient estimate with a fixed sized subset (100). However, memory overhead for all non-MBC models is a function of set size. Thus, Set Transformer uses more than twice the memory of UMBC to achieve a similar log-likelihood on a set of 500 elements.

Lastly, in~\cref{fig:celeba-umbc}, we show how our proposed unbiased training algorithm (\textcolor{red}{\textbf{red}}) improves the generalization performance of UMBC models compared to training with a limited subset of 100 elements (\textcolor{darkpastelgreen}{\textbf{green}}). Notably, the performance of our algorithm is indistinguishable from that of training models with the full set gradient (\textcolor{celestialblue}{\textbf{blue}}). We present similar plots for Deepsets and SSE in~\cref{fig:celeba-deepset,fig:celeba-sse}. Across all models, our training algorithm significantly and consistently improves performance compared to training with random subsets of 100 elements, while requiring the same amount of memory. These empirical results verify both efficiency and effectiveness of our proposed method.
\input{figures/text_table.tex}
\vspace{-0.1in}
\subsection{Long Document Classification}
\vspace{-0.05in}
In this task, we are given a long document $X_i = (\rvx_{i,1}, \ldots, \rvx_{i,N_i})$ consisting of an average of $707.99$ words. The goal of this task is to predict a binary multi-label $y_i=(y_{i,1}, \ldots, y_{i,c})\in\{0,1\}^{c}$ of the document, where $c$ is the number of classes. We ignore the order of words and consider the document as a multiset of words, \ie, a set allowing duplicate elements. Specifically, given a training dataset $((X_i, y_i))_{i=1}^n$, we process each set $X_i$ with the set encoder to obtain the set representation $f_\theta(X_i)\in\RR^{k\times d}$. We then use a decoder  $g_\lambda:\RR^{k\times d}\to\RR^c$ to output the probability of each class and compute the cross entropy loss:
\begin{equation}
\Scale[0.9]{
    \displaystyle{\ell(\rvz_i, y_i) = -\sum_{j=1}^c\left(y_{i,j}\log \tilde{\sigma}(z_{i,j})+ (1-y_{i,j})\log(1-\tilde{\sigma}(z_{i,j})) \right)}}
\label{eq:bce-loss}
\end{equation}
where $\rvz_i = (z_{i,1}, \ldots, z_{i,c}) = (g_\lambda \circ f_\theta)(X_i)$ and $\tilde{\sigma}$ denotes the sigmoid function. Finally we optimize $\theta$ and $\lambda$ to minimize the loss $L(\theta, \lambda)=1/n\sum_{i=1}^n\ell(\rvz_i, y_i)$.

\vspace{-0.05in}
\textbf{Setup.}
All models are trained on the inverted EURLEX dataset~\citep{eurlex} consisting of long legal documents divided into sections. The order of sections are inverted following prior work~\citep{long-text}. To predict a label, we give the whole document to the model without any truncation. We compare the micro F1 of each model. 

We compare UMBC to Deepsets, SSE, ToBERT~\citep{tobert}, and Longformer~\citep{longformer}.
For Deepsets and SSE, we use the pre-trained word embedding from BERT~\citep{bert} without positional encoding and 2 layer fully connected (FC) networks for feature extractor $\phi$. We use another 3 layer FC network for the decoder. For UMBC, we use the same feature extractor as  SSE but instead use the pre-trained BERT as a decoder, with a randomly initialized linear classifier. We remove the positional encoding of BERT for UMBC to ignore word order. For all the MBC models, we train them both with full set denoted as ``w/ full" and with our gradient approximation method on a subset of 100 elements denoted as ``w/ 100".

\input{figures/text_fig.tex}
\vspace{-0.05in}
\textbf{Results.}
As shown in Table~\ref{tab:text-exp}, our proposed UMBC outperforms all baselines including non-MBC models --- Longformer and ToBERT which require excessive amounts of GPU memory for training models with long sequences. This result again verifies the expressive power of UMBC with BERT (a non-MBC model) for long document classification. Moreover, with significantly less GPU memory, all MBC models (Deepsets, SSE, and UMBC) trained with our unbiased gradient approximation using a subset of 100 elements, achieve similar performance to the models trained with full set.
Lastly, in Figure~\ref{fig:umbc_text}, we plot the micro F1 score as a function of the cardinality of the subset used for gradient computation when training the UMBC model. Our proposed unbiased gradient approximation (\textcolor{red}{\textbf{red}}) shows consistent performance for all subset cardinalities. In contrast, training the model with a small random subset (\textcolor{darkpastelgreen}{\textbf{green}}) is unstable, resulting in underperformance and higher F1 variance.

\input{figures/camelyon_table.tex}

\vspace{-0.1in}
\subsection{Multiple Instance Learning (MIL)}\label{sec:mil}
\vspace{-0.05in}
In MIL, we are given a `bag' of instances with a corresponding bag label, but no labels for each instance within the bag. Labels should not depend on the order of the instances, \ie, MIL can be recast as a set classification problem. Specifically, given a set $X_i=\{\rvx_{i,j} \}_{j=1}^{N_i}$, the goal is to predict its binary label $y_i\in\{0,1\}$. For this task, we obtain two streams of set representations and compute the cross entropy loss from the decoder $g_\lambda: \RR^{k\times d}\to \RR$ output as:
\begin{gather}
     \rvz_{i,1} = \max\{w^\top \phi(\rvx_{i,j}) + b \}_{j=1}^{N_i}, \quad \rvz_{i,2} = f_\theta(X_i) \\
     \Lcal_i = \frac{1}{2}\left(\ell(\rvz_{i,1}, y_i) + \ell(g_\lambda(\rvz_{i,2}), y_i)\right),
\end{gather}
where $w\in \RR^{d_h}$ and $b\in \RR$ are parameters and $\ell$ is the cross entropy loss described in~\eqref{eq:bce-loss} with $c=1$. We optimize all parameters $\theta, \lambda, w,\text{and } b$ to minimize the loss $1/n\sum_{i=1}^n\Lcal_i$. At test time, we predict a label $y_*$ for a set $X_*$ as:
\begin{gather}
    p_* = \frac{1}{2}\left(\tilde{\sigma}(\rvz_{*,1}) + \tilde{\sigma}(g_\lambda(\rvz_{*,2}) \right), \:  y_* = \one\{p_*\geq \tau \},
\end{gather}
where  $\rvz_{*,1}=\max\{w^\top\phi(\rvx):\rvx\in X_* \}$, $\rvz_{*,2}=f_\theta(X_*)$, $\tilde \sigma$ is the sigmoid function, $\one$ is indicator function and $0<\tau<1$ is threshold tuned on the validation set.

\textbf{Setup.}
We evaluate all models on the Camelyon16 Whole Slide Image cancer detection dataset~\citep{camelyon16}. Each instance consists of a high resolution image of tissue from a medical scan which is pre-processed into $256\times256$ patches of RGB pixels. After pre-processing, the average number of patches in a single set is over 9,300 (7.3GB), making each input roughly equivalent to processing \emph{1\% of ImageNet1k}~\citep{imagenet}. The largest input in the training set contains 32,382 patches (25.4 GB). We utilize a ResNet18~\citep{resnet} which is pretrained on Camelyon16~\citep{ds-mil} via SimCLR~\citep{simclr} as a backbone feature extractor whose weights can be downloaded from \href{https://github.com/binli123/dsmil-wsi}{this repository}\footnote{\href{https://github.com/binli123/dsmil-wsi}{https://github.com/binli123/dsmil-wsi}}. Our goal is to first pretrain MBC set encoders on the extracted features, and then use the unbiased estimation of the full set gradient to fine-tune the feature extractor on the full input sets. We evaluate the performance of UMBC against non-MBC MIL baselines: DS-MIL~\citep{ds-mil} and AB-MIL~\citep{ab-mil}, as well as MBC baselines: DeepSets and SSE.

\textbf{Results.} As shown in~\cref{tab:camelyon}, our UMBC model achieves the best accuracy and competitive AUROC score. Note that SSE shows the worst performance due to its constrained architecture, which even underperforms DeepSets in this task. These empirical results again verify the expressive power of our UMBC model. Moreover, we can further improve the performance of UMBC via fine-tuning the backbone network, ResNet18, which is only feasible as a consequence of our unbiased full set gradient approximation which incurs constant memory overhead. However, it is not possible for the non-MBC models to fine-tune with the ResNet18 since it is computationally prohibitive to compute the gradient of the ResNet18 with sets consisting of tens of thousands of patches with $256\times 256$ resolution.

\vspace{-0.1in}
\subsection{Ablation Study}
\vspace{-0.05in}
To validate effectiveness of activation functions $\sigma$ for attention in~\eqref{eq:sse-module}, we train UMBC + Set Transformer with different activation functions listed in~\cref{tab:attn-acts} for the amortized clustering and MIL pretraining tasks. As shown in~\cref{mvn-attn,mil-attn}, softmax attention outperforms all the other activation functions whereas the slot-sigmoid used for attention in SSE underperforms. This experiment highlights the importance of choosing the proper activation function for attention, which is enabled by our UMBC framework.
\input{figures/ablation.tex}

%% file: figures/mvn.tex
\begin{figure*}[t]
\centering
\vspace{-0.05in}
    \begin{subfigure}{0.5\textwidth}
        \centering
        \includegraphics[width=\linewidth]{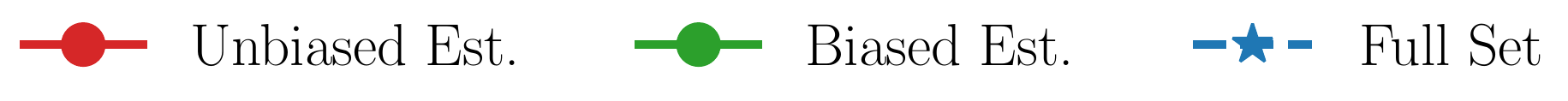}
        \vspace{-0.245in}
    \end{subfigure} \\
    \begin{subfigure}{0.3\textwidth}
        \centering
        \includegraphics[width=\linewidth]{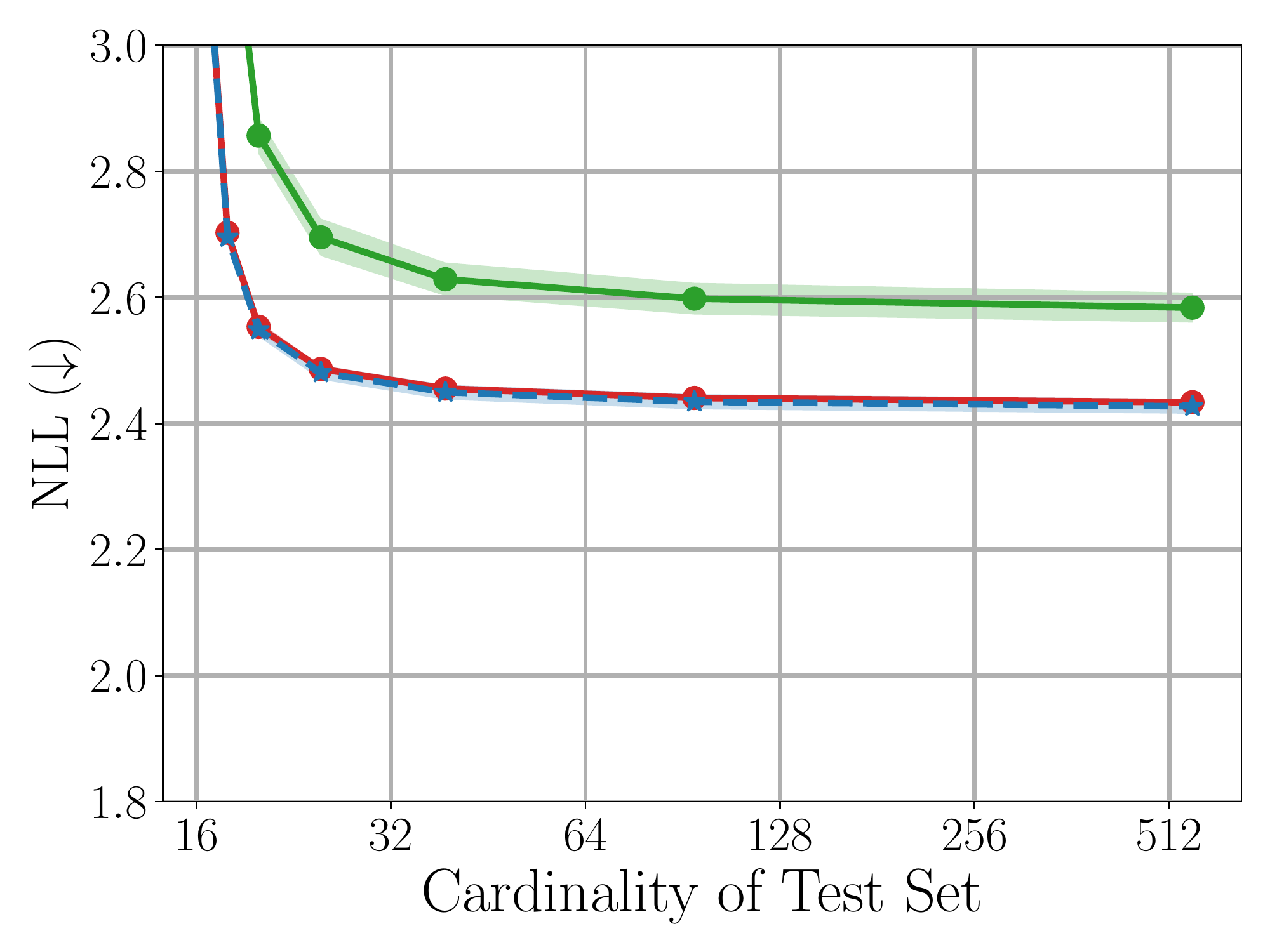}
        \captionsetup{justification=centering,margin=0.5cm}		
        \vspace{-0.25in}
        \caption{\small DeepSets}
        \label{mvn-deepset}	
    \end{subfigure}
    \begin{subfigure}{0.3\textwidth}
	\centering
	\includegraphics[width=\linewidth]{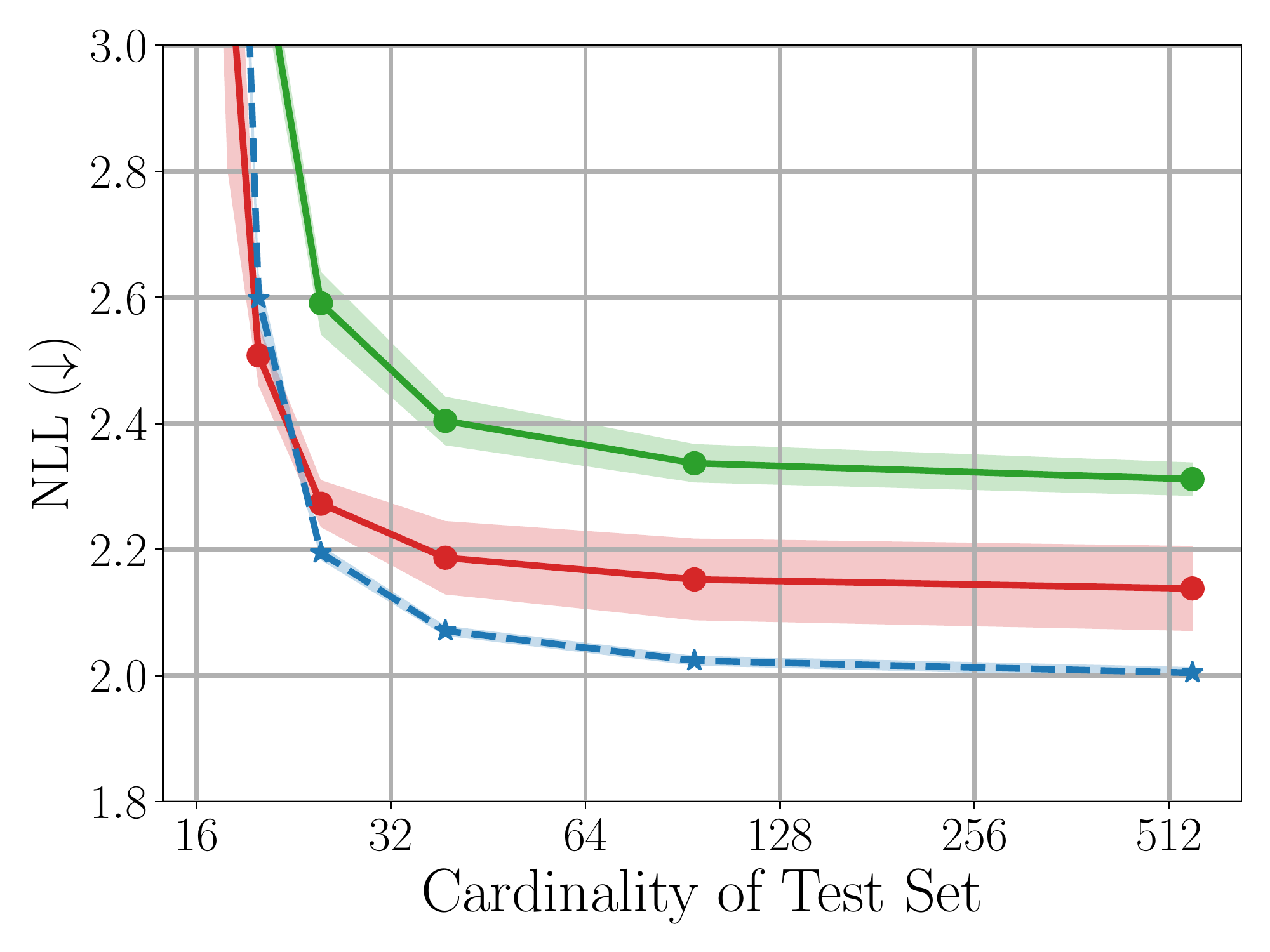}
	\captionsetup{justification=centering,margin=0.5cm}
	\vspace{-0.25in}
	\caption{\small SSE} 
	\label{mvn-sse}
    \end{subfigure}%
    \begin{subfigure}{0.3\textwidth}
	\centering
	\includegraphics[width=\linewidth]{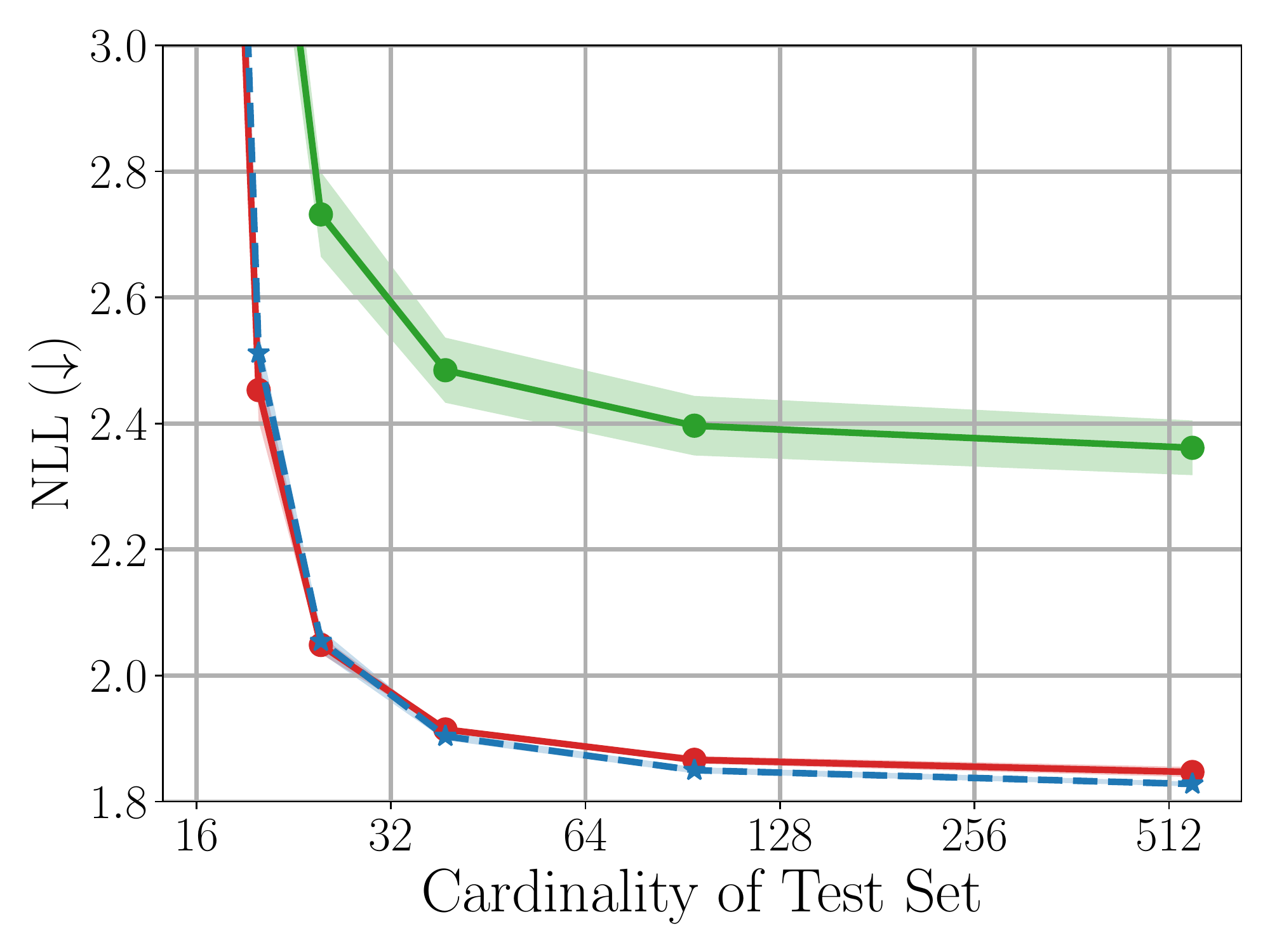}
	\captionsetup{justification=centering,margin=0.5cm}
	\vspace{-0.25in}
	\caption{\small UMBC + ST (Ours)}
	\label{mvn-umbc}
    \end{subfigure}%
    \vspace{-0.18in}
    \caption{\small Performance of 
    \textbf{(a)} DeepSets \textbf{(b)} SSE, and \textbf{(c)} UMBC with varying set sizes for amortized clustering on Mixtures of Gaussians. The unbiased estimate of the full set gradient outperforms the biased estimate, and is usually indistinguishable from the full set gradient. 
    }
    \label{mvn}
    \vspace{-0.22in}
\end{figure*}

%% file: figures/mvn_table.tex
\begin{table}[t]
\vspace{-0.1in}
    \centering
    \caption{\small Clustering: NLL, mem. usage and wall clock time.}
    \resizebox{0.8\linewidth}{!}{
    \begin{tabular}{lcccc}
    \toprule
    Model & MBC & NLL($\downarrow$) & Memory (Kb) & Time (Ms)\\
    \midrule
    DeepSets & \cmark &$2.43\pm$\tiny$0.004$ & $16$ & $0.46\pm$\tiny$0.07$ \\
    SSE & \cmark& $2.13\pm$\tiny$0.067$ & $61$ & $0.83\pm$\tiny$0.08$ \\
    SSE (Hierarchical) & \cmark& $2.38\pm$\tiny$0.057$ & $125$ & $0.84\pm$\tiny$0.06$ \\
    FSPool &\xmark & $3.52\pm$\tiny$0.192$ & $43$ & $0.79\pm$\tiny$0.08$ \\
    Diff EM &\xmark & $5.58\pm$\tiny$0.966$ & $476$ & $7.11\pm$\tiny$0.31$ \\
    Set Transformer (ST)&\xmark & $11.6\pm$\tiny$2.180$ & $225$ & $2.26\pm$\tiny$0.135$ \\
    \midrule
    \textbf{UMBC} + FSPool &\cmark& $2.01\pm$\tiny$0.027$ & $70$ & $1.18\pm$\tiny$0.10$ \\
    \textbf{UMBC} + Diff EM &\cmark & $2.13\pm$\tiny$0.084$ & $502$ & $8.57\pm$\tiny$0.33$ \\
    \textbf{UMBC} + ST &\cmark & $\textbf{1.84}\pm$\tiny$0.008$ & $100$ & $1.63\pm$\tiny$0.11$ \\
    \bottomrule
    \end{tabular}
    }
    \label{tab:mvn-memory}
\vspace{-0.25in}
\end{table}



%% file: figures/celeba.tex
\begin{figure*}
\vspace{-0.05in}
\centering
    \begin{subfigure}{0.9\textwidth}
        \centering
        \includegraphics[width=\linewidth]{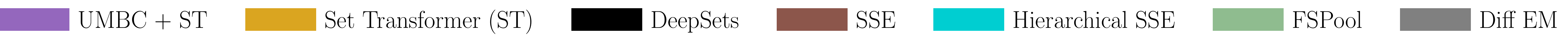}
        \vspace{-0.2in}
    \end{subfigure} \\
    \begin{subfigure}{0.3\textwidth}
		\centering
		\includegraphics[width=\linewidth]{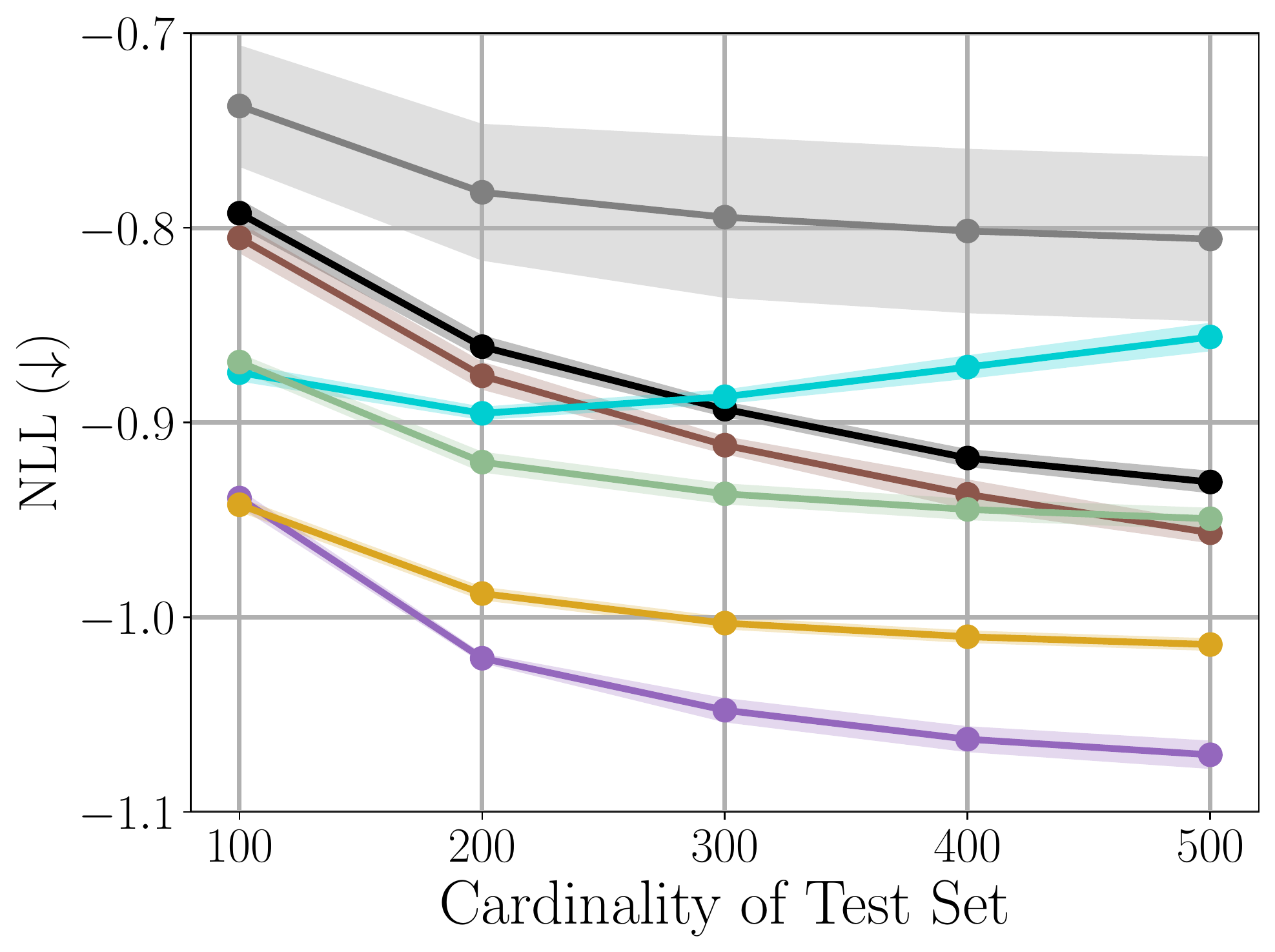}
		\captionsetup{justification=centering,margin=0.5cm}
		\vspace{-0.25in}
		\caption{\small}
		\label{fig:celeba-all-model}
	\end{subfigure}%
    \begin{subfigure}{0.3\textwidth}
    \centering
    \includegraphics[width=\linewidth]{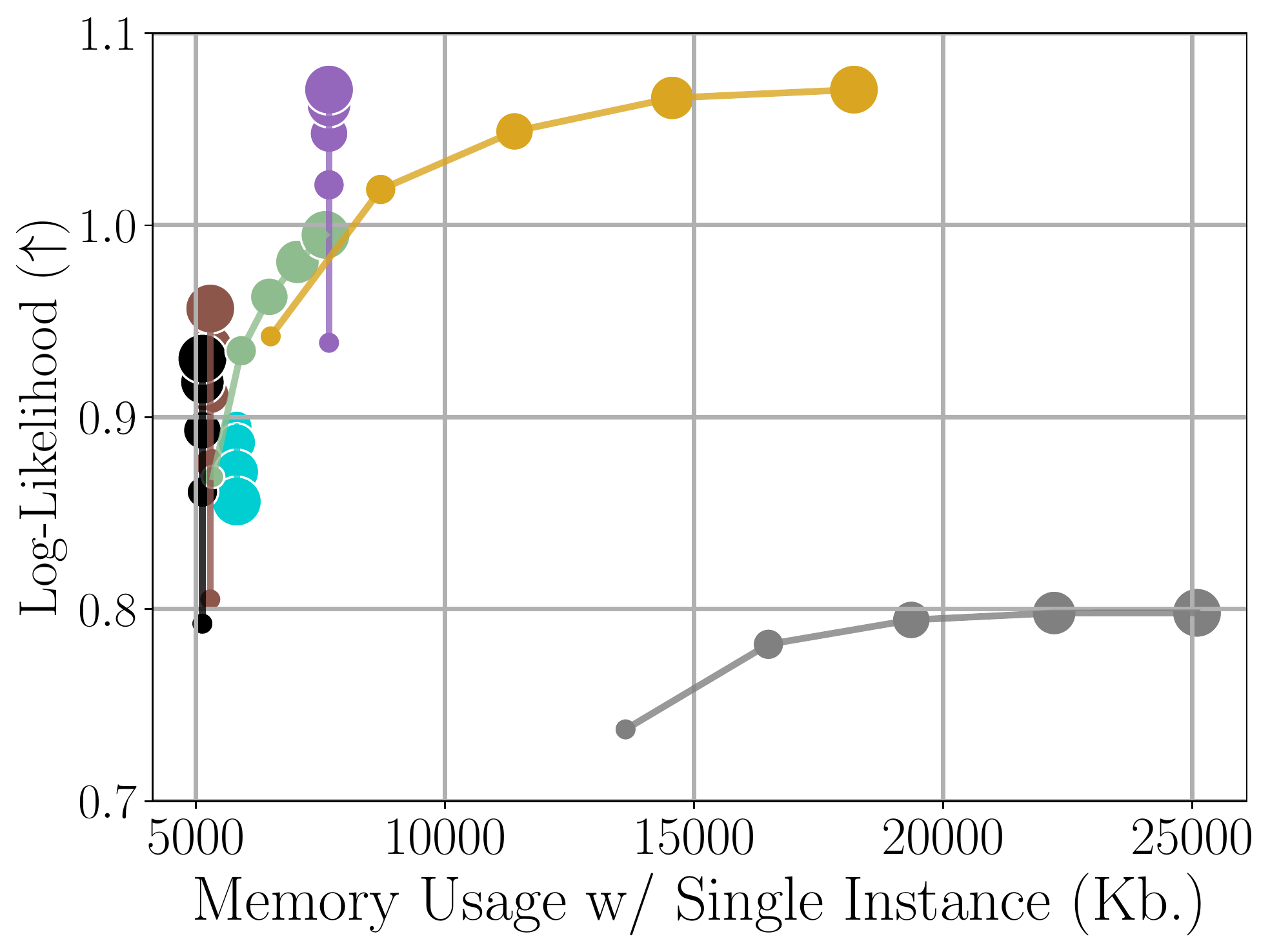}
    \captionsetup{justification=centering,margin=0.5cm}		
    \vspace{-0.25in}
    \caption{\small}
    \label{fig:celeba-mem}	
    \end{subfigure}
	\begin{subfigure}{0.3\textwidth}
		\centering
		\includegraphics[width=\linewidth]{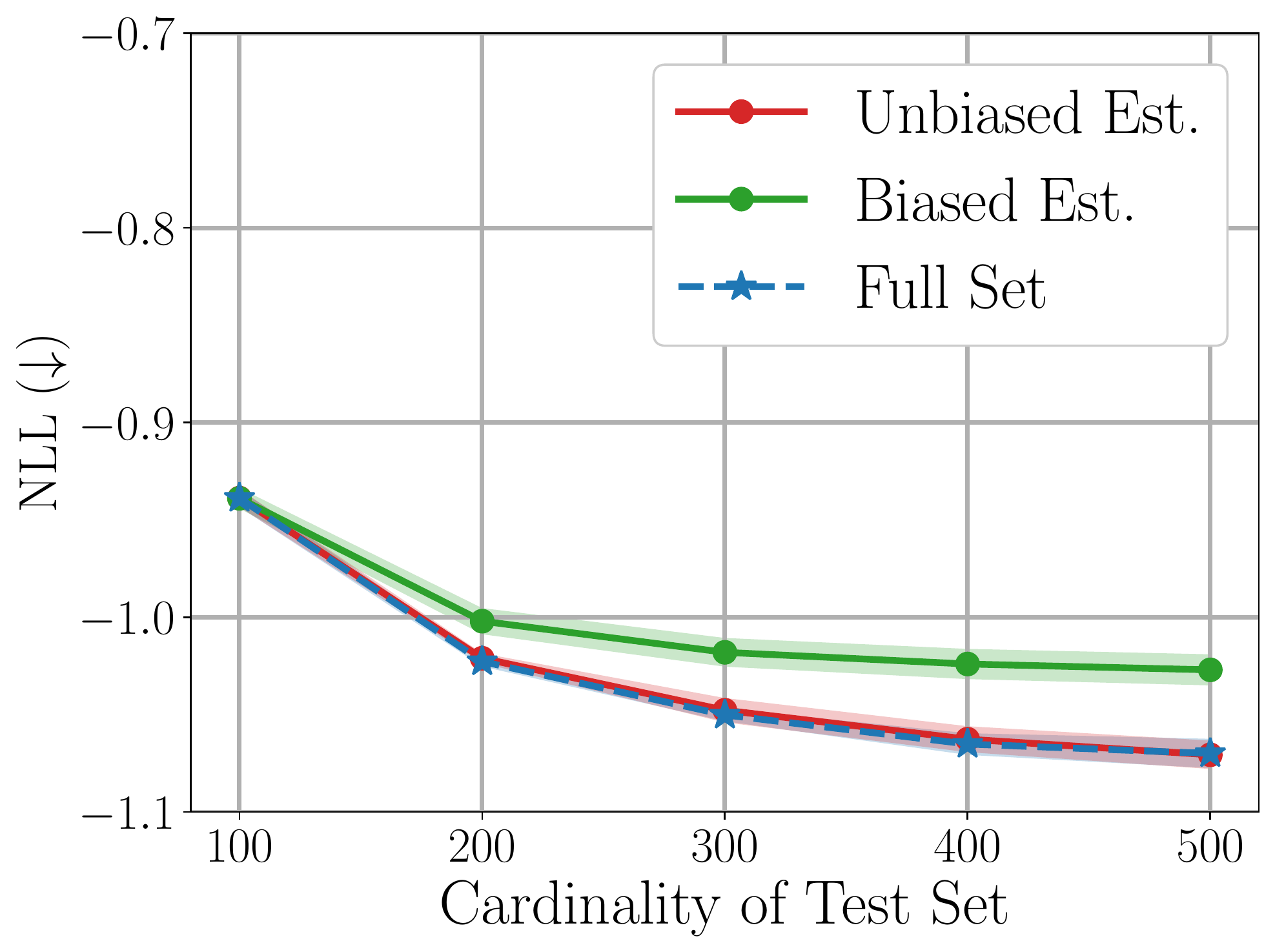}
		\captionsetup{justification=centering,margin=0.5cm}
		 \vspace{-0.25in}
		\caption{\small} 
		\label{fig:celeba-umbc}
	\end{subfigure}%
\vspace{-0.18in}	
 \caption{\small
	\textbf{(a)} Comparison of different models with varying sizes of sets. \textbf{(b)} Memory usage for models to process \emph{a single set} with different cardinalities  denoted as the size of the circles. MBC models with our unbiased full set gradient approximation can process any set size with a \textbf{constant memory} overhead.   \textbf{(c)} The effect of our algorithm compared to training with a small random subset and the full set.
	}
	\label{fig:celeba-exp}
\vspace{-0.2in}
\end{figure*}

%% file: figures/text_table.tex
\begin{table}[t]
\vspace{-0.1in}
    \centering
    \caption{\small Micro F1 score and  memory usage on Quadro RTX8000 for long document classification with inverted EURLEX dataset.}
    \resizebox{0.8\linewidth}{!}{
    \begin{tabular}{lccc}
        \toprule
        Model & F1 & MBC  & Memory (MB)\\
        \midrule
        Longformer & $56.47$ \tiny$\pm0.43$ & \xmark & $25\text{,}185$\\
        ToBERT & $67.11$ \tiny$\pm0.87$ & \xmark & $38\text{,}563$\\
        \midrule
        DeepSets w/ 100 &$59.97$\tiny$\pm0.59$ & \cmark & $1\text{,}295$\\
        DeepSets w/ full &$60.82$\tiny$\pm0.58$ & \cmark &$7\text{,}317$\\
        \midrule
        SSE w/ 100 & $67.60$\tiny$\pm0.17$ & \cmark & $1\text{,}319$\\
        SSE w/ full & $67.91$\tiny$\pm0.33$ & \cmark & $6\text{,}799$\\
        \midrule
        \textbf{UMBC} + BERT w/ 100  & $\textbf{70.48}$\tiny$\pm\textbf{0.23}$ &\cmark & $4\text{,}909$\\
        \textbf{UMBC} + BERT w/ full &  $\textbf{70.23}$\tiny$\pm\textbf{0.84}$ &\cmark & $11\text{,}497$\\
        \bottomrule
    \end{tabular}
    }
    \label{tab:text-exp}
\vspace{-0.25in}
\end{table}

%% file: figures/text_fig.tex
\begin{figure}[t]
    \centering
    \includegraphics[width=0.55\linewidth]{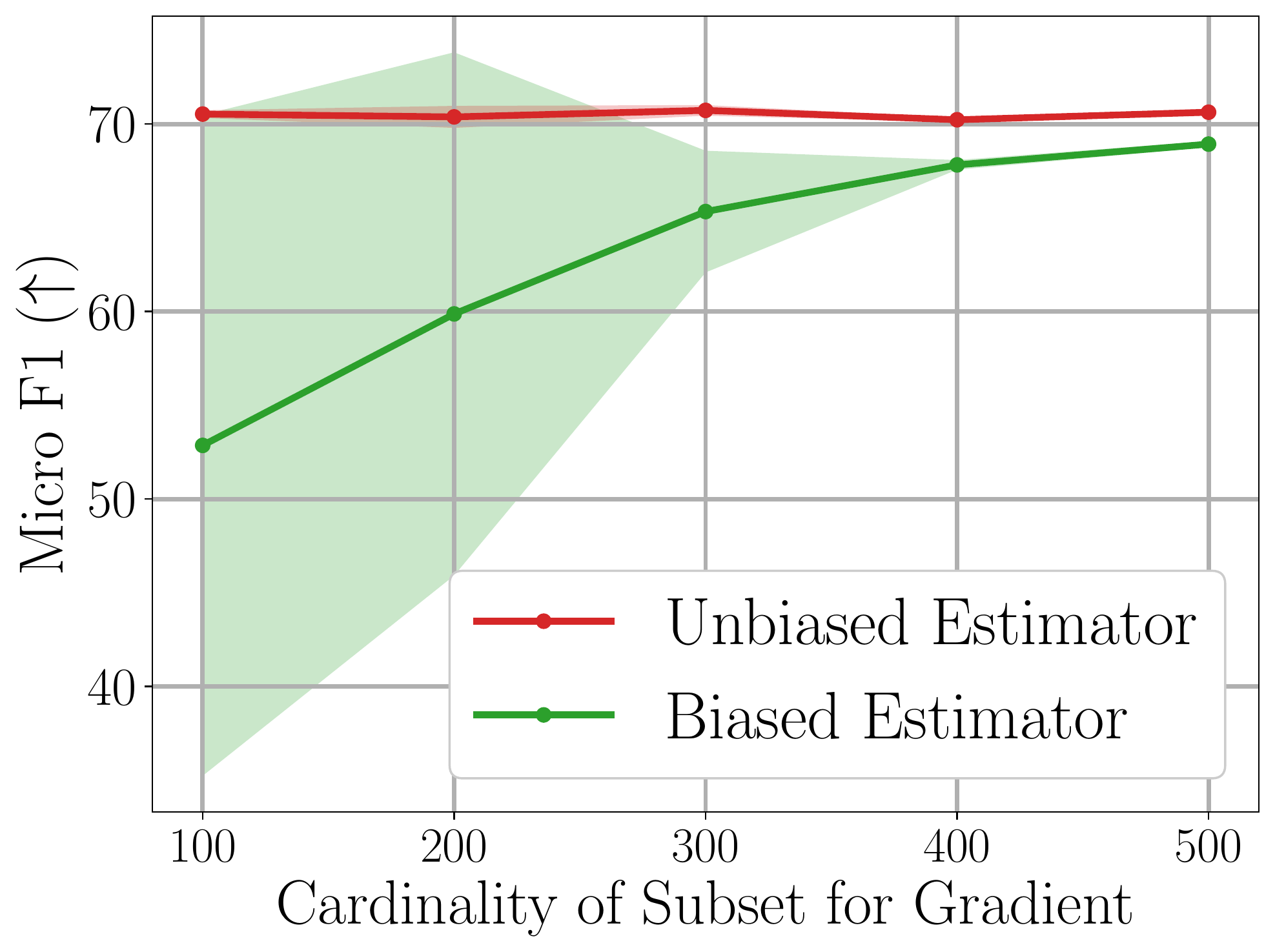}
    \vspace{-0.18in}
    \caption{\small F1 score with varying the size of subset for gradient computation for our method and random sampling.}
    \label{fig:umbc_text}
\vspace{-0.25in}
\end{figure}

%% file: figures/camelyon_table.tex

\begin{table}[t]
    \centering
    \caption{Acc. and AUC of each model on the Camelyon16 dataset.}
    \resizebox{0.9\linewidth}{!}{
    \begin{tabular}{lccc|cc}
        \toprule
        Model & MBC & Accuracy & AUROC & Accuracy & AUROC \\
        \midrule
        \multicolumn{2}{l}{} & \multicolumn{2}{c}{Pretrain} & \multicolumn{2}{c}{MBC Finetune} \\
        \midrule
        DS-MIL & \xmark & $86.36$\tiny$\pm 0.88$ & $0.866$\tiny$\pm0.00$ & - & - \\
        AB-MIL & \xmark & $86.82$\tiny$\pm0.00$ & \underline{$0.884$\tiny$\pm0.01$} & - & - \\
        \midrule
        DeepSets & \cmark & $82.02$\tiny$\pm0.65$ & $0.848$\tiny$\pm0.01$ & $82.02$\tiny$\pm0.65$ & $0.848$\tiny$\pm0.01$ \\
        SSE & \cmark & $74.73$\tiny$\pm1.04$ & $0.748$\tiny$\pm0.03$ & $74.57$\tiny$\pm1.27$ & $0.755$\tiny$\pm0.04$ \\
        \textbf{UMBC} + ST & \cmark & \underline{$87.91$\tiny$\pm1.41$} & $0.874$\tiny$\pm0.01$ & $\mathbf{88.84}$\tiny$\pm\mathbf{0.88}$ & $\mathbf{0.892}$\tiny$\mathbf{\pm0.01}$ \\
        \bottomrule
    \end{tabular}
    }
    \label{tab:camelyon}
\vspace{-0.1in}
\end{table}

%% file: figures/ablation.tex
\begin{figure}
  \begin{minipage}[b]{0.54\linewidth}
    \centering
    \includegraphics[width=0.99\linewidth]{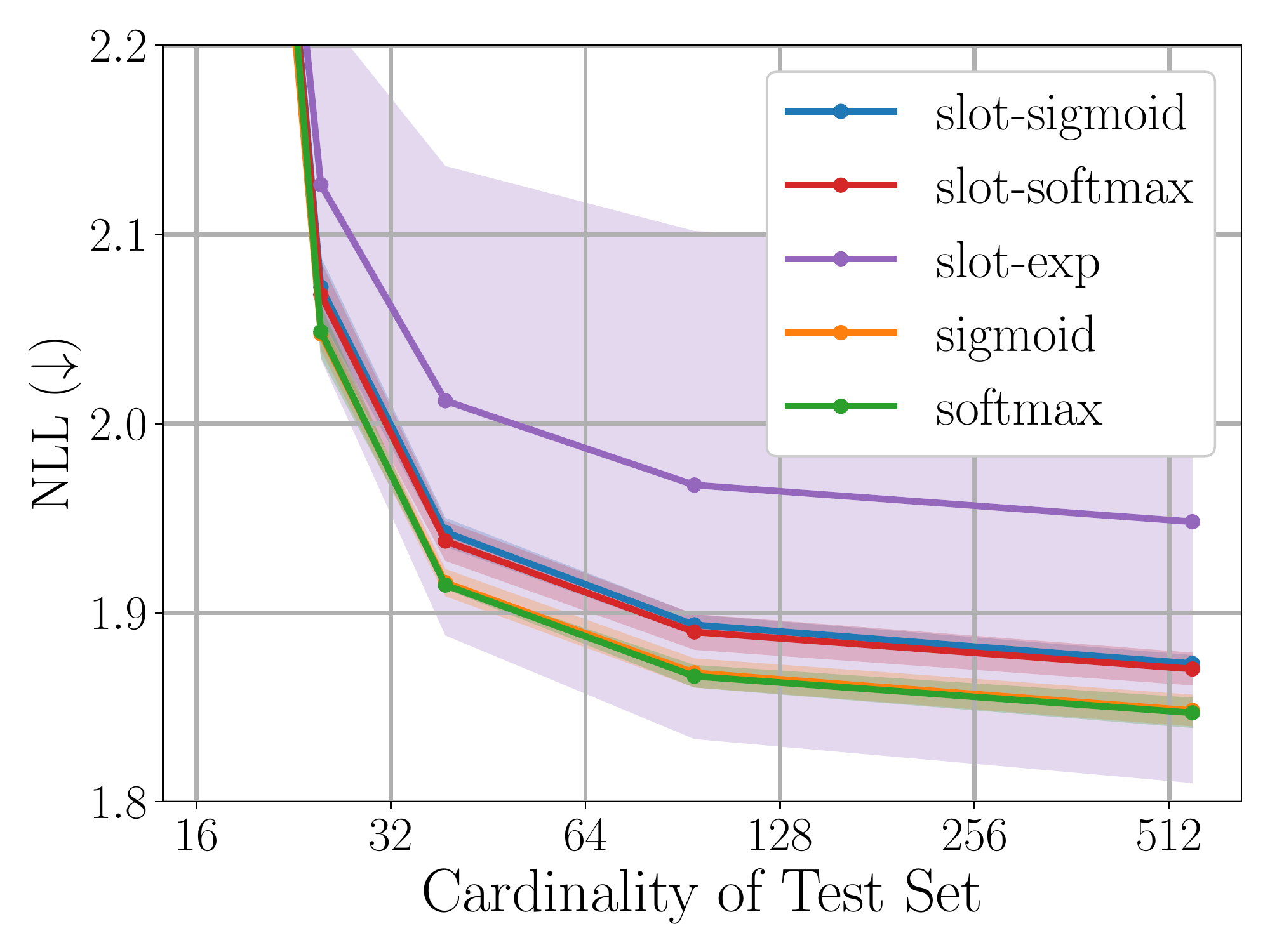}
  \vspace{-0.35in}
  \caption{$\sigma$ for amortized clustering}
  \label{mvn-attn}
  \end{minipage}
  \begin{minipage}[b]{0.45\linewidth}
    \centering
    \resizebox{\linewidth}{!}{
\begin{tabular}[b]{ lcc }
  \toprule
  Activation & Acc. & AUC \\
  \midrule
  slot-sigmoid & $83.72$\tiny$\pm1.45$ & $0.846$\tiny$\pm0.01$ \\
  slot-exp & $83.26$\tiny$\pm0.42$ & $0.850$\tiny$\pm0.01$ \\
  sigmoid & $81.71$\tiny$\pm4.09$ & $0.847$\tiny$\pm0.03$ \\
  slot-softmax & $84.81$\tiny$\pm0.04$ & $0.848$\tiny$\pm0.01$ \\
  softmax & $\textbf{87.91}$\tiny$\pm\textbf{1.41}$ & $\textbf{0.874}$\tiny$\pm\textbf{0.01}$ \\
  \bottomrule
\end{tabular}
}
\vspace{-0.2in}
\captionof{table}{$\sigma$ for MIL}
\label{mil-attn}
\end{minipage}
\vspace{-0.3in}
\end{figure}

%
%
%
%

%% file: sections/6_conclusion.tex
\section{Limitations and Future Work}

One potential limitation of our method is a higher time complexity needed for our proposed unbiased gradient estimation during mini-batch training. If $n$ represents the size of a single subset, and $N$ represents the size of the whole set, then the naive sampling strategy of \citet{mbc} has a time complexity of $\mathcal{O}(nk)$ while our full set gradient approximation has a complexity of $\mathcal{O}(Nk)$ during training since we must process the full set. However, we note some things we gain in exchange for this higher time complexity below.

Our unbiased gradient approximation can achieve higher performance than the biased sampling of a single subset as denoted in our experiments (specifically, see~\cref{mvn,fig:celeba-umbc,fig:umbc_text}). Additionally, due to the stop gradient operation in~\cref{eq:stop-grad}, UMBC achieves a constant memory overhead for any training set size. Our experiment in~\cref{fig:celeba-mem} shows a constant memory overhead for SSE and DeepSets only because we apply our unbiased gradient approximation to those models in that experiment. The original models as they were presented in the original works do not have a constant memory overhead. As a result, our method can process huge sets during training, and practically any GPU size can be accommodated by adjusting the size of the gradient set. For example, the average set in the experiment on Camelyon16 contains 9329 patches (7.3 GB), while the largest input in the training set contains 32,382 patches (25.46 GB). Even though the set sizes are large, models can be trained on all inputs using a single 12GB GPU due to the constant memory overhead. 

Another potential limitation is that UMBC (and SSE) use a cross attention layer with parameterized slots in order to achieve an MBC model. However, the fixed parameters, which are independent to the input set, can be seen as a type of bottleneck in the attention layer which is not present in traditional self-attention. Therefore we look forward to seeing future work which may find ways to make the slot parameters depend on the input set, which may increase overall model expressivity. 
  
\section{Conclusion}
In order to overcome the limited expressive power and training scalability of existing MBC set functions, such as DeepSets and SSE, we have proposed Universal MBC set functions that allow mixing both MBC and non-MBC components to leverage a broader range of architectures which increases model expressivity while universally maintaining the MBC property. Additionally, we generalized MBC attention activation functions, showing that many functions, including the softmax, are MBC. Furthermore, for training scalability, we have proposed an unbiased approximation to the full set gradient with a constant memory overhead for processing a set of any size. Lastly, we have performed extensive experiments to verify the efficiency and efficacy of our scalable set encoding framework, and theoretically shown that UMBC is a universal approximator of continuous permutation invariant functions and converges to stationary points of the total loss with the full set.

\section*{Acknowledgements}
This work was supported by Institute of Information \& communications Technology Planning \& Evaluation (IITP) grant funded by the Korea government(MSIT)  (No.2019-0-00075, Artificial Intelligence Graduate School Program(KAIST)), 
the Engineering Research Center Program through the National Research Foundation of Korea (NRF) funded by the Korean Government MSIT (NRF-2018R1A5A1059921), 
Institute of Information \& communications Technology Planning \& Evaluation (IITP) grant funded by the Korea government(MSIT) (No. 2021-0-02068, Artificial Intelligence Innovation Hub), 
Institute of Information \& communications Technology Planning \& Evaluation (IITP) grant funded by the Korea government(MSIT) (No. 2022-0-00184, Development and Study of AI Technologies to Inexpensively Conform to Evolving Policy on Ethics), 
Institute of Information \& communications Technology Planning \& Evaluation (IITP) grant funded by the Korea government(MSIT) (No.2022-0-00713), and
Samsung Electronics (IO201214-08145-01)

%% file: sections/7_appendix.tex
\appendix

\section{Proofs}\label{app:proofs}
\subsection{Proof of Theorem~\ref{thm:perm-inv}}\label{app:proof-theorem-perm-inv}
\begin{proof}
Let $\mathfrak{S}_N$ be a set of all permutations on $\{1, \ldots, N\}$ and let $\pi \in \mathfrak{S}_N$ be given. For a given set $X=[\rvx_1 \cdots \rvx_N]^\top\in\RR^{N\times d_x}$, and permutation $\pi$, we can construct a permutation matrix $P\in\mathbb{R}^{N\times N}$ such that 
$$
PX =  \begin{bmatrix}
    - \rvx_{\pi(1)}^\top - \\
    \vdots \\
    -\rvx_{\pi(N)}^\top-
\end{bmatrix}. 
$$
Since we apply the feature extractor $\phi$ independently to each element of the set $X$,
$$
\Phi(PX) =  \begin{bmatrix}
    - \phi(\rvx_{\pi(1)})^\top - \\
    \vdots \\
    -\phi(\rvx_{\pi(N)})^\top-
\end{bmatrix} = P\Phi(X). 
$$

With an elementwise, strictly positive activation function $\sigma$,
\begin{align*}
\sigma(\sqrt{d^{-1}}\cdot Q(\Phi(PX)W^K)^\top) &= \sigma(\sqrt{d^{-1}}\cdot Q(P\Phi(X)W^K)^\top )\\
&= \sigma(\sqrt{d^{-1}}\cdot Q(\Phi(X)W^K)^\top P^\top)\\
&= \sigma(\sqrt{d^{-1}}\cdot QK(X)^\top)P^\top \\
&= \hat{A} P^\top.
\end{align*}

For $p=1$, the un-normalized attention score with the permutation $\pi$ is 
$$
\nu_1(\hat{A}P^\top)= \begin{bmatrix}
    \hat{A}_{1, \pi(1)}/\sum_{i=1}^k\hat{A}_{i, \pi(1)} & \cdots &\hat{A}_{1,\pi(N)} / \sum_{i=1}^k \hat{A}_{i, \pi(N)} \\
    \vdots & \ddots & \vdots \\
    \hat{A}_{k, \pi(1)}/\sum_{i=1}^k\hat{A}_{i, \pi(1)} & \cdots &\hat{A}_{k,\pi(N)} / \sum_{i=1}^k \hat{A}_{i, \pi(N)}
\end{bmatrix} = \nu_1(\hat{A})P^\top.
$$
Since $\nu_2$ is the identity mapping, $\nu_p(\hat{A}P^\top)= \nu_p(\hat{A})P^\top$ for $p=1,2$. 

Now, we consider the matrix multiplication of  
$$\nu_p(\hat{A})P^\top=\begin{bmatrix}
    \nu_p(\hat{A})_{1, \pi(1)} & \cdots & \nu_p(\hat{A})_{1, \pi(N)}  \\
    \vdots  & \ddots & \vdots \\
    \nu_p(\hat{A})_{k, \pi(1)} & \cdots & \nu_p(\hat{A})_{k, \pi(N)} 
\end{bmatrix} \text{ and } P\Phi(X)W^V = \begin{bmatrix}
    \phi(\rvx_{\pi(1)})^\top W^V \\
    \vdots \\
     \phi(\rvx_{\pi(N)})^\top W^V
\end{bmatrix}.
$$

Since
\begin{align*}
\begin{bmatrix}
    \nu_p(\hat{A})_{1, \pi(1)} & \cdots & \nu_p(\hat{A})_{1, \pi(N)}  \\
    \vdots  & \ddots & \vdots \\
    \nu_p(\hat{A})_{k, \pi(1)} & \cdots & \nu_p(\hat{A})_{k, \pi(N)} 
\end{bmatrix} 
\begin{bmatrix}
    \phi(\rvx_{\pi(1)})^\top W^V \\
    \vdots \\
    \phi(\rvx_{\pi(N)})^\top W^V 
\end{bmatrix}
&=\begin{bmatrix}
    \sum_{j=1}^N \nu_p(\hat{A})_{1,\pi(j)} \phi(\rvx_{\pi(j)})^\top W^V \\
    \vdots \\
    \sum_{j=1}^N \nu_p(\hat{A})_{k,\pi(j)} \phi(\rvx_{\pi(j)})^\top W^V \\
\end{bmatrix}\\
&=\begin{bmatrix}
    \sum_{j=1}^N \nu_p(\hat{A})_{1,j} \phi(\rvx_j)^\top W^V \\
    \vdots \\
    \sum_{j=1}^N \nu_p(\hat{A})_{k,j} \phi(\rvx^\top_j) W^V \\
\end{bmatrix}\\
&= \nu_p(\hat{A})V(X),
\end{align*}

$\hat{f}_\theta$ is permutation invariant. Since $
\bbf_\theta(X)_i=\sum_{j=1}^N \nu_p(\hat{A})_{i,j}=\sum_{j=1}^N\nu_p(\hat{A})_{i, \pi(j)}$, $\text{diag}\left(\bbf_\theta(X)\right)^{-1}$ is also invariant with respect to the permutation of input $X$, which leads to the conclusion that $f_\theta(PX) = \text{diag}\left(\bbf_\theta (PX)\right)^{-1}\hat{f}_\theta(PX) = \text{diag}(\bbf_\theta(X))^{-1}\hat{f}_\theta(X).$\\
$\therefore f_\theta$ is permutation invariant. 
\end{proof}

\subsection{Proof for Theorem~\ref{thm:mbc}} \label{app:proof:mbc}
\label{proof:mbc}
\begin{proof}
Let input set $X\in\mathbb{R}^{N\times d_x}$ be given and let $\zeta(X)=\{S_1, \ldots, S_l\}$ be a partition of $X$ with $|S_i| = N_i$, \ie, $X=\bigcup_{i=1}^lS_i$ and $S_i\cap S_j = \emptyset$ for $i\neq j$. Since a Universal MBC set encoder is permutation invariant, without loss of generality we can assume that,
\begin{align}
    K(X) = \begin{bmatrix}
        K(X)_1 \\
        \vdots \\
        K(X)_l
    \end{bmatrix}, \quad
    V(X) = \begin{bmatrix}
        V(X)_1 \\
        \vdots \\
        V(X)_l
    \end{bmatrix}
\end{align}
where $K(X)_i = \Phi(S_i) W^K\in\mathbb{R}^{N_i \times d}$ and $V(X)_i = \Phi(S_i) W^V\in\mathbb{R}^{N_i \times d}$ for $i=1\ldots, l$. 
Then we can express the matrix $\nu_p(\hat{A})$ as follows:
\begin{align}
    \nu_p(\hat{A}) =  \begin{bmatrix}
        \nu_p(\hat{A}^{(1)})
        \cdots 
        \nu_p(\hat{A}^{(l)})
    \end{bmatrix} ,
\end{align}
where $\hat{A}^{(i)} =  \sigma(\sqrt{d^{-1}}\cdot QK(X)^\top_i)\in \mathbb{R}^{k\times N_i}$ for $i=1\ldots, l$ since $\nu_p(\hat{A})_{i,j}$ is independent to $\nu_p(\hat{A})_{i,q}$ for all $q\neq j$. 

Since 
$$
\hat{f}_\theta (X) =\begin{bmatrix}
        \nu_p(\hat{A}^{(1)}) \cdots \nu_p(\hat{A}^{(l)})
    \end{bmatrix}
    \begin{bmatrix}
        V(X)_1 \\
        \vdots \\
        V(X)_l
    \end{bmatrix},
$$
the following equality holds
\begin{align}
\begin{split}
\hat{f}_\theta(X) &= \begin{bmatrix}
        \nu_p(\hat{A}^{(1)}) \cdots \nu_p(\hat{A}^{(l)})
    \end{bmatrix}
    \begin{bmatrix}
        V(X)_1 \\
        \vdots \\
        V(X)_l
    \end{bmatrix} \\
    &= \sum_{i=1}^l \nu_p(\hat{A}^{(i)})V(X)_i \\
    &= \sum_{i=1}^l\hat{f}_\theta (S_i).
\end{split}
\label{eq:f-hat-mbc}
\end{align}
Thus, $\hf_\theta(X)$ is mini-batch consistent.

Since 
$$
\bbf_\theta(X)_i = \sum_{q=1}^l \sum_{j=1}^{N_i} \nu_p(\hat{A}^{(q)})_{i,j},
$$
we can decompose $\bbf_\theta(X)$ into a summation of $\bbf_\theta(S_i)$ as 
\begin{align}
\begin{split}
\bbf_\theta(X) &= \sum_{i=1}^l\left(\sum_{j=1}^{N_i}\nu_p(\hat{A}^{(i)})_{1,j}, \ldots, \sum_{j=1}^{N_i}\nu_p(\hat{A}^{(i)})_{k,j}  \right)^\top \\
&= \sum_{i=1}^l\nu_p(\hat{A}^{(i)})\1_{N_i}  \\
&= \sum_{i=1}^l\bbf_\theta(S_i),
\label{eq:f-bar-mbc}
\end{split}
\end{align}
where $\1_{N_i}=(1,\ldots, 1)\in \RR^{N_i}$. It implies that $\bbf_\theta(X)$ is mini-batch consistent (MBC). Combining~\eqref{eq:f-hat-mbc} and~\eqref{eq:f-bar-mbc}
$$
f_\theta(X) = \text{diag}\left(\sum_{i=1}^l\bbf_\theta(S_i)\right)^{-1} \left(\sum_{i=1}^l\hf_\theta(S_i) \right).
$$
Now, we define a function, 
\begin{align*}
h(\{f_\theta(S_1), \ldots, f_\theta(S_l) \} ) \coloneqq h_1(\{\bbf_\theta(S_1), \ldots, \bbf_\theta(S_l) \}) \cdot h_2(\{\hf_\theta(S_1), \ldots, \hf_\theta(S_l) \} ),   
\end{align*}
where 
\begin{gather*}
    h_1(\{\bbf_\theta(S_1), \ldots \bbf_\theta(S_l)\} \coloneqq \text{diag}\left(\sum_{i=1}^l\bbf_\theta(S_i)\right)^{-1} \\
     h_2(\{\hf_\theta(S_1), \ldots, \hf_\theta(S_l) \} ) \coloneqq \sum_{i=1}^l \hf_\theta(S_i).
\end{gather*}
Then $f_\theta(X) = h(\{f_\theta(S_1), \ldots, f_\theta(S_l) \} )$. Since $\zeta(X)$ is arbitrary, $f_\theta$ is mini-batch consistent.
\end{proof}

\subsection{Proof of Corollary~\ref{cor:mbc-no-mbc}}\label{app:proof-mbc-no-mbc}
\begin{proof}
Let $\hat{g}_\omega: \RR^{k\times d} \to \Zcal$ be an arbitrary set encoder and let $f_\theta : \Xcal \to \RR^{k\times d}$ be a UMBC set encoder. Given a set $X\in\Xcal$ and a partition $\zeta(X)$, we get 
$$f_\theta(X)= \text{diag}\left(\sum_{S\in\zeta(X)}\bbf_\theta(S)\right)^{-1} \left(\sum_{S\in\zeta(X)} \hf_\theta(S) \right)\in\RR^{k\times d},$$     
as shown in section~\ref{proof:mbc}. We assume that $k$ is small enough so that we can load $f_\theta(X)$ in memory after we compute $f_\theta(X)$. As a consequence, we can directly evaluate $\hat{g}_\omega (f_\theta(X))$ without partitioning $f_\theta(X)$ into smaller subsets $\{S_1, \ldots, S_l\}$ and aggregating $\hat{g}_\omega(S_i)$. \\
$\therefore \hat{g}_\omega \circ f_\theta$ is mini-batch consistent.
\end{proof}

\subsection{Proof of Theorem~\ref{thm:perm-eq}}\label{app:proof-perm-equiv}
\begin{proof}
Let $\pi \in\mathfrak{S}_k$ be a permutation on $\{1\ldots, k\}$ and let $\mathfrak{S}_k$ be a set of all permutations on $\{1,\ldots, k\}$.  Define a matrix 
$$
Z = \begin{bmatrix}
    -\rvz^\top_1 -\\
    \vdots \\
    -\rvz^\top_k - 
\end{bmatrix} \coloneqq \varphi(\Sigma;X,\theta) \in \RR^{k\times d}
$$
with the input set $X\in\RR^{N\times d_x}$ and the given slots $\Sigma = [\rvs_1 \cdots \rvs_k]^\top\in\RR^{k\times d_s}$. Then we can identify a permutation matrix $P\in\mathbb{R}^{k\times k}$ such that,
$$
P\Sigma = \begin{bmatrix}
    -\rvs^\top_{\pi(1)}- \\
    \vdots \\
    -\rvs^\top_{\pi(k)}-
\end{bmatrix}.
$$

Since the query $Q$ with the  permutation $P$ is  $P\Sigma W^Q= PQ$, the un-normalized attention score with the permutation matrix $P$ is
$$
\sigma(\sqrt{d^{-1}}\cdot\ PQK(X)^\top) = P\sigma(\sqrt{d^{-1}}\cdot\ QK(X)^\top) = P\hat{A}. 
$$

Since the normalization matrix $\bbf_\theta(X)$ is a function of the slots $\Sigma$, we define a new normalization matrix by permuting the slots $\Sigma$ with the given permutation matrix $P$ as
$$
\text{diag}\left(\bbf_\theta(X); P\Sigma\right)^{-1}=\begin{bmatrix}
    \frac{1}{{c}_1} & & & \\
    & \frac{1}{{c}_2} & & \\
    & & \ddots & \\ 
    & & & \frac{1}{{c}_k}
\end{bmatrix}
$$
where ${c}_i = \sum_{j=1}^N \nu_p(P\hat{A})_{i,j}$ for $i=1,\ldots, k$. Note that 

\begin{align*}
\nu_1(P\hat{A}) &= \begin{bmatrix}
    \hat{A}_{\pi(1), 1} / \sum_{i=1}^k\hat{A}_{\pi(i),1} & \cdots &\hat{A}_{\pi(1), N} / \sum_{i=1}^k\hat{A}_{\pi(i),N} \\
    \vdots & \ddots & \vdots \\
    \hat{A}_{\pi(k), 1} / \sum_{i=1}^k \hat{A}_{\pi(i), 1} & \cdots & \hat{A}_{\pi(k), N} / \sum_{i=1}^k \hat{A}_{\pi(i),N}
\end{bmatrix} \\
&= P\hat{A}\cdot\text{diag}\left(\sum_{i=1}^k\hat{A}_{i,1}, \ldots, \sum_{i=1}^k \hat{A}_{i,N}\right)^{-1} \\
&=P\nu_1(\hat{A})
\end{align*}
and $\nu_2(P\hat{A})=P\nu_2(\hat{A})$.

Then we get ${c}_i = \bbf_\theta(X)_{\pi(i)}$ since
\begin{align*}
{c}_i &=\sum_{j=1}^N\nu_p(P\hat{A})_{i,j} \\
&=\sum_{j=1}^N(P\nu_p(\hat{A}))_{i,j} \\
&= \sum_{j=1}^N\sum_{l=1}^k P_{i,l}\nu_p(\hat{A})_{l,j}  \\
&=\sum_{l=1}^kP_{i,l} \left(\sum_{j=1}^N \nu_p(\hat{A})_{l,j}\right) \\
&= \sum_{l=1}^kP_{i,l} \bbf_\theta(X)_l \\
& = \bbf_\theta(X)_{\pi(i)}.
\end{align*}
The last equality holds since $i$-th row of the permutation matrix $P$ has a single non-zero entry which is 1. 
Thus, 
$$
\begin{bmatrix}
    {{c}_1} & & & \\
    & {{c}_2} & & \\
    & & \ddots & \\ 
    & & & {{c}_k}
\end{bmatrix} 
= \begin{bmatrix}
    {\bbf_\theta(X)}_{\pi(1)} & & & \\
    & {\bbf_\theta(X)}_{\pi(2)} & & \\
    & & \ddots & \\ 
    & & & {\bbf_\theta(X)}_{\pi(k)}
\end{bmatrix} 
=P\text{diag}(\bbf_\theta(X)),
$$
which implies that 
\begin{align*}
\text{diag}\left(\bbf_\theta(X); P\Sigma\right)^{-1}&=\begin{bmatrix}
    {\frac{1}{{c}_1}} & & & \\
    & \frac{1}{{{c}_2}} & & \\
    & & \ddots & \\ 
    & & & \frac{1}{{{c}_k}}
\end{bmatrix} \\
&= \begin{bmatrix}
    {\bbf_\theta(X)}_{\pi(1)} & & & \\
    & {\bbf_\theta(X)}_{\pi(2)} & & \\
    & & \ddots & \\ 
    & & & {\bbf_\theta(X)}_{\pi(k)}
\end{bmatrix}^{-1}. 
\end{align*}

Finally, combining all the pieces, we get  

\begin{align*}
    &\text{diag}\left(\bbf_\theta(X); P\Sigma\right)^{-1}\nu_p(P\hat{A})V(X) \\
    &=\text{diag}\left(\bbf_\theta(X); P\Sigma\right)^{-1}P\nu_p(\hat{A})V(X) \\
    &=  \begin{bmatrix}
    {\bbf_\theta(X)}_{\pi(1)} & & & \\
    & {\bbf_\theta(X)}_{\pi(2)} & & \\
    & & \ddots & \\ 
    & & & {\bbf_\theta(X)}_{\pi(k)}
\end{bmatrix}^{-1}\begin{bmatrix}
    (\nu_p(\hat{A})V(X))_{\pi(1), 1} & \cdots & (\nu_p(\hat{A})V(X))_{\pi(1), d} \\
    \vdots & \ddots & \vdots \\
    (\nu_p(\hat{A})V(X))_{\pi(k),1} & \cdots & (\nu_p(\hat{A})V(X))_{\pi(k), d}
\end{bmatrix} \\
&= \begin{bmatrix}
    (\nu_p(\hat{A})V(X))_{\pi(1), 1} /\bbf_\theta(X)_{\pi(1)} & \cdots & (\nu_p(\hat{A})V(X))_{\pi(1), d} / \bbf_\theta(X)_{\pi(1)} \\
    \vdots & \ddots & \vdots \\
    (\nu_p(\hat{A})V(X))_{\pi(k),1} / \bbf_\theta(X)_{\pi(k)} & \cdots & (\nu_p(\hat{A})V(X))_{\pi(k), d}/\bbf_\theta(X)_{\pi(k)}
\end{bmatrix} \\
&= P \begin{bmatrix}
    {\bbf_\theta(X)}_{1} & & & \\
    & {\bbf_\theta(X)}_{2} & & \\
    & & \ddots & \\ 
    & & & {\bbf_\theta(X)}_{k}
\end{bmatrix}^{-1}\nu_p(\hat{A})V(X) \\
&= P \text{diag}\left(\bbf_\theta(X)\right)^{-1} \hf_\theta(X) \\
&= P \varphi(\Sigma;X,\theta) \\
&= P Z \\
&= \begin{bmatrix}
    -\rvz^\top_{\pi(1)} -\\
    \vdots \\
    -\rvz^\top_{\pi(k)} -
\end{bmatrix}.
\end{align*}
$\therefore \varphi(\Sigma; X,\theta)$ is permutation equivariant.
\end{proof}

\subsection{Proof of Theorem~\ref{thm:univ-approx}}
\begin{proof}
Following the previous proofs~\citep{deepsets, univ-approx} for \emph{uncountable set} $\mathfrak{X}$, we assume a set size is fixed. In other words, we restrict the domain $\Xcal\subset 2^\mathfrak{X}$ to $[0,1]^M$. Let $f\in\mathcal{F}$ be given. By using the proof of Theorem 13 from~\citet{univ-approx} (with a more detailed proof in \citealp{deepsets}), the function $f$ is continuously sum-decomposable via $\RR^{M+1}$ as:
\begin{equation*}
    f(X) = \rho\left(\Psi(X)\right) \\
\end{equation*}
for all $X=\{x_1, \ldots, x_M\} \in [0,1]^M$, where $\Psi:[0,1]^M\to \Qcal \subseteq\RR^{M+1}$ is invertible and defined by
\begin{align*}
\Psi(X) = \left(\sum_{x\in X} \psi_0(x), \ldots, \sum_{x\in X}\psi_M(x) \right),\quad \psi_q(x) = x^q, \quad \text{for $q=0,\ldots, M$},
\end{align*}
and $\rho:\RR^{M+1}\to\RR$ is continuous and defined by
$$
\rho = f\circ \Psi^{-1}.
$$
We want to show that UMBC with some continuously decomposable permutation invariant deep neural network can approximate the function $f$ by showing that $\rho$ and $\Psi$ are approximated by components of the UMBC model. Let $h: [0,1]^M \to\RR$ be a continuously decomposable permutation invariant deep neural network defined by
$$
h(Z) = \kappa\left(\sum_{\rvz \in Z}\xi(\rvz)\right),
$$
where $Z=\{\rvz_i\in\RR^{M+1}\}_{i=1}^k$, $\kappa:\RR^{M+1}\to \RR$ is a deep neural network, and $\xi:\RR^{M+1}\to\RR^{M+1}$ is defined by 
$$\xi(\rvz) = \frac{1}{k}\cdot\hat{\xi}(M\cdot\rvz)
$$
with some deep neural network $\hat{\xi}: \RR^{M+1}\to\RR^{M+1}$.
First, we want to show that Deepsets with average pooling is a special case of UMBC.  Set the slots $\Sigma$ as the zero matrix $\mathbf{0}\in\RR^{k\times d_s}, d=d_h=M+1$, and $W^V =I_{M+1}$. Then by using Lemma 1 from \citet{set-trans}, $f_\theta$ becomes average pooling, \ie,
$$
f_\theta(X) = \begin{bmatrix}
    \frac{1}{M}\sum_{i=1}^M \phi(x_i) \\
    \vdots \\
    \frac{1}{M}\sum_{i=1}^M \phi(x_i)
\end{bmatrix} \in \RR^{k \times (M+1)}.
$$
Then the composition of UMBC and $h$ becomes continuously sum-decomposable function as follows:
\begin{align*}
(h\circ f_\theta)(X) &= \kappa\left(\sum_{j=1}^k \xi(f_\theta(X)_j) \right) \\
&= \kappa\left( \sum_{j=1}^k \frac{1}{k} \hat{\xi}\left( \frac{M}{M}\sum_{i=1}^M \phi(x_i) \right)\right)  \\
&=(\kappa\circ \hat{\xi}) \left( \sum_{i=1}^M \phi(x_i) \right)
\end{align*}
where $f_\theta(X)_j$ is $j$-th row of $f_\theta(X)$. By defining $\hat\rho\coloneqq\kappa\circ \hat{\xi}$ and $\hat\Psi(X)\coloneqq\sum_{x\in X}\phi(x)\in\hat \Qcal \subseteq \RR^{M+1}$,
\begin{align*}
\sup_{X}\norm{f(X) - (h\circ f_\theta)(X)}_2 &= \sup_{X} \norm{(\rho\circ \Psi)(X) - (\hat\rho\circ \hat\Psi)(X) + (\rho\circ \hat\Psi)(X) - (\rho\circ \hat\Psi)(X)}_2
\\ & \le \sup_{X} \norm{(\rho\circ \Psi)(X) - (\rho\circ \hat\Psi)(X)}_2 + \sup_{X}\norm{ (\hat\rho\circ \hat\Psi)(X) - (\rho\circ \hat\Psi)(X)}_2
\\ & \le   \sup_{X} \norm{ \rho(\hat\Psi(X)) - \rho(\Psi(X))}_2 + \sup_{z \in \hat \Qcal} \norm{\rho(z) - \hat\rho(z)}_2.
\end{align*}
Since $\hat \Psi:[0,1]^M\to\hat \Qcal \subseteq \RR^{M+1}$ is continuous and $[0,1]^M$ is compact, $\hat \Qcal $ is compact. Since $\hat \Qcal$ is compact and $\rho$ is continuous, and the nonlinearity of $\hat \rho$ is not a polynomial of finite degree, Theorem 3.1 of \citep{pinkus1999approximation}  implies the following (as a network with one hidden layer can be approximated by a network of greater depth by using the same construction for the first layer and approximating the identity function with later layers): for any $\epsilon'>0$, there exists $\tau' \ge 1$ and parameters of  $\hat \rho$  such that if the width of $\hat \rho$ is at least $\tau'$, then $\sup_{z \in \hat \Qcal} \|\rho(z) - \hat\rho(z)\|<\epsilon'$. Combining these, we have that 
$$
\sup_{X}\norm{f(X) - (h\circ f_\theta)(X)}_2 < \sup_{X} \norm{\rho(\hat\Psi(X)) - \rho(\Psi(X))}_2 + \epsilon'_{(\tau')}
$$ 
where $\epsilon'_{(\tau')}$ depends on the width $\tau'$ of $\hat \rho$. 
Since the nonlinearity of $\phi$ has nonzero Taylor coefficients up to degree $M$, the proof of Theorem 3.4 of \citep{poly-univ-approx} implies the following: there exists $\tau\ge 1$ such that if the width of $\phi=(\phi_0,\phi_1,\dots,\phi_M)$ is at least $\tau$, for every $\delta>0$, there exists parameters of $\phi$ such that $\sup_{x}\lvert\phi_q(x) - \psi_q(x)\rvert < \frac{\delta}{2M(M+1)}$ for $q \in \{0,1,\dots,M\}$. Let us fix the width of $\phi$ to be at least $\tau$. By the triangle inequality,
\begin{align*}
    \norm{\hat{\Psi}(X) - \Psi(X)}_2 &= \norm{\sum_{x\in X}\left(\phi_0(x)-\psi_0(x), \ldots, \phi_M(x)-\psi_M(x)\right) }_2 \\
    &\leq \sum_{x\in X} \norm{\left(\phi_0(x)-\psi_0(x), \ldots, \phi_M(x)-\psi_M(x)\right)}_2 \\
    &\leq \sum_{x\in X} \sum_{q=0}^M \lvert \phi_q(x) - \psi_q(x)\rvert \\
    &\leq \lvert X \rvert \sum_{q=0}^M \sup_x \lvert \phi_q(x) - \psi_q(x)\rvert \\
    &< M(M+1) \frac{\delta}{2M(M+1)} \\
    &= \frac{\delta}{2}
\end{align*}
for all $X\in [0,1]^M$. It implies that  for every $\delta>0$ there exists the parameters of $\phi$ such that
$$
\sup_X \norm{\hat{\Psi}(X) - \Psi(X)}_2 \leq \frac{\delta}{2} < \delta.
$$
Since $\Psi:[0,1]^M\to \Qcal \subseteq\RR^{M+1}$ is continuous and $[0,1]^M$ is compact, $\Qcal$ is compact. 
Since $\Qcal$ and $\hat{\Qcal}$ are compact, $\tilde{\Qcal}\coloneqq\Qcal \cup \hat{\Qcal}$ is compact. 
Define $\tilde \rho: \tilde \Qcal \to\RR$ by $\tilde \rho(z)=\rho(z)$ for all $z \in \tilde \Qcal$. Replacing $\rho$ with $\tilde \rho$, 
$$
\sup_{X}\|f(X) - (h\circ f_\theta)(X)\| < \sup_{X} \|\tilde \rho(\hat\Psi(X)) - \tilde \rho(\Psi(X))\| + \epsilon'_{(\tau')}.
$$ 
Since  $\tilde \Qcal$ is compact and $\tilde \rho$ is continuous on $\tilde \Qcal$, $\tilde \rho$ is uniformly continuous. Thus, for any $\epsilon>0$ there is a $\delta_0$ such that 
$$
\text{for any } z_1, z_2 \in \tilde{\Qcal} \text{ with }\norm{z_1 - z_2}_2 < \delta_0 \Rightarrow \norm{\tilde{\rho}(z_1) - \tilde{\rho}(z_2)}_2 <\epsilon.
$$
Since $\sup_{X}\|\hat\Psi(X)-\Psi(X)\|_2<\delta$ with an arbitrary small $\delta>0$, we can take a small $\delta$ such that $\delta < \delta_0$, \ie $\|\hat\Psi(X)-\Psi(X)\|_2<\delta_0$ for all $X\in [0,1]^M$. 
Then $\text{ for all } X \in [0,1]^M$,
$$
\norm{\tilde \rho(\hat \Psi(X)) - \tilde \rho(\Psi(X))}_2 < \epsilon.
$$
It implies that 
\begin{align*}
\norm{\tilde \rho(\hat \Psi(X)) - \tilde \rho(\Psi(X))}_2 \leq 
\sup_X\norm{\tilde \rho(\hat \Psi(X)) - \tilde \rho(\Psi(X))}_2
\leq\epsilon.
\end{align*}
Thus, we get 
$$
\sup_{X}\norm{f(X) - (h\circ f_\theta)(X)}_2 \leq \epsilon+ \epsilon'_{(\tau')},
$$  
where $\epsilon'_{(\tau')}$ depends on the width $\tau'$ of $\hat \rho$, while $\epsilon>0$ is arbitrarily small with a fixed width of $\phi$ (due to the universal approximation result with a bounded width of \citealp{poly-univ-approx}). Let $\epsilon_0>0$ be given. Then, we set $\tau^\prime$ to be sufficiently large to ensure that $\epsilon'_{(\tau')}<\epsilon_0/2$ and set $\epsilon <\epsilon_0/2$, obtaining
\begin{align*}  
\sup_{X}\norm{f(X) - (h\circ f_\theta)(X)}_2 &< \epsilon_0.
\end{align*} 
Since $\epsilon_0>0$ was arbitrary, this proves the following desired result: 
(a formal restatement of this theorem) suppose that the nonlinear activation function of $\phi$ has nonzero Taylor coefficients up to degree $M$. Let $h$ be a continuously-decomposable permutation-invariant deep neural network with the nonlinear activation functions that are not polynomials of finite degrees. 
Then, there exists $\tau\ge 1$ such that if the width of $\phi$ is at least $\tau$, then for any $\epsilon_0>0$, there exists $\tau' \ge 1$ for which the following statement holds: if the width of $h$ is at least $\tau'$, then there exist trainable parameters of $f_\theta$ and $h$ satisfying
$$
\sup_{X}\norm{f(X) - (h\circ f_\theta)(X)}_2 < \epsilon_0.
$$ 
\end{proof}

\subsection{Proof for Theorem~\ref{thm:umbaised}}
$f_{\theta}(X) \in \RR^{k \times d}$ is defined by  
$$
f_{\theta}(X_{})_{i}=\frac{1}{\bbf_{\theta} (X)_{i}}\hf_{\theta}(X_{}) _i \in \RR^d
$$
for all $i=1,\dots,k$, where $\hf_{\theta}(X)_i\in \RR^d$ is $i$-th row of $\hf_{\theta}(X_{})_{}$ defined in~\eqref{eq:sse} and $\bbf_{\theta} (X)_{i} \in \RR$ is $i$-th component of $\bbf_\theta(X)$ which is defined in~\eqref{eq:sse-constant}.

\begin{proof} 
From the mini-batch consistency and definition of $\hf_\theta$ and $\bbf_\theta$, we have that for any partition procedure $\zeta$, 
\begin{align*}
\hf_\theta(X) =  \sum_{S \in \zeta(X)} \hf_\theta(S) \ \text{ and } \ \bbf_{\theta}(X) = \sum_{S \in \zeta(X)}  \bbf_{\theta}(S). 
\end{align*}
By using this and defining  $\ell_i(q)\coloneqq\ell(q, y_i)\in \RR$ and $\rvz^{(i)}_{j}\coloneqq f_{\theta}(X_{i})_{j} \in \RR^d$, where $f_\theta(X_i)_j$ is $j$-th row of $f_\theta(X_i)$, the chain rule along with the linearity of the derivative operator yields  that
for any partition procedure $\zeta$,
\begin{align}
\label{eq:2}
\begin{split}
 \frac{\partial L(\theta,\lambda)}{\partial \theta}&= \frac{1}{n}\sum_{i=1}^n \sum_{j=1}^k\frac{\partial (\ell_i\circ g_\lambda)(\rvz^{(i)}_j)}{\partial \rvz^{(i)}_j} \frac{\partial f_{\theta}(X_i)_{j} }{\partial \theta}
\\   & = \frac{1}{n}\sum_{i=1}^n \sum_{j=1}^k\frac{\partial (\ell_i\circ g_\lambda)(\rvz^{(i)}_j)}{\partial \rvz^{(i)}_j}\left(\frac{1}{\bbf_{\theta} (X_{i})_{j}}\frac{\partial \hf_{\theta}(X_{i}) _j }{\partial \theta} -  \hf_{\theta}(X)_j \frac{1}{\bbf_{\theta} (X_{i})_{j}^2} \frac{\partial \bbf_{\theta} (X_{i})_{j}^{} }{\partial \theta}  \right)
\\   & = \frac{1}{n}\sum_{i=1}^n \sum_{j=1}^k\frac{\partial (\ell_i\circ g_\lambda)(\rvz^{(i)}_j)}{\partial \rvz^{(i)}_j}\left(\frac{1}{\bbf_{\theta} (X_i)_{j}} \sum_{S \in \zeta(X_i)} \frac{\partial \hf_{\theta}(S) _j }{\partial \theta} - \frac{1}{\bbf_{\theta} (X_{i})_{j}^2} \hf_{\theta}(X_{i})_j  \sum_{S \in \zeta(X_{i})}\frac{\partial \bbf_{\theta} (S)_{j}^{} }{\partial \theta}  \right).
\end{split}
\end{align}
Similarly, by defining $\bell_i(q)\coloneqq\ell(q, \by_i) \in \RR$ and $\bz^{(i)}_j\coloneqq\hf_{\theta}(\bX_{i})_{j} \in \RR^{d}$,
  
\begin{align}
 \frac{\partial L_{t,1}(\theta, \lambda)}{\partial \theta}&=\frac{1}{m}\sum_{i=1}^m \frac{|\zeta_t(\bX_i)|}{|\bzeta_t(\bX_i)|} \sum_{j=1}^k\frac{\partial (\bell_i\circ g_\lambda)(\bz^{(i)}_j)}{\partial \bz^{(i)}_j} \frac{\partial f_{\theta}^{\bzeta_t}(\bX_{i})_{j} }{\partial \theta}
\\ \nonumber & =\frac{1}{m}\sum_{i=1}^m \frac{|\zeta_t(\bX_i)|}{|\bzeta_t(\bX_i)|} \sum_{j=1}^k\frac{\partial (\bell_i\circ g_\lambda)(\bz^{(i)}_j)}{\partial \bz^{(i)}_j} \left( \frac{1}{\bbf_{\theta}(\bX_{i})_{j}} \sum_{\bS \in\bzeta_t(\bX_{i})} \frac{\partial \hf_{\theta}(\bS) _j }{\partial \theta} - \frac{1}{\bbf_{\theta}(\bX_{i})_{j}^2} \hf_{\theta}(\bX_{i})_j  \sum_{\bS \in\bzeta_{t}(\bX_{i})}\frac{\partial \bbf_{\theta}(\bS)_{j}^{} }{\partial \theta}  \right).
\end{align}
Let $t\in\NN_+$ be fixed. By the linearity of expectation, we have that
\begin{align} \label{eq:1}
\EE_{((\bX_{i}, \by_{i}))^m_{i=1}} \EE_{(\bzeta_t(\bX_i))^m_{i=1}}\left[\frac{\partial L_{t,1}(\theta,\lambda)}{\partial \theta} \right]
& =\EE_{((\bX_{i}, \by_{i}))^m_{i=1}} \left[\frac{1}{m}\sum_{i=1}^m \sum_{j=1}^k \frac{\partial (\bell_i\circ g_\lambda)(\bz^{(i)}_j)}{\partial\bz^{(i)}_j}   \mathrm{Q_{ij}}\right] 
\end{align}
where 
$$
\mathrm{Q}_{ij}=\frac{1}{\bbf_{\theta}(\bX_{i})_{j}} \EE_{\bzeta_t(\bX_i)}\left[\frac{|\zeta_t(\bX_i)|}{|\bzeta_t(\bX_i)|}    \sum_{\bS \in\bzeta_t(\bX_{i})} \frac{\partial \hf_{\theta}(\bS) _j }{\partial \theta} \right]-\frac{1}{\bbf_{\theta}(\bX_{i})_{j}^2} \hf_{\theta}(\bX_{i})_j\EE_{\bzeta_t(\bX_i)}\left[\frac{|\zeta_t(\bX_i)|}{|\bzeta_t(\bX_i)|}   \sum_{\bS\in\bzeta_{t}(\bX_{i})}\frac{\partial \bbf_{\theta}(\bS)_{j}^{} }{\partial \theta}   \right].
$$
Below, we  further analyze the following  factors in the right-hand side of this equation:
\begin{align*}
\EE_{\bzeta_t(\bX_i)}\left[\frac{|\zeta_t(\bX_i)|}{|\bzeta_t(\bX_i)|}  \sum_{\bS \in\bzeta_t(\bX_i) } \frac{\partial \hf_{\theta}(\bS)_{j} }{\partial \theta} \right] \text{ and } \EE_{\bzeta_t(\bX_i)}\left[\frac{|\zeta_t(\bX_i)|}{|\bzeta_t(\bX_i)|}   \sum_{\bS\in\bzeta_{t}(\bX_{i})}\frac{\partial \bbf_{\theta}(\bS)_{j}^{} }{\partial \theta}   \right]. 
\end{align*}
Denote  the elements of $\bzeta_t(\bX_i)$ as $\{\bS_1, \bS_2,\dots, \bS_{|\bzeta_t(\bX_i)|}\}=\bzeta_t(\bX_i)$. Then, 
\begin{align*}
\EE_{\bzeta_t(\bX_i)}\left[\frac{|\zeta_t(\bX_i)|}{|\bzeta_t(\bX_i)|}  \sum_{\bS \in\bzeta_t(\bX_i) } \frac{\partial \hf_{\theta}(\bS)_{j} }{\partial \theta} \right]&=\EE_{\bS_1, \bS_2,\dots, \bS_{|\bzeta_t(\bX_i)|}}\left[\frac{|\zeta_t(\bX_i)|}{|\bzeta_t(\bX_i)|}  \sum_{l=1 }^{|\bzeta_t(\bX_i)|} \frac{\partial \hf_{\theta}(\bS_{l})_{j} }{\partial \theta} \right]
\\ & =\frac{|\zeta_t(\bX_i)|}{|\bzeta_t(\bX_i)|}\sum_{l=1 }^{|\bzeta_t(\bX_i)|}\EE_{\bS_l}\left[   \frac{\partial \hf_{\theta}(\bS_{l})_{j} }{\partial \theta} \right].
\end{align*}
Since $\bS_l$ is  drawn independently and uniformly from the elements of $\zeta_t(\bX_i)$, we have that 
$$
\EE_{\bS_l}\left[   \frac{\partial \hf_{\theta}(\bS_{l})_{j} }{\partial \theta} \right]=\frac{1}{|\zeta_t(\bX_i)|} \sum_{S \in \zeta_t(\bX_i) } \frac{\partial \hf_{\theta}(S)_{j} }{\partial \theta}.
$$
Substituting this into the right-hand side of the preceding equation,  we have that 
\begin{align*}
\EE_{\bzeta_t(\bX_i)}\left[\frac{|\zeta_t(\bX_i)|}{|\bzeta_t(\bX_i)|}  \sum_{\bS \in\bzeta_t(\bX_i) } \frac{\partial \hf_{\theta}(\bS)_{j} }{\partial \theta} \right]&=\frac{|\zeta_t(\bX_i)|}{|\bzeta_t(\bX_i)|}\sum_{l=1}^{|\bzeta_t(\bX_i)|}\frac{1}{|\zeta_t(\bX_i)|} \sum_{S \in \zeta_t(\bX_i) } \frac{\partial \hf_{\theta}(S)_{j} }{\partial \theta}=\sum_{S \in \zeta_t(\bX_i) } \frac{\partial \hf_{\theta}(S)_{j} }{\partial \theta}.
\end{align*} 
Similarly,
$$
\EE_{\bzeta_t(\bX_i)}\left[\frac{|\zeta_t(\bX_i)|}{|\bzeta_t(\bX_i)|}   \sum_{\bS\in\bzeta_{t}(\bX_{i})}\frac{\partial \bbf_{\theta}(\bS)_{j}^{} }{\partial \theta}   \right] =\sum_{S \in \zeta_t(\bX_i) } \frac{\partial \bbf_{\theta}(S)_{j} }{\partial \theta}.
$$
Substituting these into $\mathrm{Q}_{ij}$, 
$$
\mathrm{Q}_{ij} =\frac{1}{\bbf_{\theta}(\bX_{i})_{j}} \sum_{S \in \zeta_t(\bX_i) } \frac{\partial \hf_{\theta}(S)_{j} }{\partial \theta}-\frac{1}{\bbf_{\theta}(\bX_{i})_{j}^2} \hf_{\theta}(X)_j\sum_{S \in \zeta_t(\bX_i) } \frac{\partial \bbf_{\theta}(S)_{j} }{\partial \theta}. 
$$
By using this in~\eqref{eq:1} and defining $B(\bX_{i}, \by_i)_{j}=\frac{\partial (\bell_i\circ g)(\bz^{(i)}_j)}{\partial \bz^{(i)}_j} \mathrm{Q}_{ij},
$
we have that  
\begin{align*}
\EE_{((\bX_{i}, \by_{i}))^m_{i=1}} \EE_{(\bzeta_t(\bX_i))^m_{i=1}}\left[\frac{\partial L_{t,1}(\theta,\lambda)}{\partial \theta} \right]
 =\EE_{((\bX_{i}, \by_{i}))^m_{i=1}} \left[\frac{1}{m}\sum_{i=1}^m \sum_{j=1}^kB(\bX_{i}, \by_i)_j\right]&= \frac{1}{m}\sum_{i=1}^m \sum_{j=1}^{k}\EE_{(\bX_{i}, \by_{i})}[B(\bX_{i}, \by_i)_{j}]
\\ &=\sum_{j=1}^{k}\frac{1}{m}\sum_{i=1}^m  \frac{1}{n} \sum_{l=1}^n B(X_{l}, y_l)_{j}
\\ & =\frac{1}{n} \sum_{i=1}^n \sum_{j=1}^{k}B(X_{i}, y_i)_{j}. 
\end{align*}
Thus, expanding the definition of $B(X_{i}, y_i)_{j}$, 
\begin{align*}
&\EE_{((\bX_{i}, \by_{i}))^m_{i=1}} \EE_{(\bzeta_t(\bX_i))^m_{i=1}}\left[\frac{\partial L_{t,1}(\theta,\lambda)}{\partial \theta} \right]
\\ & =\frac{1}{n} \sum_{i=1}^n \sum_{j=1}^{k} \frac{\partial (\ell_i\circ g_\lambda)(\rvz^{(i)}_j)}{\partial \rvz^{(i)}_j} \left(\frac{1}{\bbf_{\theta}(X_i)_{j}} \sum_{S \in \zeta_t(X_{i}) } \frac{\partial \hf_{\theta}(S)_{j} }{\partial \theta}-\frac{1}{\bbf_{\theta}(X_{i})_{j}^2} \hf_{\theta}(X)_j\sum_{S \in \zeta_t(X_{i}) } \frac{\partial \bbf_{\theta}(S)_{j} }{\partial \theta} \right).
\end{align*}
Here, since $\sum_{S \in \zeta_t(X_{i}) } \frac{\partial \hf_{\theta}(S) }{\partial \theta}=\sum_{S \in \zeta(X_{i})} \frac{\partial \hf_{\theta}(S) }{\partial \theta}$ and $\sum_{S \in \zeta_t(X_{i}) } \frac{\partial \bbf_{\theta}(S)_{j} }{\partial \theta}=\sum_{S \in \zeta(X_{i}) } \frac{\partial \bbf_{\theta}(S)_{j} }{\partial \theta}$ for any partition procedure $\zeta$ from   the mini-batch consistency,
we have that 
\begin{align} \label{eq:3}
\begin{split}
& \EE_{((\bX_{i}, \by_{i}))^m_{i=1}} \EE_{(\bzeta_t(\bX_i))^m_{i=1}}\left[\frac{\partial L_{t,1}(\theta, \lambda)}{\partial \theta} \right]
\\  & =\frac{1}{n} \sum_{i=1}^n \sum_{j=1}^{k} \frac{\partial (\ell_i\circ g_\lambda)(\rvz^{(i)}_j)}{\partial \rvz^{(i)}_j} \left(\frac{1}{\bbf_{\theta}(X_i)_{j}} \sum_{S \in \zeta(X_{i}) } \frac{\partial \hf_{\theta}(S)_{j} }{\partial \theta}-\frac{1}{\bbf_{\theta}(X_{i})_{j}^2} \hf_{\theta}(X)_j\sum_{S \in \zeta(X_{i}) } \frac{\partial \bbf_{\theta}(S)_{j} }{\partial \theta} \right).  
\end{split}
\end{align}
By comparing equation \eqref{eq:2} and equation \eqref{eq:3}, we conclude that 
$$
\EE_{((\bX_{i}, \by_{i}))^m_{i=1}} \EE_{(\bzeta_t(\bX_i))^m_{i=1}}\left[\frac{\partial L_{t,1}(\theta, \lambda)}{\partial \theta} \right]=\frac{\partial L(\theta, \lambda)}{\partial \theta}.
$$
Since $t$ was arbitrary, this holds for any $t\in\mathbb{N}_+$.

Now we want to show~\eqref{eq:lambda} holds for any $t\in\mathbb{N}_+$. Since $\bbf^{\bzeta_t}_\theta(\bX_i)  = \bbf_\theta(\bX_i)$  and $\hf^{\bzeta_t}_\theta(\bX_i) = \hf_\theta(\bX_i)$ for all $i\in [m]$,
\begin{align}
    \EE_{((\bX_{i}, \by_{i}))^m_{i=1}} \EE_{(\bzeta_t(\bX_i))^m_{i=1}}\left[\frac{\partial L_{t,2}(\theta, \lambda)}{\partial \lambda} \right] &= \EE_{((\bX_{i}, \by_{i}))^m_{i=1}} \EE_{(\bzeta_t(\bX_i))^m_{i=1}}\left[\frac{1}{m}\sum_{i=1}^m\frac{1}{\lvert \bzeta_t(\bX_i)\rvert}\frac{\partial(\bell_i\circ g_\lambda) (f^{\bzeta_t}_\theta(\bX_i))}{ \partial \lambda} \right] \nonumber \\
    &= \EE_{((\bX_{i}, \by_{i}))^m_{i=1}} \EE_{(\bzeta_t(\bX_i))^m_{i=1}}\left[\frac{1}{m}\sum_{i=1}^m\frac{1}{\lvert \bzeta_t(\bX_i)\rvert}\frac{\partial(\bell_i\circ g_\lambda) (f_\theta(\bX_i))}{ \partial \lambda} \right].\label{eq:lambda-1}
\end{align}
Since we independently and uniformly sample $\bS_1, \bS_2, \ldots, \bS_{\lvert \bzeta_t(\bX_i)\rvert}$ from $\zeta_t(\bX_i)$ and $\frac{\partial (\bell_i \circ g_\lambda)(f_\theta(\bX_i)}{\partial \lambda}$ is constant with respect to the sampling,

\begin{align}
\EE_{\bzeta_t(\bX_i)}\left[\frac{1}{\lvert\bzeta_t(\bX_i)\rvert} \frac{\partial(\bell_i\circ g_\lambda) (f_\theta(\bX_i))}{ \partial \lambda} \right] &= \EE_{\bS_1, \bS_2, \ldots, \bS_{\lvert \bzeta_t(\bX_i)\rvert}} \left[  \frac{1}{\lvert \bzeta_t(\bX_i) \rvert}\frac{\partial(\bell_i\circ g_\lambda) (f_\theta(\bX_i))}{ \partial \lambda} \right]  \nonumber \\   
&= \frac{1}{\lvert \bzeta_t(\bX_i)\rvert}\sum_{j=1}^{\lvert \bzeta_t(\bX_i) \rvert} \EE_{\bS_j}\left[\frac{\partial(\bell_i\circ g_\lambda) (f_\theta(\bX_i))}{ \partial \lambda}\right] \nonumber\\
&= \frac{1}{\lvert \bzeta_t(\bX_i)\rvert}\sum_{j=1}^{\lvert \bzeta_t(\bX_i) \rvert} \frac{\partial(\bell_i\circ g_\lambda) (f_\theta(\bX_i))}{ \partial \lambda} \nonumber\\
& = \frac{\partial(\bell_i\circ g_\lambda) (f_\theta(\bX_i))}{ \partial \lambda}.
\label{eq:lambda-2}
\end{align}

Since we sample a mini-batch $((\bX_i, \by_i))_{i=1}^m$ independently and uniformly from the whole training set $((X_i, y_i))_{i=1}^n$, if we apply \eqref{eq:lambda-2} to \eqref{eq:lambda-1}, we get
\begin{align}
\begin{split}
    \EE_{((\bX_{i}, \by_{i}))^m_{i=1}} \left[\frac{1}{m}\sum_{i=1}^m\frac{\partial(\bell_i\circ g_\lambda) (f_\theta(\bX_i))}{ \partial \lambda} \right] &= \frac{1}{m}\sum_{i=1}^m\EE_{((\bX_{i}, \by_{i}))^m_{i=1}} \left[\frac{\partial(\bell_i\circ g_\lambda) (f_\theta(\bX_i))}{ \partial \lambda} \right] \\
     &= \frac{1}{m}\sum_{i=1}^m\left(\frac{1}{n} \sum_{j=1}^n\frac{\partial(\ell_j\circ g_\lambda) (f_\theta(X_j))}{ \partial \lambda}\right)  \\
     &= \frac{1}{n} \sum_{j=1}^n\frac{\partial(\ell_j\circ g_\lambda) (f_\theta(X_j))}{ \partial \lambda} \\
     &= \frac{\partial L(\theta, \lambda)}{\partial \lambda}
\end{split}
\end{align}
Therefore, we conclude that
$$
\EE_{((\bX_{i}, \by_{i}))^m_{i=1}} \EE_{(\bzeta_t(\bX_i))^m_{i=1}}\left[\frac{\partial L_{t,2}(\theta,\lambda)}{\partial \lambda} \right] = \frac{\partial L(\theta, \lambda)}{\partial \lambda}.
$$
Since $t$ was arbitrary, this holds for any $t\in\mathbb{N}_+$.
\end{proof}
\subsection{Unbiased Estimation of Full Set Gradient for MIL}
\begin{corollary}
For multiple instance learning, given a training set $((X_i, y_i))_{i=1}^n$, we define the loss as follows:
\begin{gather*}
     \rvz_{i,1} = \max\{w^\top \phi(\rvx_{i,j}) + b \}_{j=1}^{N_i}, \quad \rvz_{i,2} = f_\theta(X_i), \\
     \Lcal_i = \frac{1}{2}\left(\ell(\rvz_{i,1}, y_i) + \ell(g_\lambda(\rvz_{i,2}), y_i)\right), \\
     L(\theta, \lambda, w, b) = \frac{1}{n}\sum_{i=1}^n\Lcal_i.
\end{gather*}

For every iteration $t\in\NN_+$ we sample a mini-batch of training data $((\bX_i, \by_i))_{i=1}^m\sim D[((X_i,y_i))_{i=1}^n]$ and sample random partition $(\zeta_t(\bX_i))_{i=1}^m$. Then we sample a mini-batch of subsets $\bzeta_t(\bX_i)\sim D([\zeta_t(\bX_i)]$. Let $\psi \coloneqq (w,b) \in \RR^{d_h+1}$ and define a function $h_\psi:\Xcal\to\RR$ by $h_\psi(X)\coloneqq \max\{w^\top \phi(\rvx) +b: \rvx\in X\}$. Similar to \eqref{eq:stop-grad} we define $h^{\bzeta_t}_\psi$ as,
\begin{align*}
    h_\psi^{\bzeta_t}(\bX_i) &\coloneqq h^{\bzeta_t, \zeta_t}_\psi(\bX_i)\\
    &\coloneqq \max\{h_\psi(S), \sg(h_\psi(S^\prime))\mid S\in\bzeta_t(\bX_i), S^\prime \in \zeta_t(\bX_i)\setminus \bzeta_t(\bX_i) \}.
\end{align*}

Then we update the parameters $\theta, \lambda$ and $\psi$ using the gradient of the following functions as
\begin{gather*}
    L_{t,1}(\theta, \lambda, \psi) =  \frac{1}{2m}\sum_{i=1}^m \frac{\lvert \zeta_t(\bX_i)\rvert}{\lvert \bzeta_t(\bX_i)\rvert}\left(\ell(h^{\bzeta_t}(\bX_i),\by_i) + \ell(g_\lambda(f^{\bzeta_t}_\theta(\bX_i),\by_i) ) \right) \\
    L_{t,2}(\theta, \lambda) = \frac{1}{2m}\sum_{i=1}^m  \ell(g_\lambda(f^{\bzeta_t}_\theta(\bX_i),\by_i) ) 
\end{gather*}
\begin{align*}
    \theta_{t+1} &= \theta_t - \eta_t \frac{\partial L_{t,1}(\theta_t, \lambda_t, \psi_t)}{\partial \theta_t}\\
    \psi_{t+1} &= \psi_t - \eta_t \frac{\partial L_{t,1}(\theta_t, \lambda_t, \psi_t)}{\partial \psi_t} \\
    \lambda_{t+1} &= \lambda_t - \eta_t \frac{\partial L_{t,2}(\theta_t, \lambda_t)}{\partial \lambda_t},
\end{align*}
where $\eta_t >0$ is a learning rate.
If we assume that there exists a unique maximum value $\max\{w^\top \phi(\rvx) +b: \rvx\in X_i \}\in\RR$  for each $i\in\{1,\ldots, n\}$ and sample a single subset from $\zeta_t(\bX_i)$ for $i\in\{1,\ldots,m\}$, \ie $\lvert \bzeta_t(\bX_i)\rvert=1$,  then the following holds:
\begin{align}
    \EE_{((\bX_i, \by_i))_{i=1}^m}\EE_{(\bzeta_t(\bX_i))_{i=1}^m}\left[\frac{\partial L_{t,1}(\theta,\lambda,\psi)}{\partial \theta}\right] &= \frac{\partial L(\theta,\lambda,\psi)}{\partial \theta} \label{eq:theta} \\
    \EE_{((\bX_i, \by_i))_{i=1}^m}\EE_{(\bzeta_t(\bX_i))_{i=1}^m}\left[\frac{\partial L_{t,1}(\theta,\lambda,\psi)}{\partial \psi}\right] &= \frac{\partial L(\theta,\lambda, \psi)}{\partial \psi} \label{eq:psi} \\
    \EE_{((\bX_i, \by_i))_{i=1}^m}\EE_{(\bzeta_t(\bX_i))_{i=1}^m}\left[\frac{\partial L_{t,2}(\theta,\lambda)}{\partial \lambda}\right] &= \frac{\partial L(\theta,\lambda,\psi)}{\partial \lambda} \label{eq:lam}
\end{align}
\end{corollary}

\begin{proof}
It is enough to show the \eqref{eq:theta} and~\ref{eq:psi} hold since  we have already proved \eqref{eq:lam}, which does not depend on $\psi$, in~\cref{thm:umbaised}. By defining $\ell_i(q) \coloneqq \ell(q,y_i)\in\RR$, $u_{i,j}=w^\top\phi(\rvx_{i,j})+b$ where $\rvx_{i,j}\in X_i$ for $j\in\{1,\ldots, N_i\}$, and $u_{i, M}=\max\{u_{i,j} \}_{j=1}^{N_i}$,
\begin{align*}
    \frac{\partial L(\theta, \lambda, \psi)}{\partial \psi} &=\frac{1}{n}\sum_{i=1}^n \frac{1}{2} \frac{\partial \ell_i (h_\psi(X_i) )}{\partial h_\psi(X_i)} \frac{\partial h_\psi(X_i)}{\partial \psi} \\
    &= \frac{1}{2n}\sum_{i=1}^n \frac{\partial \ell_i(h_\psi(X_i)) }{\partial h_\psi(X_i)} \frac{\partial \max\{h_\psi(S) \mid S\in\zeta(X_i)\}}{\partial \psi} \\
    &= \frac{1}{2n}\sum_{i=1}^n\frac{\partial \ell_i(h_\psi(X_i)) }{\partial h_\psi(X_i)} \frac{\partial u_{i,M}}{\partial \psi}
\end{align*}
for any partition $\zeta(X_i)$ for all $i\in \{1,\ldots,m\}$. Similarly, we  define $\bell_i(q) \coloneqq \ell(q, \by_i)\in\RR$, $\bar u_{i,j}\coloneqq w^\top \phi(\bar \rvx_{i,j}) +b$ where $\bar \rvx_{i,j}\in \bX_i$ for $j\in\{1,\ldots, N_i\}$, and $\bar u_{i,M} = \max\{\bar u_{i,j}\}_{j=1}^{N_i}$. Let $t\in\NN_+$ be fixed and define,
\begin{align*}
    \bar{B}_{t,i} &\coloneqq \{ h_\psi(S): S\in \bzeta_t(\bX_i)\}.   
\end{align*}
With linearity of expectation and properties of the max operation, we have that
\begin{align}   &\EE_{((\bX_i,\by_i))_{i=1}^m}\EE_{(\bzeta_t(\bX_i))_{i=1}^m}\left[\frac{\partial L_{t,1}(\theta,\lambda,\psi)}{\partial \psi}\right] \nonumber\\
   &= \frac{1}{2m}\sum_{i=1}^m \EE_{((\bX_i,\by_i))_{i=1}^m}\EE_{(\bzeta_t(\bX_i))_{i=1}^m}\left[\frac{\lvert \zeta_t(\bX_i) \rvert}{\lvert \bzeta_t(\bX_i)\rvert}\frac{\partial \bell_i(h_\psi(\bX_i))}{\partial h_\psi(\bX_i)}\one\{\bar{u}_{i,M}\in \bar{B}_{t,i}\}\frac{\partial \bar u_{i,M}}{\partial \psi} \right]
   \label{eq:36}
\end{align}
Note that  we partition the set $\bX_i$ and there is a unique maximum value. Thus, only one subset $S\in\zeta_t(\bX_i)$ includes the element leading to the maximum value $\bar u_{i,M}$.  If we do not choose such a subset $S$, the gradient in \eqref{eq:36} becomes zero.
Since we sample uniformly a single subset $\bS$ from $\zeta_t(\bX_i)$, \ie $\bzeta_t(\bX_i)=\{\bS \}$, we get
\begin{align}
    \EE_{\bzeta_t(\bX_i)}\left[\frac{\lvert \zeta_t(\bX_i) \rvert}{\lvert \bzeta_t(\bX_i)\rvert} \frac{\partial \bell_i(h_\psi(\bX_i))}{\partial h_\psi(\bX_i)}\one\{\bar u_{i,M}\in \bar{B}_{t,i} \} \frac{\partial \bar u_{i,M}}{\partial \psi}\right] &=\EE_{\bS}\left[\frac{\lvert \zeta_t(\bX_i) \rvert}{\lvert \bzeta_t(\bX_i)\rvert} \frac{\partial \bell_i(h_\psi(\bX_i))}{\partial h_\psi(\bX_i)}\one\{\bar u_{i,M}=h_\psi(\bS) \} \frac{\partial \bar u_{i,M}}{\partial \psi}\right] \nonumber \\
    &=\frac{\lvert \zeta_t(\bX_i) \rvert}{1} \frac{1}{\lvert \zeta_t(\bX_i)\rvert}\frac{\partial \bell_i(h_\psi(\bX_i))}{\partial h_\psi(\bX_i)} \frac{\partial \bar u_{i,M}}{\partial \psi} \nonumber \\
    &= \frac{\partial \bell_i(h_\psi(\bX_i))}{\partial h_\psi(\bX_i)} \frac{\partial \bar u_{i,M}}{\partial \psi} \label{eq:37}
\end{align}
If we apply the right hand side of \eqref{eq:37} to~\eqref{eq:36}, we obtain
\begin{align*}
     \EE_{((\bX_i,\by_i))_{i=1}^m}\EE_{(\bzeta_t(\bX_i))_{i=1}^m}\left[\frac{\partial L_{t,1}(\theta,\lambda,\psi)}{\partial \psi}\right] 
     &= \frac{1}{2m}\sum_{i=1}^m \EE_{((\bX_i,\by_i))_{i=1}^m}\left[\frac{\partial \bell_i(h_\psi(\bX_i))}{\partial h_\psi(\bX_i)} \frac{\partial \bar u_{i,M}}{\partial \psi}\right] \\
     &= \frac{1}{2m}\sum_{i=1}^m\frac{1}{n}\sum_{j=1}^n\frac{\partial \ell_j(h_\psi(X_j))}{\partial h_\psi(X_j)} \frac{\partial  u_{j,M}}{\partial \psi} \\
     &= \frac{1}{2n}\sum_{i=1}^n \frac{\partial \ell_i(h_\psi(X_i))}{\partial h_\psi(X_i)} \frac{\partial  u_{i,M}}{\partial \psi} \\
     &= \frac{\partial L(\theta,\lambda, \psi)}{\partial \psi}.
\end{align*}
With the chain rule, we get,
\begin{align*}
    \frac{\partial L(\theta, \lambda, \psi)}{\partial \theta} = \frac{1}{2n}\sum_{i=1}^n\left(\frac{\partial \ell_i(h_\psi(X_i))}{\partial \theta} + \frac{\partial(\ell_i \circ g_\lambda)(f_\theta(X_i)) }{\partial \theta} \right).
\end{align*}
Since we have already shown that 
\begin{align}
     \frac{1}{2n}\sum_{i=1}^n\frac{\partial( \ell_i \circ g_\lambda)(f_\theta(X_i)) }{\partial \theta} = \EE_{((\bX_i, \by_i))_{i=1}^m}\EE_{(\bzeta_t(\bX_i))_{i=1}^m}\left[\frac{1}{2m} \sum_{i=1}^m \frac{\lvert \zeta_t(\bX_i) \rvert}{\lvert \bzeta_t(\bX_i) \rvert}\frac{\partial (\bell_i \circ f^{\bzeta_t}_\theta)(\bX_i) }{\partial \theta}\right]
\label{eq:first-theta}
\end{align}
in~\cref{thm:umbaised},
it suffices to show that 
\begin{align}
    \frac{1}{2n} \sum_{i=1}^n \frac{\partial \ell_i(h_\psi(X_i))}{\partial \theta} = \EE_{((\bX_i, \by_i))_{i=1}^m}\EE_{(\bzeta_t(\bX_i))_{i=1}^m}\left[\frac{1}{2m}\sum_{i=1}^m\frac{\lvert \zeta_t(\bX_i) \rvert}{\lvert \bzeta_t(\bX_i) \rvert}\frac{\partial\bell_i(h^{\bzeta_t}_\psi(\bX_i))}{\partial \theta}\right].
\label{eq:38}
\end{align}
For the left hand side of~\eqref{eq:38}, we have that
\begin{align}
\begin{split}
    \frac{1}{2n} \sum_{i=1}^n \frac{\partial \ell_i(h_\psi(X_i))}{\partial \theta} &= \frac{1}{2n}\sum_{i=1}^n \frac{\partial \ell_i(h_\psi(X_i))}{\partial h_\psi(X_i)} \frac{\partial \max\{h_\psi(S)\mid S\in\zeta(X_i)\}}{\partial \theta} \\
    &= \frac{1}{2n}\sum_{i=1}^n \frac{\partial \ell_i(h_\psi(X_i))}{\partial h_\psi(X_i)} \frac{\partial  u_{i,M}}{\partial \theta} \\
    &= \frac{1}{2n}\sum_{i=1}^n \frac{\partial \ell_i(h_\psi(X_i))}{\partial h_\psi(X_i)} \frac{\partial(  w^\top \phi(\rvx_{i,M})+b)}{\partial \theta} \\
    &= \frac{1}{2n}\sum_{i=1}^n \frac{\partial \ell_i(h_\psi(X_i))}{\partial h_\psi(X_i)} w^\top\frac{\partial\phi(\rvx_{i,M})}{\partial \theta}.
\end{split}
\label{eq:39}
\end{align}
For the right hand side of~\eqref{eq:38}, with linearity of expectation, we obtain
\begin{align}
    \EE_{((\bX_i, \by_i))_{i=1}^m}\EE_{(\bzeta_t(\bX_i))_{i=1}^m}\left[\frac{1}{2m}\sum_{i=1}^m\frac{\lvert \zeta_t(\bX_i) \rvert}{\lvert \bzeta_t(\bX_i) \rvert}\frac{\partial\bell_i(h^{\bzeta_t}_\psi(\bX_i))}{\partial \theta}\right] &= \frac{1}{2m}\sum_{i=1}^m \EE_{((\bX_i, \by_i))_{i=1}^m}\EE_{(\bzeta_t(\bX_i))_{i=1}^m}\left[\frac{\lvert \zeta_t(\bX_i) \rvert}{\lvert \bzeta_t(\bX_i) \rvert} \frac{\partial\bell_i(h^{\bzeta_t}_\psi(\bX_i))}{\partial \theta} \right].
    \label{eq:40}
\end{align}
Now we expand the summand in the right hand side,
\begin{align}
    \EE_{\bzeta_t(\bX_i)}\left[\frac{\lvert \zeta_t(\bX_i) \rvert}{\lvert \bzeta_t(\bX_i) \rvert} \frac{\partial\bell_i(h^{\bzeta_t}_\psi(\bX_i))}{\partial \theta} \right] &= \EE_{\bzeta_t(\bX_i)}\left[\frac{\lvert \zeta_t(\bX_i) \rvert}{\lvert \bzeta_t(\bX_i) \rvert}\frac{\partial\bell_i(h_\psi(\bX_i))}{\partial h_\psi(\bX_i)}\one\{\bar u_{i,M}\in  \bar{B}_{t,i}\} \frac{\partial \bar u_{i,M}}{\partial \theta} \right] \nonumber \\
    &= \EE_{\bS}\left[\frac{\lvert \zeta_t(\bX_i) \rvert}{\lvert \bzeta_t(\bX_i) \rvert}\frac{\partial\bell_i(h_\psi(\bX_i))}{\partial h_\psi(\bX_i)}\one\{\bar u_{i,M}=h_\psi(\bS) \} \frac{\partial \bar u_{i,M}}{\partial \theta} \right] \nonumber \\
    &= \frac{\lvert \zeta_t(\bX_i) \rvert}{1} \frac{1}{\lvert \zeta_t(\bX_i)\rvert}\frac{\partial\bell_i(h_\psi(\bX_i))}{\partial h_\psi(\bX_i)} \frac{\partial \bar u_{i,M}}{\partial \theta} \nonumber\\
    &=\frac{\partial\bell_i(h_\psi(\bX_i))}{\partial h_\psi(\bX_i)}w^\top \frac{\partial \phi(\bar \rvx_{i,M})}{\partial \theta}. \label{eq:41}
\end{align}
Now we apply the right hand side of~\eqref{eq:41} to~\eqref{eq:40}. Then we get,
\begin{align}
\begin{split}
    \EE_{((\bX_i, \by_i))_{i=1}^m}\EE_{(\bzeta_t(\bX_i))_{i=1}^m}\left[\frac{1}{2m}\sum_{i=1}^m\frac{\lvert \zeta_t(\bX_i) \rvert}{\lvert \bzeta_t(\bX_i) \rvert}\frac{\partial\bell_i(h^{\bzeta_t}_\psi(\bX_i))}{\partial \theta}\right] &= \frac{1}{2m} \sum_{i=1}^m\EE_{((\bX_i, \by_i))_{i=1}^m}\left[\frac{\partial \bell_i(h_\psi(\bX_i))}{\partial h_\psi(\bX_i)} w^\top \frac{\partial \phi(\bar \rvx_{i,M})}{\partial \theta} \right] \\
    &= \frac{1}{2m} \sum_{l=1}^m\frac{1}{n}\sum_{i=1}^n\frac{\partial \ell_i(h_\psi(X_i))}{\partial h_\psi(X_i)} w^\top \frac{\partial \phi(\rvx_{i,M})}{\partial \theta} \\
     &= \frac{1}{2n} \sum_{i=1}^n\frac{\partial \ell_i(h_\psi(X_i))}{\partial h_\psi(X_i)} w^\top \frac{\partial \phi(\rvx_{i,M})}{\partial \theta} \\
     &=\frac{1}{2n} \sum_{i=1}^n \frac{\partial \ell_i(h_\psi(X_i))}{\partial \theta}.
\label{eq:second-theta}
\end{split}
\end{align}
Finally combining~\eqref{eq:first-theta} and~\eqref{eq:second-theta}, we arrive at the the conclusion:
\begin{align*}
    &\EE_{((\bX_i, \by_i))_{i=1}^m}\EE_{(\bzeta_t(\bX_i))_{i=1}^m}\left[\frac{\partial L_{t,1}(\theta, \lambda, \psi)}{\partial \theta} \right] \\
   &=\EE_{((\bX_i, \by_i))_{i=1}^m}\EE_{(\bzeta_t(\bX_i))_{i=1}^m}\left[\frac{1}{2m}\sum_{i=1}^m\frac{\lvert \zeta_t(\bX_i)\rvert}{\lvert \bzeta_t(\bX_i)\rvert}\left(\frac{\partial \bell_i(h^{\bzeta_t}_\psi(\bX_i))}{\partial \theta} + \frac{\partial(\bell_i \circ g_\lambda)(f^{\bzeta_t}_\theta(\bX_i)) }{\partial \theta} \right)\right] \\
   &=  \frac{1}{2n} \sum_{i=1}^n \frac{\partial \ell_i(h_\psi(X_i))}{\partial \theta}+\frac{1}{2n}\sum_{i=1}^n\frac{\partial( \ell_i \circ g_\lambda)(f_\theta(X_i)) }{\partial \theta} \\
   &=\frac{\partial L(\theta, \lambda, \psi)}{\partial \theta}.
\end{align*}
\end{proof}

\subsection{SSE's Training Method Is a Biased Approximation to the Full Set Gradient}
\label{app:bias}
In this section, we show that sampling a single subset, and computing the gradient as an approximation to the gradient of $L(\theta, \lambda)$, which is proposed by~\citet{mbc}, is a biased estimation of full set gradient. Since $\hf_\theta$ with an attention activation function comprised of $\nu_1$ and a sigmoid for $\sigma$ is equivalent to a Slot Set Encoder, and is a special case of UMBC, we focus on the gradient of $\hf_\theta$. Specifically, at every iteration $t\in\NN_+$, we sample a mini-batch $((\bX_i, \by_i))_{i=1}^m$ from the training dataset $((X_i,y_i))_{i=1}^n$. We choose a partition $\zeta_t(\bX_i)$ for each $\bX_i$ and sample a single subset $\bS_i$ from the partition $\zeta_t(\bX_i)$. If we compute the gradient of the loss as
\begin{align}
\frac{\partial }{\partial\lambda}\left(\frac{1}{m} \sum_{i=1}^m  (\bell_i \circ g_\lambda) (\hf_\theta(\bS_i))\right),
\label{eq:43}
\end{align}
then it is a \emph{biased estimation} of $\frac{\partial L(\theta, \lambda)}{\partial \lambda}$, where $\bell_i(\cdot)$ is defined by $\bell_i(q)\coloneqq \ell(q, \by_i)$.

\begin{proof}

The gradient of $L(\theta,\lambda)$ with respect to the parameter $\lambda$ is
\begin{align}\label{eq:fullset-grad}
\begin{split}
    \frac{\partial L(\theta, \lambda)}{\partial \lambda} 
    &= \frac{1}{n} \sum_{i=1}^n\frac{\partial(\ell_i\circ g_\lambda)(\hf_\theta(X_i))}{ \partial \lambda} \\
    &= \frac{1}{n} \sum_{i=1}^n\frac{\partial(\ell_i\circ g_\lambda) \left( \sum_{S \in \zeta(X_{i})} \hf_{\theta}(S) \right)}{ \partial \lambda}
\end{split}
\end{align}
for a partition $\zeta_t(X_i)$ of the set $X_i$, where $\ell_i(\cdot)$ is defined by $\ell_i(q)\coloneqq \ell(q,y_i)$. 
However, the expectation of~\eqref{eq:43} is not equal to the full set gradient in~\eqref{eq:fullset-grad}:
\begin{align}
\begin{split}
    \EE_{((\bX_{i}, \by_{i}))^m_{i=1}} \EE_{ \bS_i  } \left[\frac{1}{m}\sum_{i=1}^m\frac{\partial(\bell_i\circ g_\lambda) (\hf_\theta(\bS_i))}{ \partial \lambda} \right] &= \frac{1}{m}\sum_{i=1}^m\EE_{((\bX_{i}, \by_{i}))^m_{i=1}} \EE_{\bS_i }\left[\frac{\partial(\bell_i\circ g_\lambda) (\hf_\theta(\bS_i))}{ \partial \lambda} \right] \\
     &= \frac{1}{n} \sum_{i=1}^n \frac{1}{ \vert \zeta_t(X_{i}) \vert } \sum_{ S \in \zeta_t(X_{i}) } \frac{\partial(\ell_i\circ g_\lambda) (\hf_\theta(S)) }{ \partial \lambda}  \\
     &\textcolor{red}{\neq} \frac{1}{n} \sum_{i=1}^n\frac{\partial(\ell_i\circ g_\lambda) \left(\sum_{S \in \zeta_t(X_{i})} \hf_{\theta}(S) \right)}{ \partial \lambda} \\
     &= \frac{1}{n} \sum_{i=1}^n\frac{\partial(\ell_i\circ g_\lambda)( \hat{f}_\theta(X_i))}{ \partial \lambda}
\label{eq:28}
\end{split}
\end{align}
To see why this is the case, we analyze the case of real valued function $g_\lambda:\RR^d\to \RR$ with $\lambda\in\RR^d$ and a squared loss function 
\begin{gather*}
    \ell(z_i, y_i) \coloneqq \frac{1}{2} (z_i-y_i)^2  \\
    z_i \coloneqq \lambda^\top \hf_\theta(X_i) \coloneqq g_\lambda(\hf_\theta(X_i))  \in \RR.
\end{gather*}
Since $\hf_\theta$ is sum decomposable, \ie $\hf_\theta(X_i) = \sum_{j=1}^{N_i} \hf_\theta(\rvx_{i,j})$ where $X_i=\{\rvx_{i,1}, \ldots, \rvx_{i, N_i}\}$, the full set gradient from~\eqref{eq:fullset-grad} becomes,
\begin{align}
\begin{split}
\frac{1}{n} \sum_{i=1}^n \frac{\partial}{\partial \lambda} \left(\frac{1}{2} (z_i-y_i)^2 \right) &= \frac{1}{n} \sum_{i=1}^n (z_i-y_i)\frac{\partial z_i}{\partial \lambda} \\
&= \frac{1}{n} \sum_{i=1}^n (z_i-y_i)\hf_\theta(X_i) \\
&=\frac{1}{n} \sum_{i=1}^n (z_i-y_i)\left(\sum_{l=1}^{N_i}\hf_\theta(\rvx_{i,l}) \right)\\
&= \frac{1}{n} \sum_{i=1}^n \left(\lambda^\top\left(\sum_{j=1}^{N_i} \hf_\theta(\rvx_{i,j})\right)-y_i\right)\left(\sum_{l=1}^{N_i}\hf_\theta(\rvx_{i,l}) \right) \\
&= \frac{1}{n} \sum_{i=1}^n \left(\left(\sum_{j=1}^{N_i} \lambda^\top\hf_\theta(\rvx_{i,j})\right)-y_i\right)\left(\sum_{l=1}^{N_i}\hf_\theta(\rvx_{i,l}) \right).\\
&= \frac{1}{n} \sum_{i=1}^n \left(\sum_{j=1}^{N_i} \lambda^\top\hf_\theta(\rvx_{i,j})-\frac{y_i}{N_i}\right)\left(\sum_{l=1}^{N_i}\hf_\theta(\rvx_{i,l}) \right).
\end{split}
\end{align}
Assume that $\zeta_t(X_i)  = \{\{\rvx_{i,1}\}, \ldots, \{\rvx_{i,N_i}\}\}$  and we sample a single subset $\bS_i$ from the partition $\zeta_t(\bX_i)$ for all $i\in\{1,\ldots, m\}$ and $t\in\NN_+$.
Then gradient of the subsampling a single subset from~\eqref{eq:28} becomes,
\begin{align}
    \begin{split}
        \frac{1}{n}\sum_{i=1}^n \frac{1}{N_i}\sum_{j=1}^{N_i} \frac{1}{2}\frac{\partial \left(\lambda^\top \hf_\theta(\rvx_{i,j})-y_i\right)^2}{\partial\lambda} &=\frac{1}{n}\sum_{i=1}^n\sum_{j=1}^{N_i}\left(\frac{\lambda^\top\hf_\theta(\rvx_{i,j})-y_i}{N_i}\right)\hf_\theta(\rvx_{i,j}) \\
        &\textcolor{red}{\neq} \frac{1}{n} \sum_{i=1}^n \left(\sum_{j=1}^{N_i} \lambda^\top\hf_\theta(\rvx_{i,j})-\frac{y_i}{N_i}\right)\left(\sum_{l=1}^{N_i}\hf_\theta(\rvx_{i,l}) \right).
    \end{split}
\end{align}

Therefore, the random subsampling of a a single subset in the method proposed by~\citet{mbc} \emph{is not} an unbiased estimate of the gradient of the full set.
\end{proof}


\section{Optimization} \label{app:opt}
Define $\btheta=(\theta,\lambda)$ and $g_t(\btheta)=(\frac{\partial L_{t,1}(\theta,\lambda)}{\partial \theta},\frac{\partial L_{t,2}(\theta,\lambda)}{\partial \lambda})\T$. We assume that $\ell(q,y)\ge 0$ for all $(q,y)$.
Let $(\btheta_t)_{t\in \NN_+}$ be a sequence generated by $\btheta_{t+1}= \btheta_{t} - \eta_t g_t(\btheta)$ with an initial point $\btheta_1$ and a step size sequence $(\eta_t)_{t \in \NN_+}$, where $\btheta_t\in \Rcal\subseteq \RR^D$ for $t \in \NN_+$ with an open convex set $\Rcal$. Here, $\RR^D$ is an open set and thus it is allowed to choose $\Rcal = \RR^D$ (or any other open convex set). We do not assume that the loss function or model is convex. We also do not make any assumption on the initial point $\btheta_1$.
To analyze the optimization behavior formally, we consider the following standard assumption in the literature \citep{lee2016gradient,mertikopoulos2020almost,fehrman2020convergence}:

\begin{assumption} \label{assumption:1}
There exist  $\varsigma>0$ such that for any $\btheta ,\btheta'\in \Rcal$, $t\in [T]$, and $k \in \{1,2\}$, 
\begin{align*}
\|\nabla L_{t,k}(\btheta) - \nabla L_{t,k}(\btheta')\|_{2} \le \varsigma\|\btheta - \btheta'\|_2. \end{align*}
\end{assumption}

We use the following  lemma on a general function from a previous work \citep[Lemma 2]{kawaguchi2022understanding}:
\begin{lemma} \label{lemma:1}
For any differentiable function $\varphi: \dom(\varphi) \rightarrow \RR$ with an open  convex set $\dom(\varphi) \subseteq \RR^{n_\varphi}$, if   $\|\nabla \varphi(z') - \nabla \varphi(z)\| _{2}\le \varsigma_{ \varphi} \|z'-z\|_{2}$ for all $z,z' \in \dom(\varphi)$, then
\begin{align}
\varphi(z') \le \varphi(z) + \nabla \varphi(z)\T (z'-z) + \frac{\varsigma_{ \varphi}}{2} \|z'-z\|^2_2 \quad   \text{for all $z,z' \in \dom(\varphi) $}.
\end{align}
\end{lemma}
In turn, Lemma \ref{lemma:1} implies the following lemma:
\begin{lemma} \label{lemma:2}
For any differentiable function $\varphi: \dom(\varphi) \rightarrow \RR_{\ge 0}$ with an open  convex set $\dom(\varphi) \subseteq \RR^{n_\varphi}$ such that    $\|\nabla \varphi(z') - \nabla \varphi(z)\| \le \varsigma_{ \varphi} \|z'-z\|$ for all $z,z' \in \dom(\varphi)$, the following holds: \text{for all $z\in \dom(\varphi) $} such that  $z- \frac{1}{\varsigma_{\varphi}} \nabla\varphi(z) \in \dom(\varphi) $, 
\begin{align}
\|\nabla\varphi(z)\|_2^2 \le 2\varsigma_{ \varphi}\varphi(z)  \quad   \text{for all $z\in \dom(\varphi) $}.
\end{align}
\end{lemma}
\begin{proof}
Since  $\varphi: \dom(\varphi) \rightarrow \RR_{\ge 0}$ (nonnegative), if $\nabla\varphi(z)=0$, the desired statement holds. Thus, we consider the remaining case of  $\nabla\varphi(z)\neq0$ in the rest of this proof. We invoke Lemma \ref{lemma:1} with $z'=z- \frac{1}{\varsigma_{\varphi}} \nabla\varphi(z)$, yielding
\begin{align*}
0\le \varphi(z') &\le \varphi(z) + \nabla \varphi(z)\T (z'-z) + \frac{\varsigma_{ \varphi}}{2} \|z'-z\|^2_2
\\ & = \varphi(z) -  \frac{1}{\varsigma_{\varphi}}\|\nabla \varphi(z)\| _{2}^{2}+ \frac{1}{2\varsigma_{\varphi}} \|\nabla\varphi(z)\|^2_2
\\ & = \varphi(z) -  \frac{1}{2\varsigma_{\varphi}}\|\nabla \varphi(z)\| _{2}^{2} \end{align*}     
By rearranging, this implies that $\|\nabla\varphi(z)\|_2^2 \le 2\varsigma_{ \varphi}\varphi(z)$. 
\end{proof}

Since we are dealing with a general non-convex and non-invex function (as the choice of architecture and loss is very flexible) where gradient-based optimization might only converge to a stationary point (to avoid the curse of dimensionality), we consider the convergence in terms of stationary points of $L$:  
\begin{theorem} Suppose that Assumption \ref{assumption:1} holds and the step size sequence $(\eta_t)_{t \in \NN_+}$ satisfies $\sum_{t=1}^{\infty}\eta_t^2 < \infty$.  Then,
there exists a constant $c$ independent of $(t,T)$ such that 
$$
\min_{t\in [T]}\EE[\|\nabla L(\btheta_{t})\|_2^2] \le   \frac{  c\EE[L(\btheta_{1})]}{\sum_{t=1}^T\eta_t}.
$$
\end{theorem}
\begin{proof} 
Assumption \ref{assumption:1}  implies that
$$
\|\nabla L_{t,k}(\btheta) - \nabla L_{t,k}(\btheta')\|_{2} ^{2}\le \varsigma^{2}\|\btheta - \btheta'\|_2^{2}
$$
which implies that 
$$
\|g_t(\btheta)-  g_t(\btheta')\|_{2}^2=\|\nabla L_{t,1}(\btheta) - \nabla L_{t,1}(\btheta')\|_{2}^{2}+\|\nabla L_{t,2}(\btheta) - \nabla L_{t,2}(\btheta')\|_{2}^2 \le2 \varsigma^{2}\|\btheta - \btheta'\|_2^{2}.
$$
This implies that 
$$
\|g_t(\btheta)-  g_t(\btheta')\|_{2} \le \sqrt{2} \varsigma\|\btheta - \btheta'\|_2
$$ 
 
 Using this and Theorem \ref{thm:umbaised}  along with Jensen's inequality, we have that for any $\btheta ,\btheta'\in \Rcal$,
\begin{align*}
\|\nabla L(\btheta) - \nabla L(\btheta')\|_{2}  =\| \EE[g_t(\btheta)] -  \EE[g_t(\btheta')]\|_{2} \le \EE[\|g_t(\btheta) -  g_t(\btheta')\|_{2} ]\le \sqrt{2}\varsigma\|\btheta - \btheta'\|_{2}.    
\end{align*}
Thus, $L$ satisfies the conditions of  Lemma \ref{lemma:1} and Lemma \ref{lemma:2}. Since $\btheta_t\in \Rcal\subseteq \RR^D$ for $t \in \NN_+$, using Lemma \ref{lemma:1} for the function $ L$, we have that 
\begin{align*}
 L(\btheta_{t+1})\le L(\btheta_{t})+\nabla L(\btheta_{t})\T (\btheta_{t+1}-\btheta_{t})+ \frac{ \sqrt{2}\varsigma \|\btheta_{t+1}-\btheta_{t}\|_2^{2}}{2}.  
\end{align*}
Using $\btheta_{t+1}-\btheta_{t}=- \eta_t g_t(\btheta_{t})$, 
$$
L(\btheta_{t+1})\le L(\btheta_{t})- \eta_t\nabla L(\btheta_{t})\T g_t(\btheta_{t})+ \frac{\sqrt{2}\varsigma \eta_t^{2} \| g_t(\btheta_{t})\|_2^{2}}{2}.
$$
Using Lemma \ref{lemma:2} for  $ \| g_t(\btheta_{t})\|_2^2=\|\frac{\partial L_{t,1}(\theta_{t},\lambda_{t})}{\partial \theta_{t}}\|_2^2+ \|\frac{\partial L_{t,2}(\theta_{t},\lambda_{t})}{\partial \lambda_{t}}\|_2^2$, we have that 
$$
L(\btheta_{t+1})\le L(\btheta_{t})- \eta_t\nabla L(\btheta_{t})\T g_t(\btheta)+\sqrt{2}  \varsigma ^{2}\eta_t^{2}  (L_{t,1}(\btheta_{t})+L_{t,2}(\btheta_{t})).
$$
Define $L_{t}(\btheta_{t})=L_{t,1}(\btheta_{t})+L_{t,2}(\btheta_{t})$. Using the linearity and monotonicity of expectation, 
\begin{align*}
&\EE_{((\bX_{i}, \by_{i}))^m_{i=1} } \EE_{(\bzeta_t(\bX_i))^m_{i=1}}[L(\btheta_{t+1})|\btheta_{t}]
\\ & \le L(\btheta_{t})- \eta_t\nabla L(\btheta_{t})\T\EE_{((\bX_{i}, \by_{i}))^m_{i=1} } \EE_{(\bzeta_t(\bX_i))^m_{i=1}}[g_t(\btheta)]
+\sqrt{2}  \varsigma ^{2}\eta_t^{2}  \EE_{((\bX_{i}, \by_{i}))^m_{i=1} } \EE_{(\bzeta_t(\bX_i))^m_{i=1}}[L_{t}(\btheta_{t})]
\\ & \le L(\btheta_{t})- \eta_t \|\nabla L(\btheta_{t})\|_2^2+\sqrt{2}  \varsigma ^{2}\eta_t^{2}  (1+a)L(\btheta_{t})
\end{align*}
where the second inequality follows from Theorem \ref{thm:umbaised} and  $\EE_{((\bX_{i}, \by_{i}))^m_{i=1} } \EE_{(\bzeta_t(\bX_i))^m_{i=1}} [L_{t}(\btheta_{t})]=\EE_{((\bX_{i}, \by_{i}))^m_{i=1} } \EE_{(\bzeta_t(\bX_i))^m_{i=1}} [L_{t,1}(\btheta_{t})]+\EE_{((\bX_{i}, \by_{i}))^m_{i=1} } \EE_{(\bzeta_t(\bX_i))^m_{i=1}} [L_{t,2}(\btheta_{t})]\le (1+a)L_{}(\btheta_{t})$ where $a$ is the expectation of the maximum ratio $\frac{|\zeta_t(\bX_i)|}{|\bzeta_t(\bX_i)|}$. 

Taking expectation over $\btheta_{t}$ with  the law of total expectation $\EE[L(\btheta_{t+1})]=\EE_{\btheta_{t}}\EE_{((\bX_{i}, \by_{i}))^m_{i=1} } \EE_{(\bzeta_t(\bX_i))^m_{i=1}}[L(\btheta_{t+1})|\btheta_{t}]$,
\begin{align} \label{eq:4}
\EE[L(\btheta_{t+1})]
\le \EE[L(\btheta_{t})]- \eta_t  \EE[\|\nabla L(\btheta_{t})\|_2^2]+  \sqrt{2}\varsigma ^{2}\eta_t^{2}   (1+a)\EE[L(\btheta_{t})].
\end{align}
Since $\EE[\|\nabla L(\btheta_{t})\|_2^2] \ge 0$, this implies that 
\begin{align*}
\EE[L(\btheta_{t+1})] 
&\le \EE[L(\btheta_{t})]+\sqrt{2}  \varsigma ^{2}\eta_t^{2}   (1+a)\EE[L(\btheta_{t})] \\
&=(1+\sqrt{2}  \varsigma ^{2}\eta_t^{2}(1+a))\EE[L(\btheta_{t})] \\
&\le\exp(\sqrt{2}  \varsigma ^{2}\eta_t^{2}(1+a))\EE[L(\btheta_{t})],
\end{align*}
where the last inequality follows from $1+q\le\exp(q)$ for all $q \in \RR$.
Applying this inequality recursively over $t$, it holds that for any $t\in\NN_+$,
$$
\EE[L(\btheta_{t})]  \le\exp\left(  \sqrt{2}\varsigma ^{2} (1+a) \sum_{j=1}^{t-1}\eta_j^{2}\right)\EE[L(\btheta_{1})].
$$
Using this inequality in~\eqref{eq:4}, 
$$
\EE[L(\btheta_{t+1})]
\le \EE[L(\btheta_{t})]- \eta_t  \EE[\|\nabla L(\btheta_{t})\|_2^2]+\sqrt{2}  \varsigma ^{2}\eta_t^{2}   (1+a)\exp\left(  \sqrt{2}\varsigma ^{2}(1+a) \sum_{j=1}^{t-1}\eta_j^{2}\right)\EE[L(\btheta_{1})].
$$
Rearranging and summing over $t$ with,
\begin{align*}
 \sum_{t=1}^T\eta_t  \EE[\|\nabla L(\btheta_{t})\|_2^2]&\le  \sum_{t=1}^T(\EE[L(\btheta_{t})]-\EE[L(\btheta_{t+1})])+\sqrt{2}\varsigma ^{2}(1+a)\EE[L(\btheta_{1})] \sum_{t=1}^T  \eta_t^{2}   \exp\left(  \sqrt{2}\varsigma ^{2}(1+a) \sum_{j=1}^{t-1}\eta_j^{2}\right)
\\ & \le\EE[L(\btheta_{1})]-\EE[L(\btheta_{T+1})]+\sqrt{2}\varsigma ^{2}(1+a)\EE[L(\btheta_{1})]\exp\left(  \sqrt{2}\varsigma ^{2}(1+a) \sum_{j=1}^{T-1}\eta_j^{2}\right)\left(\sum_{t=1}^T  \eta_t^{2}\right) \\ & =(1+\sqrt{2}\varsigma ^{2}(1+a)R_T\exp(\sqrt{2}\varsigma ^{2} (1+a)R_{T-1}))\EE[L(\btheta_{1})]-\EE[L(\btheta_{T+1})]   
\end{align*}
where we define $R_T=\sum_{t=1}^{T}\eta_t^2$.
Since $\eta_t >0$ and $\EE[\nabla \norm{L(\btheta_t)}^2_2] >0$ for all $t\in [T]\coloneqq \{1,\ldots, T\}$,
\begin{align*}
\min_{t\in [T]} \EE[\norm{\nabla L(\btheta_t)}^2_2]\left(\sum_{t=1}^T\eta_t\right) &= \sum_{t=1}^T \eta_t \min_{t^\prime\in [T]} \EE[\norm{\nabla L(\btheta_{t^\prime})}^2_2] \\
&\leq \sum_{t=1}^T \eta_t \EE[\norm{\nabla L(\btheta_t)}_2^2].
\end{align*}
This implies  that $$
\min_{t\in [T]}\EE[\|\nabla L(\btheta_{t})\|_2^2] \le \left(\sum_{t=1}^T\eta_t \right)^{-1} \left( (1+\sqrt{2}\varsigma ^{2}(1+a)R_T\exp(\sqrt{2}\varsigma ^{2} (1+a)R_{T-1}))\EE[L(\btheta_{1})]-\EE[L(\btheta_{T+1})] \right).
$$
Since  $R_{T-1} \le R_T \le \sum_{t=1}^{\infty}\eta_t^2 < \infty$ and $\EE[L(\btheta_{T+1})]\ge 0$, this implies  that there exists a constant $c$ independent of $(t,T)$ such that
 $$
 \min_{t\in [T]}\EE[\|\nabla L(\btheta_{t})\|_2^2] \le c\EE[L(\btheta_{1})] \left(\sum_{t=1}^T\eta_t \right)^{-1}  .
$$
\end{proof}

For example, if $\eta_t=\eta_1 t^{-q}$ with $q \in (0.5, 1)$ and $\eta_1>0$, then we have $\min_{t\in [T]}\EE[\|\nabla L(\btheta_{t})\|_2^2]=\Ocal(\frac{1}{T^{1-q}})$.


\section{Details on the Mixture of Gaussians Amortized Clustering Experiment}
\label{app:mog-details}

We used a modified version of the MoG amortized clustering dataset which was used by \citet{set-trans}. We modified the experiment, adding separate, random covariance parameters into the procedure in order to make a more difficult dataset. Specifically, to sample a single task for a problem with $K$ classes,

\begin{enumerate}
    \itemsep0em 
    \item Sample set size for the batch $N \sim U(\text{train set size} / 2, \text{train set size})$.
    \item Sample class priors $\pi \sim \text{Dirichlet}([\alpha_1, \ldots, \alpha_K])$ with $\alpha_1=\cdots = \alpha_K=1$.
    \item Sample class labels $z_i \sim \text{Categorical}(\pi)$ for $i=1, ..., N$.
    \item Generate cluster centers $\bs{\mu}_i = ({\mu}_{i,1}, {\mu}_{i,2})\in\RR^2$, where ${\mu}_{i,j} \sim U(-4, 4)$ for $i=1,...,K$ and $j=1,2$.
    \item Generate cluster covariance matrices $\bs{\Sigma}_i=\text{diag}({\sigma}_{i,1}, {\sigma}_{i,2})\in\RR^{2\times2}$, where ${\sigma}_{ij} \sim U(0.3, 0.6)$ for $i=1,...,K$ and $j=1,2$.  
    \item For all $z_n$, if $z_n = i$, sample data $\mb{x}_{n} \sim \mc{N}(\bs{\mu}_i, \bs{\Sigma}_i)$
\end{enumerate}

In our MoG experiments, we set $K=4$. The Motivational Example in \cref{motivation} also used the MoG dataset, and performed MBC testing of the set transformer corresponding to the procedure outlined in \cref{app:mbc-test-st}

\input{figures/motivation_2.tex}

\subsection{Streaming Settings}\label{app:streaming-settings}

The four total streaming settings in~\cref{motivation,motivation-2} are described below:
\begin{itemize}
    \itemsep0em
    \item \textbf{single point stream} $\rightarrow$ streams each point in the set one by one. This causes the most severe under-performance by non-MBC models. 
    \item \textbf{class stream} $\rightarrow$ streams an entire class at once. Models which make complex pairwise comparisons cannot compare the input class with any other clusters, thereby degrading performance of models such as the Set Transformer.
    \item \textbf{chunk stream} $\rightarrow$ streams 8 random points at a time from the dataset, Providing, random and limited information to non-MBC models. 
    \item \textbf{one each stream} $\rightarrow$ streams a set consisting of a single instance from each class. non-MBC models can see examples of each class, but with a limited sample size, therefore non-MBC models such as Set Transformer fail to make accurate predictions. 
\end{itemize}

\subsection{Experimental Setup}

We train each model for $50$ epochs, with each epoch containing $1000$ iterations. We use the Adam optimizer with a learning rate of $1 \cdot 10^{-3}$ and no weight decay. We do not perform early stopping. We make a single learning rate adjustment at epoch $35$ which adjusts the learning rate to $10^{-4}$. When measuring NLL for results, we measure the NLL of the full set of $1024$ points. Unless otherwise specified, UMBC models use the softmax activation function. We list the architectures in~\cref{mvn-enc-arc,mvn-dec-arc,mvn-umbc-dec-arc}. All models have an additional linear output which outputs $K \times 5$ parameters for the Gaussian mixture outlined in~\cref{eq:mog}.

\input{tables/mvn_arc.tex}

\section{Measuring the Variance of Pooled Features}
\label{app:pooling-ft-vars}

In \cref{embedding-var}, we show the quantitative effect on the pooled representation between the plain Set Transformer, UMBC+Set Transformer, FSPool and DiffEM. The UMBC model always shows 0 variance, while the non-MBC models produce variance between aggregated encodings of random partitions. For a single chunk, however, non-MBC models show no variance, as random partitions of a single chunk would be equivalent to permuting the elements within the chunk (\ie non-MBC models still produce an encoding which is permutation invariant). The variance increases drastically when the set is partitioned into two chunks and then the behavior differs between the non-MBC models. Set Transformer happens to show decreasing variance as the number of chunks increases. Note that as the number of chunks increases, the cardinality of each chunk decreases. Therefore, the variance decreases as the chunk cardinality also decreases, but this \emph{does not} indicate that the models is performing better. For example, in~\cref{motivation}, when a singleton set is input to the Set Transformer, the predictions become almost meaningless even though they may have lower variance. The procedure for aggregating the encodings of set partitions for the non-MBC models is outlined in \cref{app:mbc-test-st}.

\input{tables/embedding_var.tex}

To perform this experiment, we used a randomly initialized model with $128$ hidden units, and sampled $256$ set elements from four different distributions in order to make a total set size of $1024$. We then created 100 random partitions for various chunk sizes. Chunk sizes and distributions are shown in~\cref{tab:app:encoding-variance-distributions}. We then encode the whole set in chunks and and report the observed variance over the 100 different random partitions at each of the various chunk sizes (\cref{embedding-var}). Note that the encoded set representation is a vector and \cref{embedding-var} shows a scalar value. To achieve this, we take the feature-wise variance over the 100 encodings and report the mean and standard deviation over the feature dimension. Specifically, given $\mathbf{Z} = [\mathbf{z}_1 \cdots \mathbf{z}_{100}]^\top \in \mathbb{R}^{100 \times 128}$ representing all 100 encodings with $\mathbf{z}_i = (z_{i,1},\ldots, z_{i,128})$, we compute feature-wise variance as 
\begin{gather*}
    \hat{z}_j = \frac{1}{(100 - 1)}\sum_{i=1}^{100} (z_{i,j}-\mu_j)^2, \quad \mu_j = \frac{1}{100}\sum_{i=1}^{100} z_{i,j}
\end{gather*}
for $j=1,\ldots, 128$.
We then achieve the values of y-axis and error bars in \cref{embedding-var} by a  mean and standard deviation over the feature dimension,
\begin{gather}
    y = \frac{1}{128} \sum_i \hat{z}_i, \quad\quad y_\sigma = \sqrt{\frac{1}{(128 - 1)} \sum_i (\hat{z}_i - y)^2 }.
\end{gather}

\section{A Note on MBC Testing of non-MBC models}
\label{app:mbc-test-st}

In the qualitative experiments \cref{motivation,embedding-var}, we apply MBC testing to non-MBC models in order to study the effects of using non-MBC models in MBC settings. Non-MBC models do not prescribe a way to accomplish this in the original works, so we took the approach of processing each chunk up until the pooled representation. We then performed mean pooling over the encoded chunks in the following way. Let $X$ be an input set and let $\zeta(X)=\{X_1, \ldots, X_n\}$ be a partition of the set $X$, \ie $X = \bigcup_{j=1}^N X_j$ with $X_i \cap X_j=\emptyset$ for $i\neq j$. Denote $\tilde{f}_\theta$ a non-MBC set encoding function, then our pseudo-MBC testing procedure is as follows,
\begin{gather}
Z = \frac{1}{N} \sum_{j=1}^N \tilde{f}_\theta(X_j)
\end{gather}

\input{figures/app_celeba.tex}
\section{Details on the Image Completion Experiments}
\subsection{Additional Experimental Results}
\label{app:celeba}
In figure~\cref{fig:celeba-exp-app}, we evaluate our proposed unbiased full set gradient approximation algorithm (\textcolor{red}{\textbf{red}}) with Deepsets, Slot Set Encoder (SSE) and UMBC + Set Transformer (ST) and compare our algorithm against the one training with a randomly sampled subset of 100 elements , which is a biased estimator, (\textcolor{darkpastelgreen}{\textbf{green}}) and the one computing full set gradient (\textcolor{celestialblue}{\textbf{blue}}). Across all models, our unbiased estimator significantly outperforms the models trained with a randomly sampled subset. Notably, the model trained with our proposed algorithm is indistinguishable from the model trained with full set gradient while our method only incurs constant memory overhead for any set size. These empirical results again verify efficiency of our unbiased full set gradient approximation. 

\subsection{Experimental Setup}
We train all models on CelebA  dataset for 200,000 steps with Adam optimizer~\citep{adam} and 256 batch size but no weight decay. We set the learning rate to $5\cdot 10^{-4}$ and use a cosine annealing learning rate schedule. In~\cref{celeba-enc-arc,celeba-dec-arc}, we specify the architecture of Conditional Neural Process with UMBC + Set Transformer. We use $k=128$ slots and set dimension of each slot  to $d_s=128$. For the attention layer, we use the softmax for the activation function $\sigma$ and set the dimension of attention output to $d=128$.
As an input to the set encoder, we concatenate the coordinates of each $\rvx_{i,c_j} \in \RR^{2}$ and the corresponding pixel value $y_{i, c_j}\in\RR^3$ from the context $X_i$ for $j=1,\ldots, N_i$, resulting in a $N_i \times 5$ matrix.
\input{tables/celeba_arc.tex}

\section{Details on the  Long Document Classification Experiments}
We train all models for 30 epochs with AdamW optimizer~\citep{adamw} and batch size 8. We use constant learning rate  $5\cdot 10^{-5}$. For our UMBC model, we pretrain the model while freezing BERT for 30 epochs and finetune the whole model for another 30 epochs.  In~\cref{text-arc}, we specify architecture of UMBC + BERT~\citep{bert} without positional encoding. We use $k=256$ slots and set dimension of each slot to $d_s=128$. We use slot-sigmoid for the activation function $\sigma$ and set the dimension of the attention output to $d=768$. 
\input{tables/text_arc.tex}

\section{Details on the Camelyon16 Experiments}
\label{app:camelyon-details}

\input{figures/camelyon_patch_example.tex}
The Camelyon16 Whole Slide Image dataset consists of 270 training instances and 129 validation instances. The dataset was created for a competition, and therefore the test set is hidden. We therefore follow the example set by previous works~\cite{ds-mil} and report performance achieved on the validation set. For preprocessing, we consider the $20\times$ slide magnification setting, and use OTSU's thresholding method to detect regions containing tissue within the WSI. We then split the activated regions into non overlapping patches of size $256\times256$. An example of single input patches can be seen in~\cref{fig:camelyon-patch-example}. The largest input set contains $37,345$ image patches which are each $\in \RR^{256\times256\times3}$. All patch extraction code can be found in the supplementary file. \cref{tab:camelyon-stats} contains statistics related to the numbers of patches per input for the training and the test set as well as the distribution of positive and negative labels.  

\subsection{Experimental Setup}

We use a ResNet18~\citep{resnet} which was pretrained with self-supervised contrastive learning~\citep{simclr} by~\citet{ds-mil}. The pretrained ResNet18 weights can be downloaded from \href{https://github.com/binli123/dsmil-wsi}{this repository}. Following the classification experiments done by~\citet{set-trans}, we place dropout layers before and after the PMA layer of the Set Transformer in our UMBC model. We will describe our pretraining and finetuning steps below in detail. 

\input{tables/camelyon_arc.tex}

\paragraph{Pretraining.} For pretraining, we extract the features from the pretrained ResNet18 and only train the respective MIL models~(\cref{camelyon-mil-arc}) on the extracted features. We pretrain for 200 epochs with the Adam optimizer which uses a learning rate of $5 \cdot 10^{-4}$ and a cosine annealing learning rate decay which reaches the minimum at $5 \cdot 10^{-6}$. We use $\beta_1 = 0.5$, and $\beta_2 = 0.9$ for Adam. We train with a batch size of 1 on a single GPU, and save the model which showed the best performance on the validation set, where the performance metric is $\left(\text{Accuracy} + \text{AUC} \right) / 2$. Other details can be found in~\cref{sec:mil}. These results can be seen in the left column of~\cref{tab:camelyon}.

\paragraph{Finetuning.}


For finetuning, we use our unbiased gradient approximation algorithm with a chunk size of $256$. We freeze the pretrained MIL head and only finetune the backbone resnet model. Therefore, we sequentially process each $256$ chunk for each input set until the entire set has been processed. We train for 10 total epochs, and use the AdamW optimizer with a learning rate of $5 \cdot 10^{-5}$, and a weight decay of $1 \cdot 10^{-2}$ which is not applied to bias or layernorm parameters. We use a one epoch linear warmup, and then a cosine annealing learning rate decay at every iteration which reaches a minimum at $5 \cdot 10^{-6}$. We train on 1 GPU, with a batch size of 1 and with a single instance on each GPU. 

\input{tables/camelyon_stats.tex}

\clearpage
\section{Generalizing Attention Activations}
\label{sec:method-attn-acts}

As shown in~\cref{eq:f-and-normalization}, any attention activation function which can be expressed as a strictly positive elementwise function combined with sum decomposable normalization constants $\nu_p$ and $\bbf_\theta$ represents a valid attention activation function. \cref{tab:attn-acts} shows 5 such functions with their respective normalization constants, although there are an infinite number of possible functions which can be used.

The softmax operation we propose $h_1: \RR^d \mapsto (0, 1)^d$ which is outlined immediately before \cref{thm:mbc} is mathematically equivalent to the standard softmax $h_2: \RR^d \mapsto (0, 1)^d$ which is commonly implemented in deep learning libraries because $f$ and $g$ have the same domain, the same codomain, and $h_1(\mathbf{x}) = h_2(\mathbf{x})$ for all $\mathbf{x} \in \RR^d$. Therefore the functions are mathematically equivalent, even though the implementations are not. Our proposed function $h_1$ requires separately applying the exponential, and storing and updating the normalization constant while $h_2$ is generally implemented in such a way that everything is done in a single operation.

\input{tables/attention_acts.tex}

\section{Algorithm}
We outline our unbiased full set gradient approximation here.
\input{figures/algo.tex}

%% file: figures/motivation_2.tex
\begin{figure*}
\vspace{-0.1in}
\centering
    \begin{subfigure}{0.3\textwidth}
        \centering
        \includegraphics[width=\linewidth]{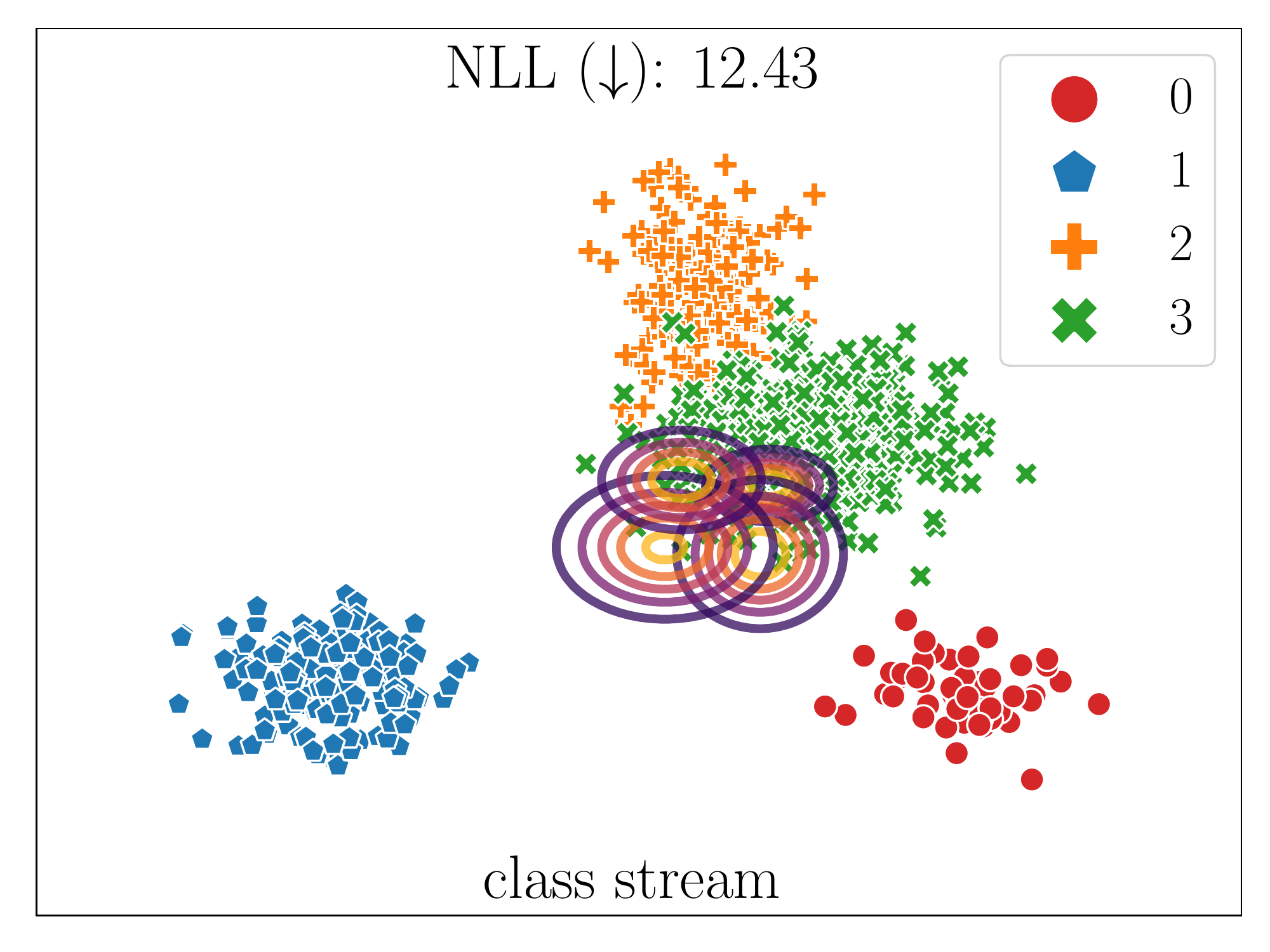}
        \captionsetup{justification=centering,margin=0.5cm}
        \vspace{-0.25in}
        \caption{\small Set Transformer}
        \label{st-class-stream}	
    \end{subfigure}%
    \begin{subfigure}{0.3\textwidth}
	\centering
	\includegraphics[width=\linewidth]{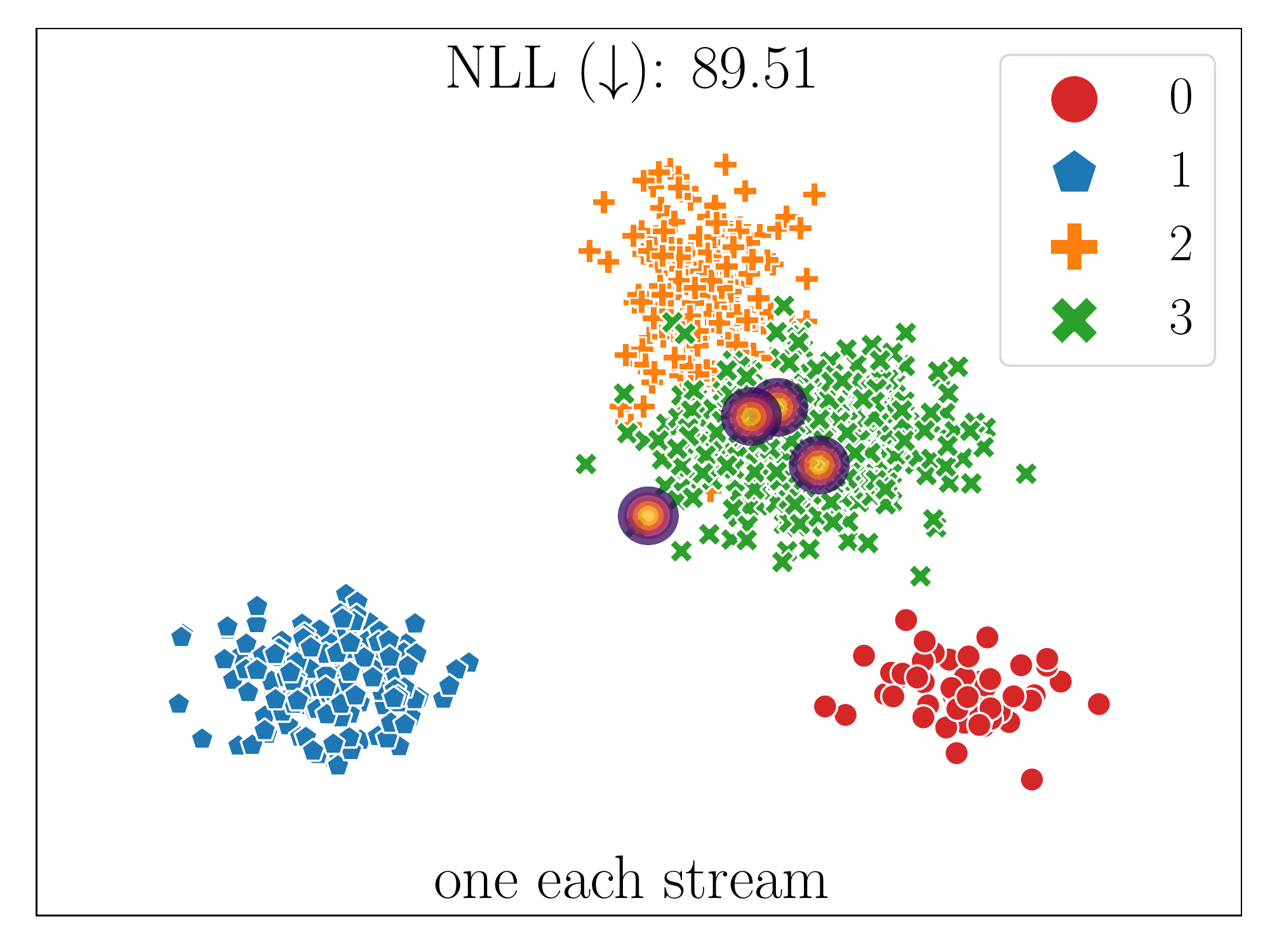}
	\captionsetup{justification=centering,margin=0.5cm}
	\vspace{-0.25in}
	\caption{\small Set Transformer} 
	\label{st-one-each-stream}
    \end{subfigure}%
    \begin{subfigure}{0.3\textwidth}
	\centering
	\includegraphics[width=\linewidth]{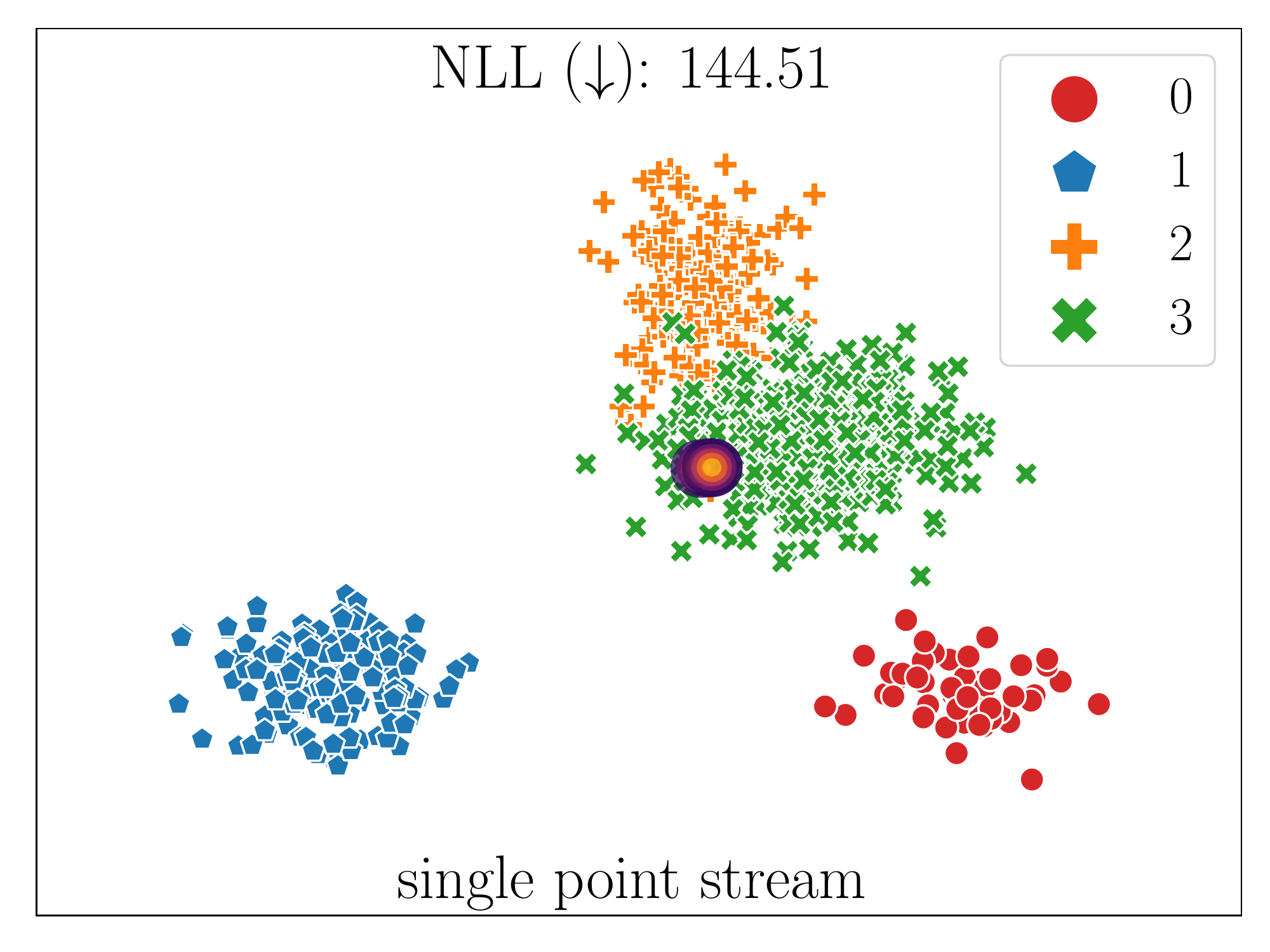}%
	\captionsetup{justification=centering,margin=0.5cm}
	\vspace{-0.25in}
	\caption{\small Set Transformer}
	\label{st-single-point-stream}
    \end{subfigure}
    \begin{subfigure}{0.3\textwidth}
	\centering
	\includegraphics[width=\linewidth]{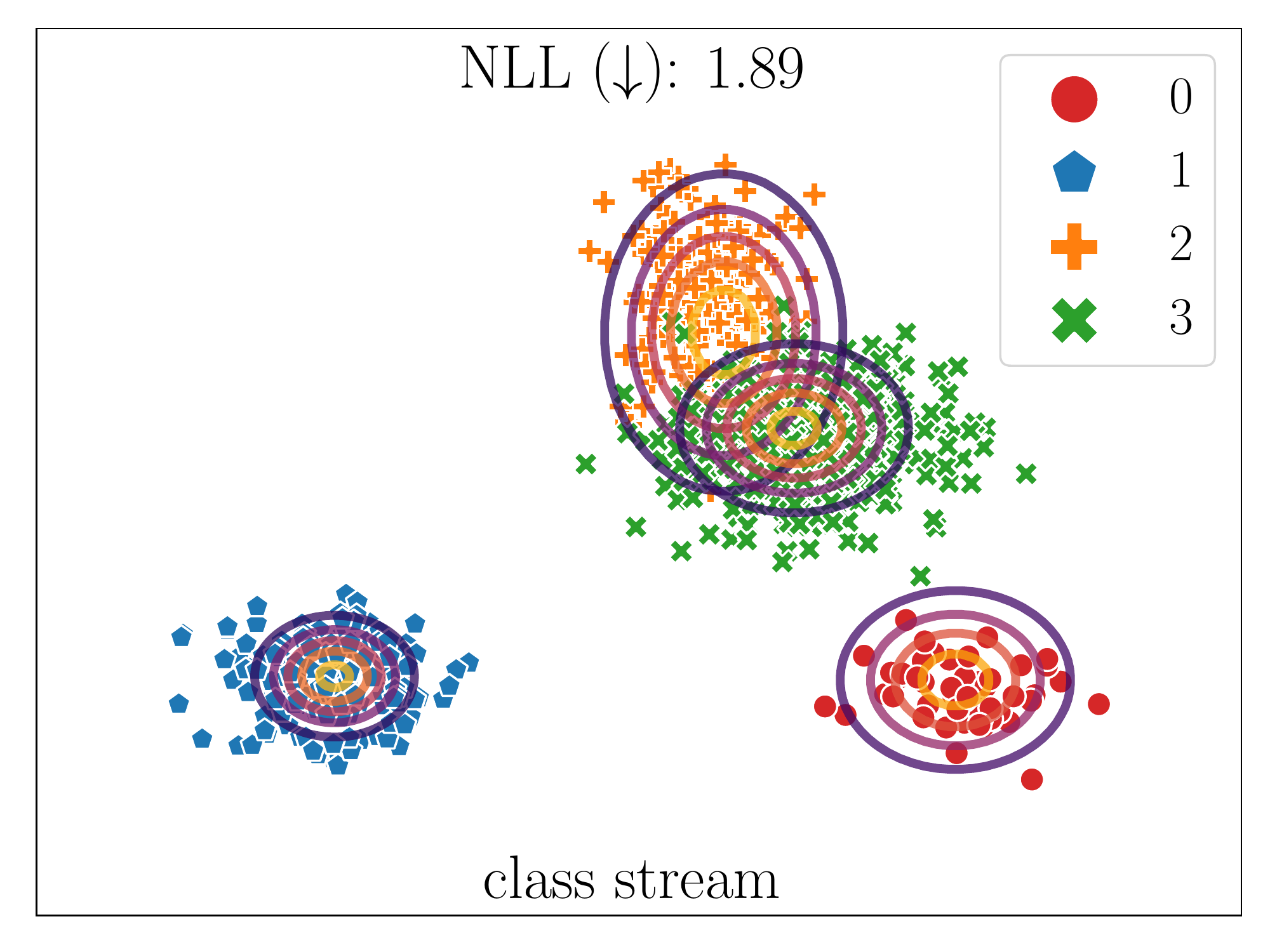}
	\captionsetup{justification=centering,margin=0.5cm}
	\vspace{-0.25in}
	\caption{\small UMBC}
	\label{umbc-class-stream}
    \end{subfigure}%
    \begin{subfigure}{0.3\textwidth}
	\centering
	\includegraphics[width=\linewidth]{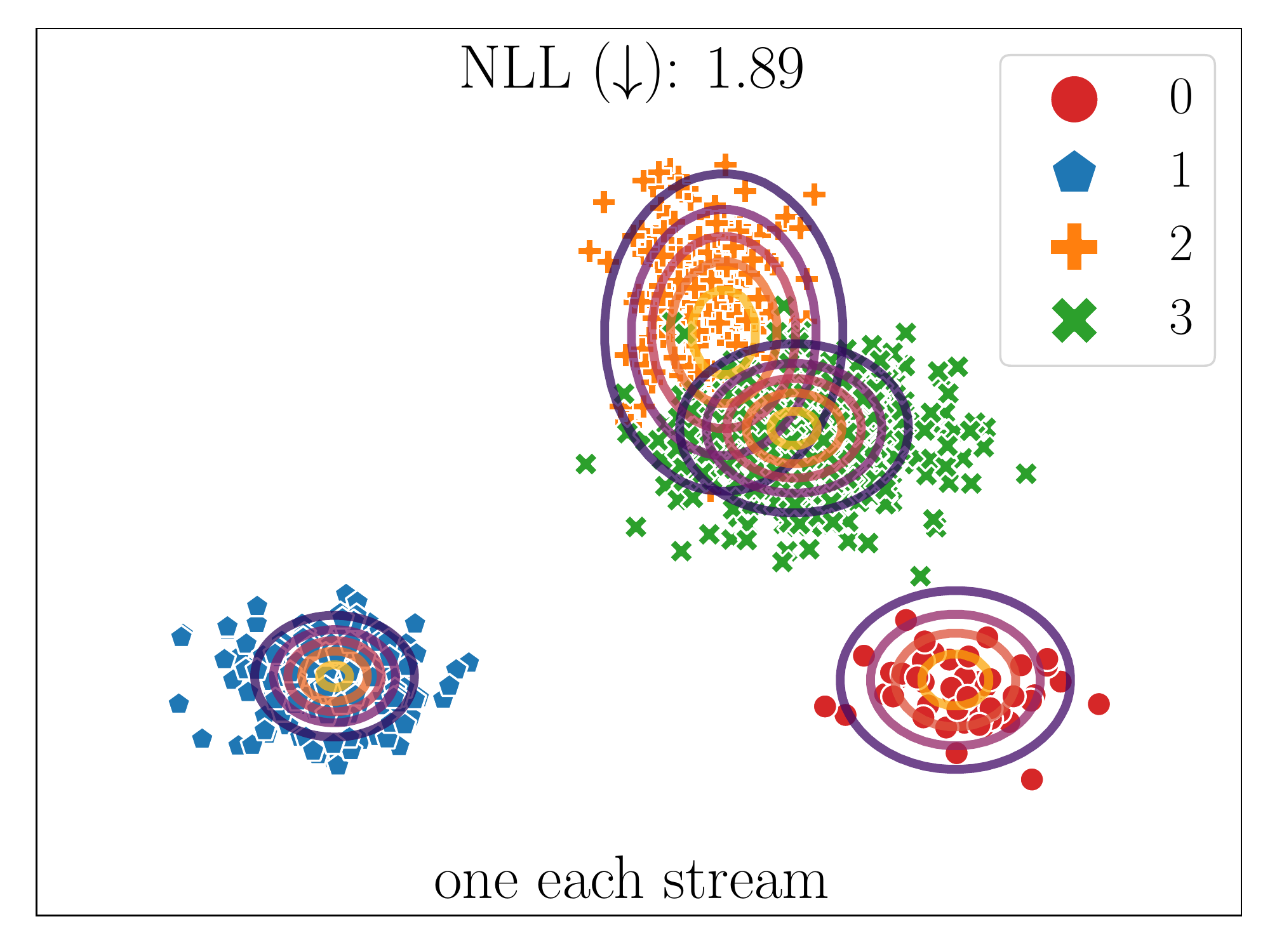}
	\captionsetup{justification=centering,margin=0.5cm}
	\vspace{-0.25in}
	\caption{\small UMBC}
	\label{umbc-one-each-stream}
    \end{subfigure}%
    \begin{subfigure}{0.3\textwidth}
	\centering
	\includegraphics[width=\linewidth]{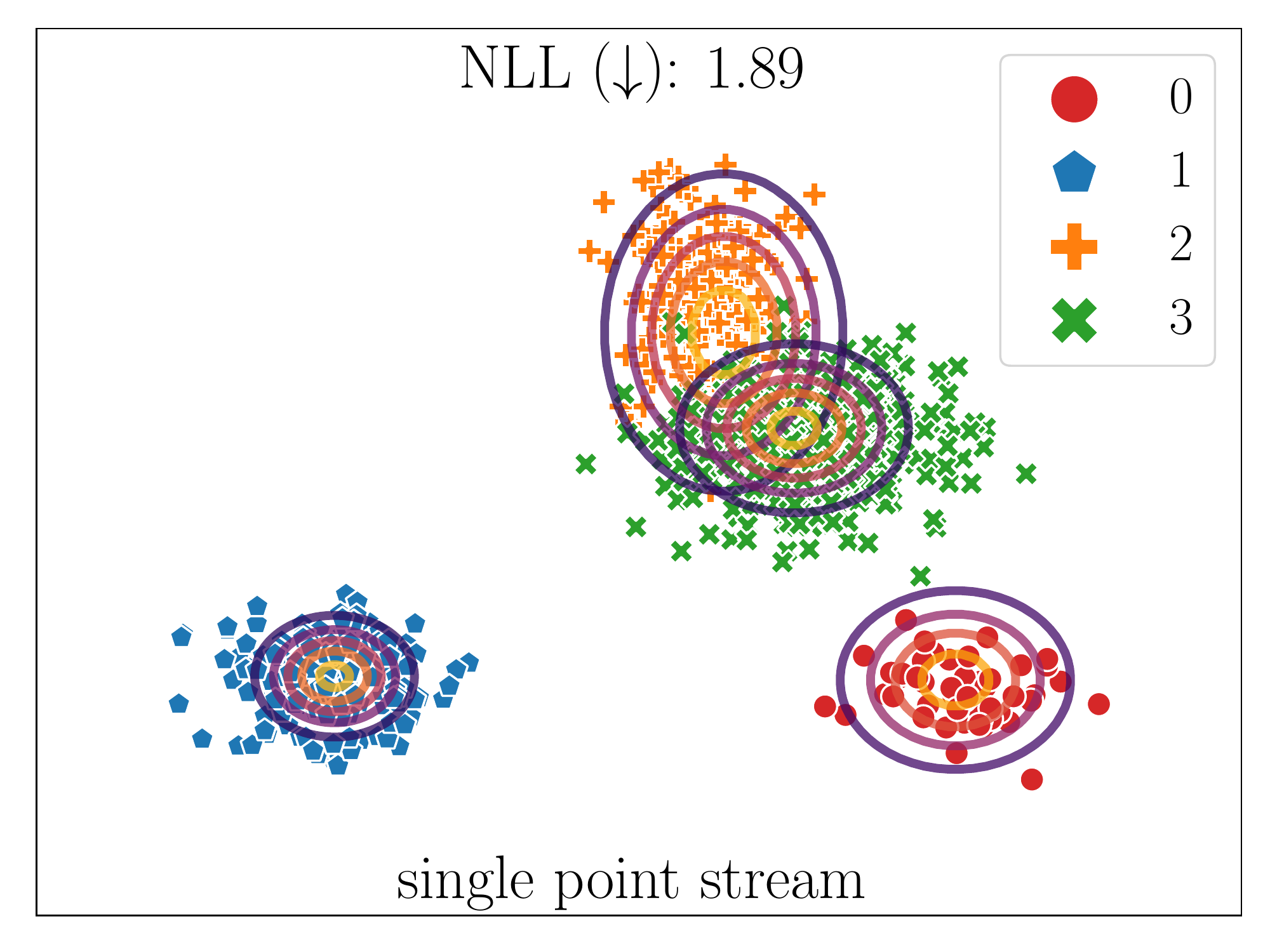}
	\captionsetup{justification=centering,margin=0.5cm}
	\vspace{-0.25in}
	\caption{\small UMBC}
	\label{umbc-single-point-stream}
    \end{subfigure}%
    \caption[Reconstruction NLL]{\small
    \textbf{Top Row}: Set Transformer provides inconsistent predictions on streaming sets when inputs cannot be stored directly. \textbf{Bottom Row}: UMBC+Set Transformer gives consistent predictions in all streaming settings. 
    }
    \label{motivation-2}
    \vspace{-0.1in}
\end{figure*}

%% file: tables/mvn_arc.tex
\begin{figure*}[t!]
    \begin{minipage}[t]{0.4\textwidth}
	\small
	\centering
	\captionof{table}{MVN generic model (Used by all encoders)}
	\resizebox{0.8\textwidth}{!}{\begin{tabular}{lll}
	    \toprule
		{\textbf{Output Size}} & {\textbf{Layers}} & {\textbf{Amount}} \\
		\midrule[0.8pt]
		{$N_i \times 2$} & {Input Set} & {$\times1$} \\
		{$N_i\times  128$} & {Linear(2, 128), ReLU} & {$\times1$} \\
		{$N_i\times  128$} & {Set Encoder} & {$\times1$} \\
		{$K\times 5$} & {Decoder} & {$\times3$} \\
  \bottomrule
	\end{tabular}
 }
	\label{mvn-enc-arc}
\hfill
\end{minipage}
\begin{minipage}[t]{0.6\textwidth}
	\small
	\centering
	\captionof{table}{Set Encoder Specific settings for baseline models.}
\resizebox{0.99\textwidth}{!}{	
 \begin{tabular}{lll|ll}
	    \toprule
		{\textbf{Name}} & {\textbf{Set Encoder}} & {\textbf{Output Size}} & {\textbf{Set Decoder}} & {\textbf{Output Size}} \\
		\midrule[0.8pt]
		{DeepSets~\citep{deepsets}} & {Mean Pooling} & {$128$} & {Linear, ReLU} & {$128$} \\
            {SSE~\citep{mbc}} & {Slot Set Encoder} & {$K \times 128$} & {Linear, ReLU} & {$128$} \\
            {FSPool~\citep{fspool}} & {Featurewise Sort Pooling} & {$128$} & {Linear, ReLU} & {$128$} \\
            {Diff. EM.~\citep{diff-em}} & {Expectation Maximization Layer} & {$1286$} & {Linear, ReLU} & {128} \\
            {Set Transformer~\citep{set-trans}} & {Pooling by Multihead Attention} & {$K \times 128$} & {Set Attention Block} & {$K \times 128$} \\
            \bottomrule
	\end{tabular}
 }
	\label{mvn-dec-arc}
\end{minipage}
\begin{minipage}[t]{1.0\textwidth}
	\small
	\centering
	\captionof{table}{Set Encoder Specific settings for UMBC models. UMBC models account for the extra encoder by using fewer layers in the decoder (2 layers instead of 3).}
\resizebox{0.8\textwidth}{!}{	
 \begin{tabular}{lll|ll|ll}
	    \toprule
		{\textbf{Name}} & {\textbf{MBC Set Encoder}} & {\textbf{Output Size}} & {\textbf{non-MBC Set Encoder}} & {\textbf{Output Size}} & {\textbf{Set Decoder}} & {\textbf{Output Size}} \\
		\midrule[0.8pt]
            {(Ours) UMBC+FSPool} & {UMBC Layer} & {$K \times 128$} & {Featurewise Sort Pooling} & {$128$} & {Linear, ReLU} & {$128$} \\
            {(Ours) UMBC+Diff EM} & {UMBC Layer} & {$K \times 128$} & {Expectation Maximization Layer} & {$1286$} & {Linear, ReLU} & {$128$} \\
            {(Ours) UMBC+Set Transformer} & {UMBC Layer} & {$K \times 128$} & {Set Attention Block} & {$K \times 128$} & {Linear, ReLU} & {$K \times 128$} \\
            \bottomrule
	\end{tabular}
 }
	\label{mvn-umbc-dec-arc}
\end{minipage}
\end{figure*}

%% file: tables/embedding_var.tex
\begin{figure*}[ht]
    \begin{minipage}[b]{0.5\linewidth}
        \centering
        \caption{Distributions used in sampling random inputs for the encoding variance experiment in~\cref{motivation}}
        \begin{tabular}{lcc}
        \toprule
        Distribution & Dimension & Number of Points \\
        \midrule
        Normal(0, 1) & $128$ & $256$ \\
        Uniform(-3, 3) & $128$ & $256$ \\ 
        Exponential(1) & $128$ & $256$ \\ 
        Cauchy(0, 1) & $128$ & $256$ \\
        \bottomrule
        \end{tabular}
    \end{minipage}
    \begin{minipage}[b]{0.5\linewidth}
        \centering
        \caption{The number of chunks and elements per chunk.}
        \resizebox{\linewidth}{!}{
        \begin{tabular}{lcccccc}
        \toprule
        Number of Chunks & $1$ & $2$ & $4$ & $8$ & $16$ & $32$ \\
        Elements per Chunk & $1024$ & $512$ & $256$ & $128$ & $64$ & $32$ \\ 
        \bottomrule
        \end{tabular}
        }
    \end{minipage}
    \label{tab:app:encoding-variance-distributions}
\end{figure*}

%% file: figures/app_celeba.tex
\begin{figure*}[t]
\centering
    \begin{subfigure}{0.5\textwidth}
        \centering
        \includegraphics[width=\linewidth]{images/celeba/legend.pdf}
        \vspace{-0.2in}
    \end{subfigure} \\
    \begin{subfigure}{0.33\textwidth}
		\centering
		\includegraphics[width=\linewidth]{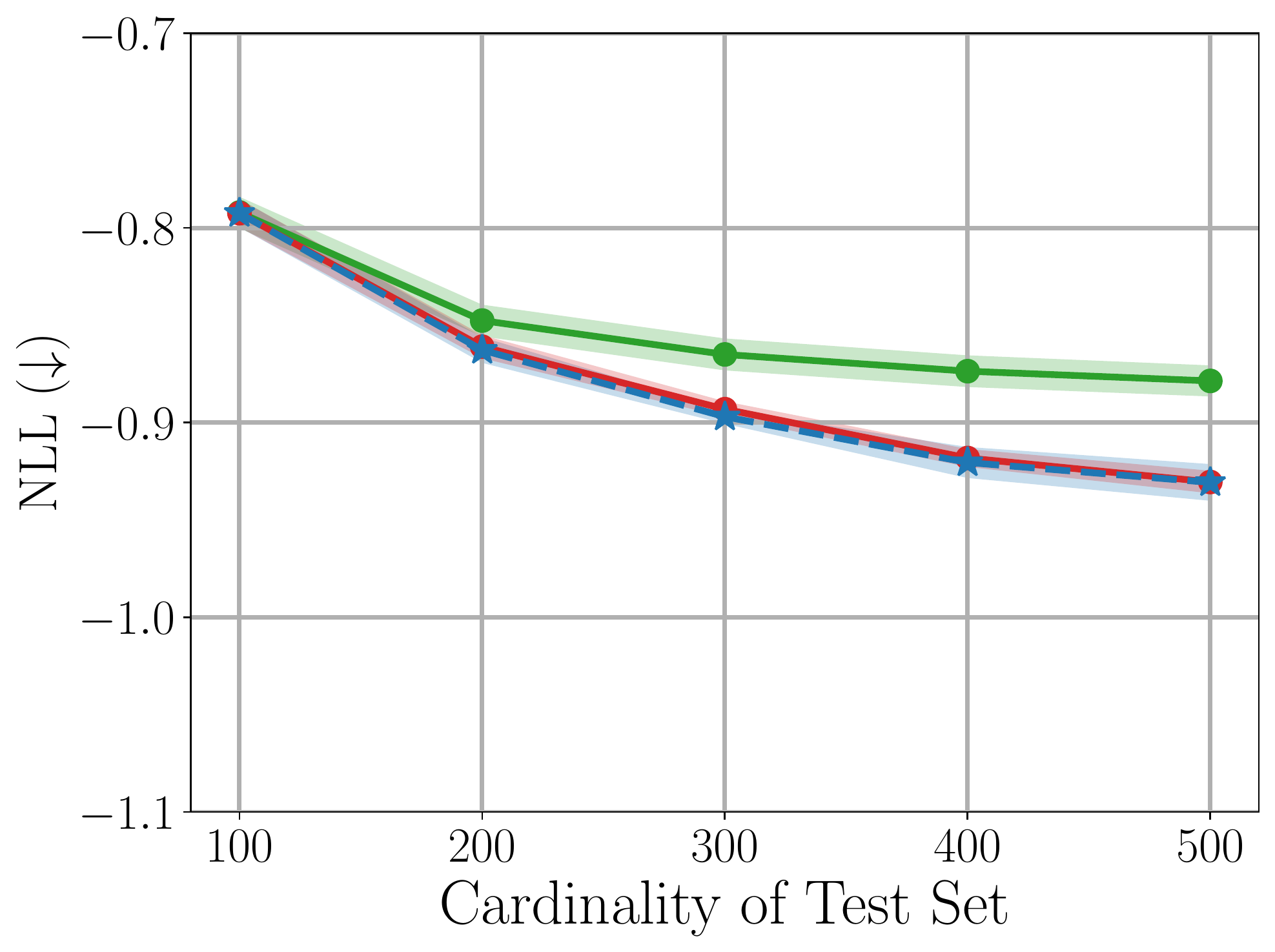}
		\captionsetup{justification=centering,margin=0.5cm}
		\vspace{-0.25in}
		\caption{\small Deepsets}
		\label{fig:celeba-deepset}
	\end{subfigure}%
    \begin{subfigure}{0.33\textwidth}
    \centering
    \includegraphics[width=\linewidth]{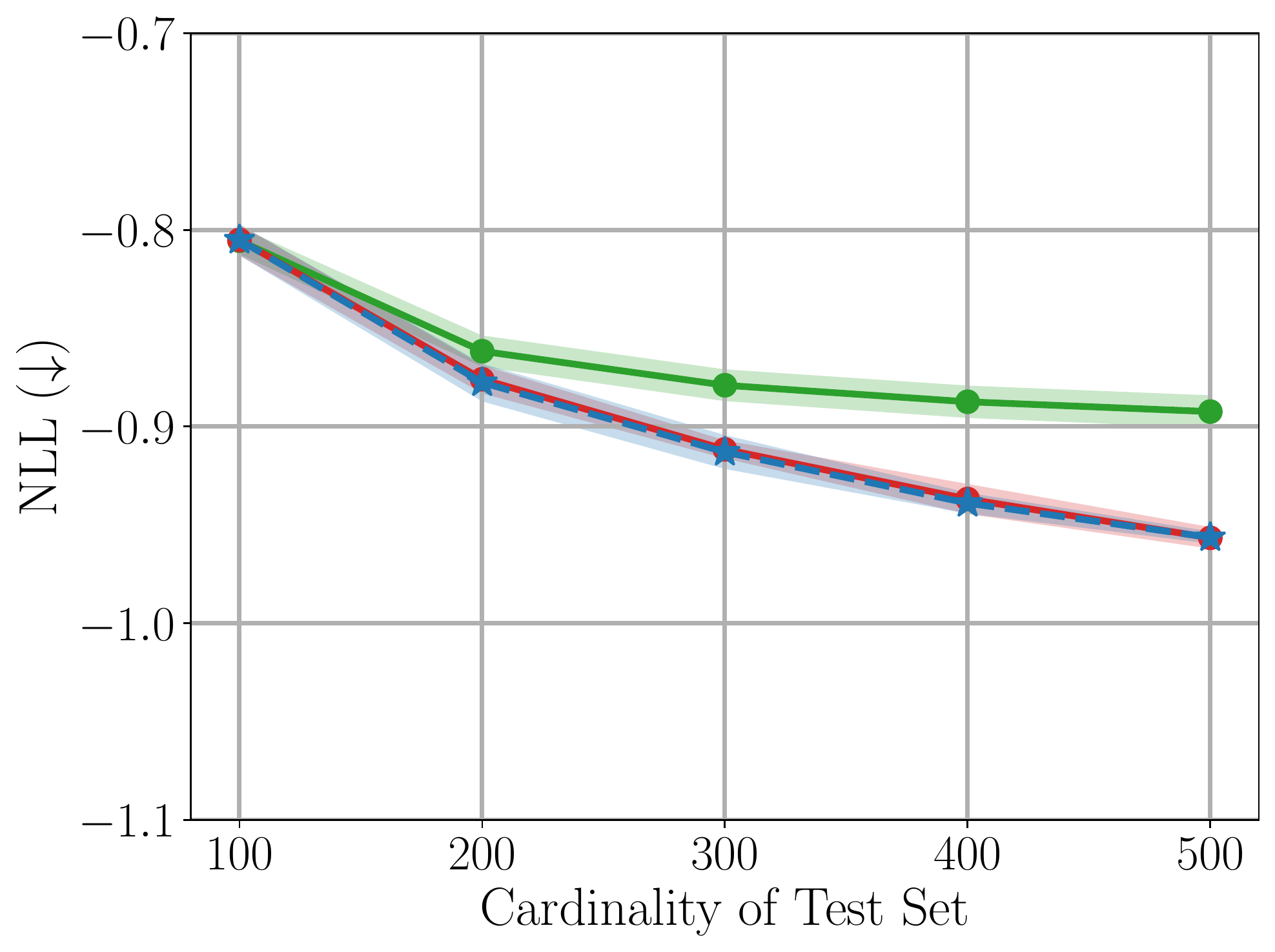}
    \captionsetup{justification=centering,margin=0.5cm}		
    \vspace{-0.25in}
    \caption{\small Slot Set Encoder}
    \label{fig:celeba-sse}	
    \end{subfigure}
	\begin{subfigure}{0.33\textwidth}
		\centering
		\includegraphics[width=\linewidth]{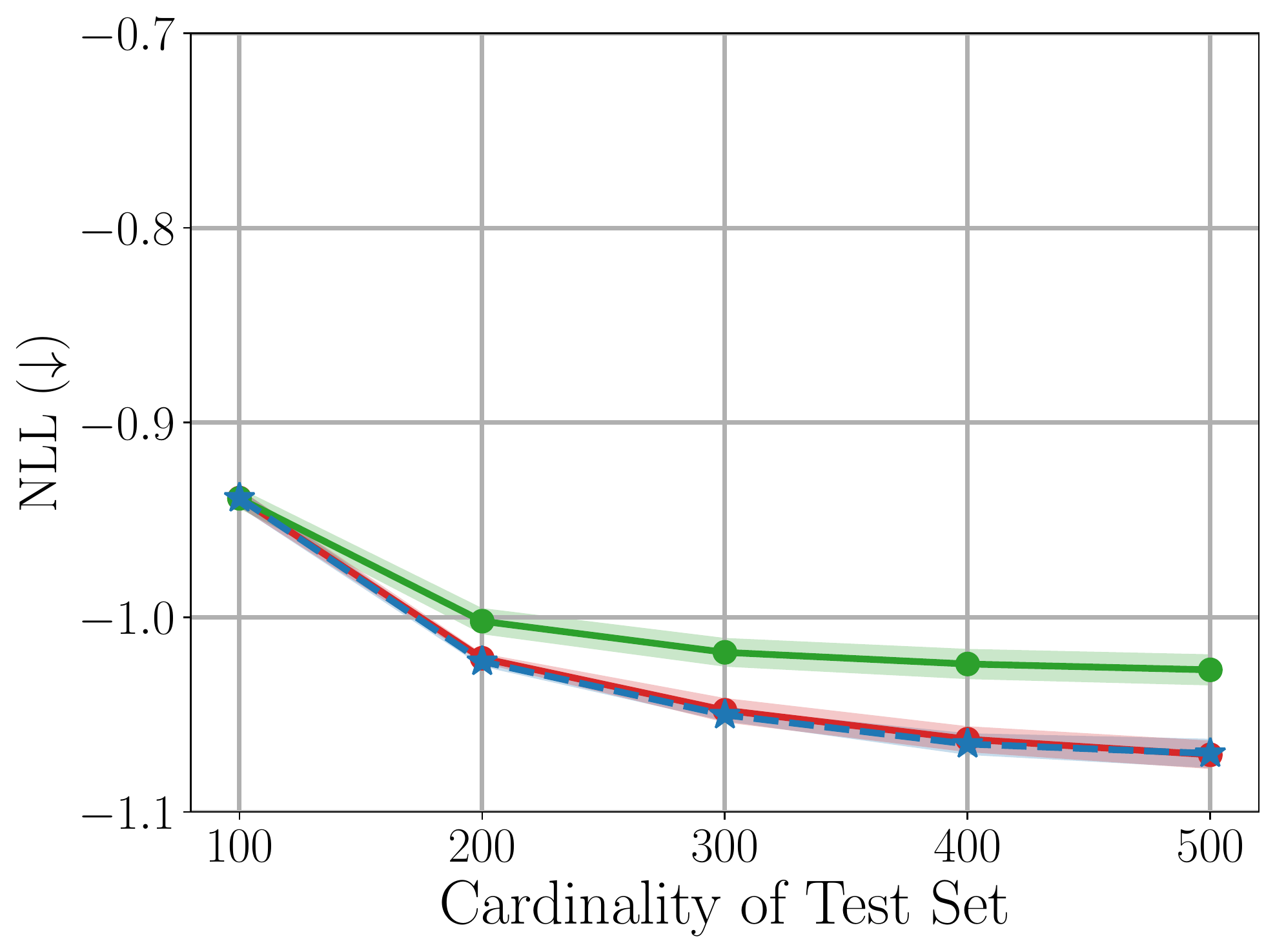}
		\captionsetup{justification=centering,margin=0.5cm}
		 \vspace{-0.25in}
		\caption{\small UMBC + ST} 
		\label{fig:celeba-umbc-2}
	\end{subfigure}%
\vspace{-0.15in}	
 \caption{\small Performance of 
    \textbf{(a)} DeepSets \textbf{(b)} SSE, and \textbf{(c)} UMBC with varying set sizes for image completion on CelebA dataset.}
	\label{fig:celeba-exp-app}
\vspace{-0.1in}
\end{figure*}

%% file: tables/celeba_arc.tex
\begin{figure*}[t!]
    \begin{minipage}[t]{0.52\textwidth}
	\small
	\centering
	\captionof{table}{UMBC Set Encoder  of Conditional Neural Process.}
	\resizebox{0.99\textwidth}{!}{\begin{tabular}{ll}
	    \toprule
		{\textbf{Output Size}} & {\textbf{Layers}} \\
		\midrule[0.8pt]
		{$N_i \times 5$} & {Input Context Set} \\
		{$N_i\times  128$} & {Linear(5, 128), ReLU} \\
		{$N_i\times  128$} & {Linear(128, 128), ReLU} \\
		{$N_i\times 128$} & {Linear(128, 128), ReLU} \\
            {$N_i\times 128$} & {Linear(128, 128), ReLU} \\
            {$128\times  128$} & {UMBC Layer} \\
            {$128\times  128$} & {Layer Normalization} \\
            {$128\times  128$} & {Set Attention Block~\citep{set-trans}} \\
            {$128\times  128$} & {Set Attention Block} \\
            {$128$} & {Pooling by Multihead Attention~\citep{set-trans}} \\
  \bottomrule
	\end{tabular}
 }
	\label{celeba-enc-arc}
\hfill
\end{minipage}
\begin{minipage}[t]{0.48\textwidth}
	\small
	\centering
	\captionof{table}{Decoder of Conditional Neural Process.}
\resizebox{0.99\textwidth}{!}{	
 \begin{tabular}{ll}
	    \toprule
		{\textbf{Output Size}} & {\textbf{Layers}} \\
		\midrule[0.8pt]
		{$128, 1024\times 2$} & {Input Set Representation and Coordinates} \\
		{$1024\times 130$} & {Tile \& Concatenate} \\
            {$1024\times  128$} & {Linear(130, 128), ReLU} \\
            {$1024\times  128$} & {Linear(128, 128), ReLU} \\
		{$1024\times  128$} & {Linear(128, 128), ReLU} \\
		{$1024\times  128$} & {Linear(128, 128), ReLU} \\
            {$1024\times  6$} & {Linear(128, 6)} \\
            \bottomrule
	\end{tabular}
 }
	\label{celeba-dec-arc}
\end{minipage}
\end{figure*}

%% file: tables/text_arc.tex
\begin{table}[ht]
    \vspace{-0.1in}
    \centering
    \caption{UMBC Set Encoder and BERT decoder for long document classification.}
    \resizebox{0.5\textwidth}{!}{\begin{tabular}{ll}
	    \toprule
		{\textbf{Output Size}} & {\textbf{Layers}} \\
		\midrule[0.8pt]
		{$N_i$} & {Input Document} \\
		{$N_i\times  768$} & {Word Embedding} \\
		{$N_i\times  768$} & {Layer Normalization} \\
		{$N_i\times  768$} & {Linear(768,768), ReLU} \\
            {$N_i\times  768$} & {Linear(768,768), ReLU} \\
            {$256\times  768$} & {UMBC Layer} \\
            {$256\times  768$} & {Layer Normalization} \\
            {$256 \times 768$} & {BERT w/o Positional Encoding~\citep{bert}} \\
            {$768$} & {[CLS] token Pooler}\\
            {$4271$} & {Dropout(0.1), Linear(768, 4271)} \\
  \bottomrule
	\end{tabular}
 }
    \label{text-arc}
    \vspace{-0.2in}
\end{table}

%% file: figures/camelyon_patch_example.tex
\begin{figure*}[t]
\centering
    \begin{subfigure}{0.23\textwidth}
	\centering
	\includegraphics[width=0.9\linewidth]{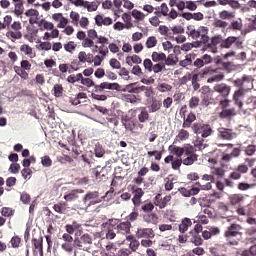}
	\captionsetup{justification=centering,margin=0.5cm}
    \end{subfigure}%
    \begin{subfigure}{0.23\textwidth}
        \centering
        \includegraphics[width=0.9\linewidth]{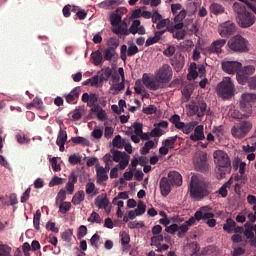}
        \captionsetup{justification=centering,margin=0.5cm}		
    \end{subfigure}%
    \begin{subfigure}{0.23\textwidth}
	\centering
	\includegraphics[width=0.9\linewidth]{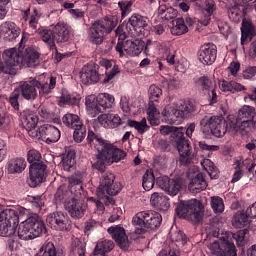}
	\captionsetup{justification=centering,margin=0.5cm}
    \end{subfigure}%
    \begin{subfigure}{0.23\textwidth}
	\centering
	\includegraphics[width=0.9\linewidth]{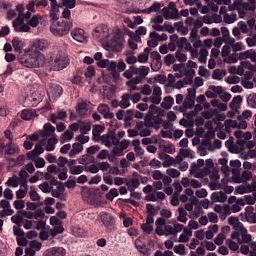}
	\captionsetup{justification=centering,margin=0.5cm}
    \end{subfigure}%
 \vspace{-0.1in}
 \caption{\small Examples of $4$ single patches from the Camelyon16 dataset. On average, each set in the dataset contains over $9,300$ patches like the ones pictured above.}
	\label{fig:camelyon-patch-example}
\vspace{-0.1in}
\end{figure*}

%% file: tables/camelyon_arc.tex
\begin{figure*}[t!]
    \begin{minipage}[t]{0.5\textwidth}
	\small
	\centering
	\captionof{table}{Camelyon16 generic model (Used by all encoders)}
	\resizebox{0.8\textwidth}{!}{\begin{tabular}{llll}
	    \toprule
		{\textbf{Output Size}} & {\textbf{Layers}} & {\textbf{Name}} & {\textbf{Amount}} \\
		\midrule[0.8pt]
		{$N_i \times 256 \times 256 \times 3$} & {Input Set} & {Bag of Instances} & {$\times1$} \\
		{$N_i\times  512$} & {ResNet18(InstanceNorm)} & {Feature Extractor} & {$\times1$} \\
		{$N_i\times  256$} & {Linear, ReLU, Linear} & {Projection} & {$\times1$} \\
            \midrule
		{$1$} & {Linear, Max Pooling} & {Instance Classifier} & {$\times1$} \\
            \midrule
            {$N_i \times 128$} & {Set Encoding Function} & {Bag Classifier} & {$\times1$} \\
            {$1$} & {Set Decoder} & {Bag Classifier} & {$\times1$} \\
  \bottomrule
	\end{tabular}
 }
	\label{camelyon-generic-arc}
\hfill
\end{minipage}
\begin{minipage}[t]{0.5\textwidth}
	\small
	\centering
	\captionof{table}{Camelyon16 Bag Classifier Models.}
\resizebox{0.99\textwidth}{!}{	
 \begin{tabular}{lll}
	    \toprule
		{\textbf{Name}} & {\textbf{Set Encoder}} & {\textbf{Output Size}} \\
		\midrule[0.8pt]
		{AB-MIL} & {Gated Attention~\citep{ab-mil}} & {$1$} \\
            {DSMIL} & {DS-MIL Aggregator~\citep{ds-mil}} & {$1$} \\
            {Deepsets} & {Max Pooling $\rightarrow$ Linear $\rightarrow$ ReLU $\rightarrow$ Linear} & {$1$} \\
            {Slot Set Encoder} & {SSE~\citep{mbc}$\rightarrow$ Linear $\rightarrow$ ReLU $\rightarrow$ Linear} & {$1$} \\
            {UMBC+SetTransformer} & {UMBC(K=$64$) $\rightarrow$ SAB $\rightarrow$ PMA(K=$1$, p=0.5) $\rightarrow$ Linear} & {$1$} \\
            \bottomrule
	\end{tabular}
 }
	\label{camelyon-mil-arc}
\end{minipage}
\end{figure*}

%% file: tables/camelyon_stats.tex
\begin{table}[ht]
    \caption{Statistics for the Camelyon16 training and test sets we used. \textbf{Left}: Number of patches (set size) per instance. \textbf{Right}: The distribution of positive and negative samples.}
    \begin{minipage}[b]{0.54\linewidth}
    \centering
    \begin{tabular}{lcc}
        \toprule
        Metric & Train & Test \\
        \midrule
        Mean & 9,329 & 9,376 \\
        Min & 154 & 1558 \\
        Max & 32,382 & 37,345 \\
        \bottomrule
    \end{tabular}
    \end{minipage}
    \begin{minipage}[b]{0.4\linewidth}
    \centering
    \begin{tabular}{lcc}
        \toprule
        Metric & Train & Test \\
        \midrule
        Positive ($+$) & 110 & 49 \\
        Negative ($-$) & 160 & 80 \\
        \bottomrule
    \end{tabular}
    \end{minipage}
    \label{tab:camelyon-stats}
\end{table}

%% file: tables/attention_acts.tex
\begin{table}[ht]
\centering
\caption{Valid UMBC attention activation functions with slot normalization $\nu_1$ and normalization over the set elements $\bbf_\theta$. Slot-exp uses $\hat{A}_{i,j}-\max_{1\leq i \leq k}\hat{A}_{i,j}$  instead of  $\nu_p$.}
\label{tab:attn-acts}
 \resizebox{0.4\linewidth}{!}{
  \begin{tabular}{cccll}
    \toprule
    function $\sigma$ &  $\nu_p$ & $\bbf_\theta$ & name & reference \\
    \midrule
    $\mathtt{sigmoid}$ & $p=1$ & - & slot-sigmoid & \citep{mbc} \\
    \midrule
    $\exp()$ & $p=1$ & \cmark & slot-softmax & \citep{slot-attn} \\
    \midrule
    $\exp()$ & $p=2$ & \cmark & softmax & \citep{set-trans} \\
    \midrule
    $\exp()$ & -\footnotemark & \cmark & slot-exp & - \\
    \midrule
    $\mathtt{sigmoid}()$ & $p=2$ & \cmark & sigmoid & - \\ 
    \bottomrule
  \end{tabular}
 }
\end{table}

%% file: figures/algo.tex
\begin{figure}[ht]
\vspace{-0.15in}
\centering
\begin{minipage}{0.75\linewidth}
\begin{algorithm}[H]
   \caption{\small Unbiased Full Set Gradient Estimation}
   \label{greedy-training}
  \begin{algorithmic}[1]
    \STATE \textbf{Input}: Dataset $((X_i,y_i))_{i=1}^n$, batch size $m$, the number of subsets $m^\prime$, learning rate $(\eta_t)_{t=1}^T$, total steps $T$, and functions $f_\theta$, and $g_\lambda$.
     \STATE Randomly initialize $\theta$ and $\lambda$
     \FORALL{$t=1,\ldots, T$}
        \STATE Sample $((\bX_i, \by_i))_{i=1}^m \sim D[((X_i,y_i))_{i=1}^n]$
        \STATE $L_{t,1}(\theta,\lambda)\leftarrow 0, L_{t,2}(\theta, \lambda) \leftarrow 0$
        \FORALL{$i=1\ldots, m$}
        \STATE Partition a set $\bX_i$ to get $\zeta_t(\bX_i)$
        \STATE Sample $\bzeta_t(\bX_i)\sim D[\zeta_t(\bX_i)]$ with $|\bzeta_t(\bX_i)|=m^\prime$
        \STATE $f_{\theta}(\bX_i) = \sum_{\bS \in \bzeta_t(\bX_i)}f_\theta(\bS) + \sum_{S \in \zeta_t(\bX_i) \backslash \bzeta_t(\bX_i)}\texttt{StopGrad}(f_\theta(S)) $
        \STATE $L_{t,1}(\theta, \lambda) \leftarrow L_{t,1}(\theta, \lambda) + \frac{1}{m} \frac{|\zeta_t(\bX_i)|}{|\bzeta_t(\bX_i)|} \ell(g_\lambda(f_\theta(\bX_i)), \by_i)$
        \STATE $L_{t,2}(\theta,\lambda) \leftarrow L_{t,2}(\theta, \lambda) + \frac{1}{m} \frac{1}{\lvert \bzeta_t(\bX_i)\rvert} \ell(g_\lambda(f_\theta(\bX_i)), \by_i)$
    \ENDFOR
    \STATE $\theta \leftarrow \theta - \eta_t\frac{\partial L_{t,1}(\theta, \lambda)}{\partial \theta}$
    \STATE$\lambda \leftarrow \lambda - \eta_t\frac{\partial L_{t,2}(\theta_t, \lambda)}{\partial \lambda}$ 
    \ENDFOR
  \end{algorithmic}
\label{algo}
\end{algorithm}
\end{minipage}
\vspace{-0.1in}
\end{figure}